%% file: iclr2027_conference.tex
\documentclass{article} 
\usepackage{iclr2027_conference,times}

\input{math_commands.tex}

\usepackage[hidelinks]{hyperref}
\usepackage{url}
\usepackage{amsmath,amssymb,amsthm}
\usepackage{graphicx}
\usepackage{booktabs}
\usepackage{etoolbox}

\newtheorem{theorem}{Theorem}[section]
\newtheorem{proposition}[theorem]{Proposition}
\newtheorem{lemma}[theorem]{Lemma}
\newtheorem{corollary}[theorem]{Corollary}

\theoremstyle{remark}
\newtheorem{remark}[theorem]{Remark}

\theoremstyle{definition}
\newtheorem{definition}[theorem]{Definition}

\title{Directional Linear Separability of Neural Representations: Geometry and Transformations}

\author{
Yi Wei\thanks{Yi Wei and Xuan Qi contributed equally to this work.} \\
National Key Laboratory of Novel Software Technology \\
School of Intelligence Science and Technology \\
Nanjing University, Jiangsu, China \\
\texttt{ywei@smail.nju.edu.cn} \\
\And
Xuan Qi\footnotemark[1] \\
AI for Good (AIGO) \\
Istituto Italiano di Tecnologia, Genoa, Italy \\
DITEN \\
University of Genoa, Genoa, Italy \\
\texttt{xuan.qi@iit.it} \\
\And
Suorong Yang \\
National Key Laboratory for Novel Software Technology \\
Nanjing University, China \\
Department of Computer Science and Technology \\
Nanjing University, China \\
\texttt{sryang@smail.nju.edu.cn} \\
\And
Furao Shen\thanks{Corresponding author.} \\
National Key Laboratory of Novel Software Technology \\
School of Artificial Intelligence \\
Nanjing University, Jiangsu, China \\
\texttt{frshen@nju.edu.cn}
}

\iclrfinalcopy

\makeatletter
\patchcmd{\@maketitle}
  {\lhead{Published as a conference paper at ICLR 2027}}
  {\lhead{Preprint}}
  {}
  {\GenericError{}{Failed to replace ICLR header}{}{}}
\makeatother

\begin{document}

\maketitle

\begin{abstract}
Neural networks build representations through affine maps and nonlinear activations. Injective affine maps preserve linear separability, raising the problem of how they prepare data for nonlinear improvement and how much gain can be guaranteed before complete separation. We introduce the directional linear separability measure (D-LSM), which quantifies unavoidable competing-sample intrusion over affine halfspaces retaining every target sample, characterize its supporting geometry, and prove invariance under injective affine embeddings. For gated activations including ReLU, GELU, and SiLU, pre-activation projection bounds yield sufficient conditions for preserving all previous exclusions and recovering additional samples, with a gain bound determined by the certified recovery count. Under an aggregate-tube condition, an explicit affine construction realizes recovery with sufficient width, scaling conditions, and simultaneous multiclass guarantees through a shared layer. Exact controlled experiments compare certified and realized gains, assess certificate coverage, and exhibit bound attainment before complete separation and in affine-tube constructions. In learned Vision Transformer (ViT) representations, a feasible lower-bound estimator yields earlier post-GELU saturation certificates of exact separability, while boundary transport numerically supports affine invariance.
\end{abstract}

\section{Introduction}
\label{sec:introduction}

Neural classifiers typically combine successive representation transformations with a final affine readout. This architecture raises a geometric question: how do intermediate transformations reshape class geometry to improve linear separability?

In an affine--activation module, injective affine maps preserve linear separability but change the coordinates seen by the activation. Activations can separate additional competitors from a target class or undo previous exclusions. The challenge is to identify pre-activation geometry guaranteeing quantitative improvement before complete separation.

We address this challenge by measuring competing-sample intrusion under exact target preservation. We derive pre-activation conditions for retaining existing exclusions and recovering additional samples, then construct affine coordinates that realize them. For finite target and competing collections \(A,B\subset\mathbb{R}^d\), we define the directional linear separability measure (D-LSM):
\begin{equation}
s(A,B):=1-\frac{1}{|A|}\min_{H:A\subseteq H}|B\cap H|,
\end{equation}
where \(H\) ranges over affine halfspaces. Exchanging \(A\) and \(B\) changes the preservation constraint and normalization, as illustrated in Figure~\ref{fig:directional_lsm}.

\begin{figure}[t]
\centering
\begin{minipage}[t]{0.4\linewidth}
\centering
\includegraphics[width=\linewidth]{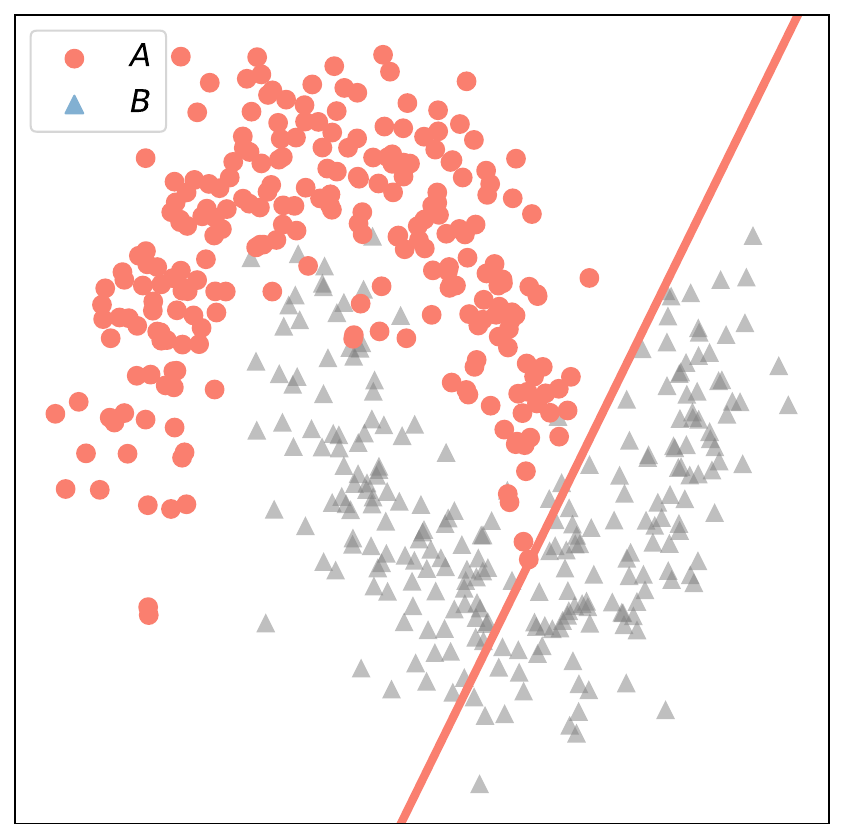}
\small (a) \(A\) as target: \(s(A,B)=0.508\)
\end{minipage}
\hspace{0.04\linewidth}
\begin{minipage}[t]{0.4\linewidth}
\centering
\includegraphics[width=\linewidth]{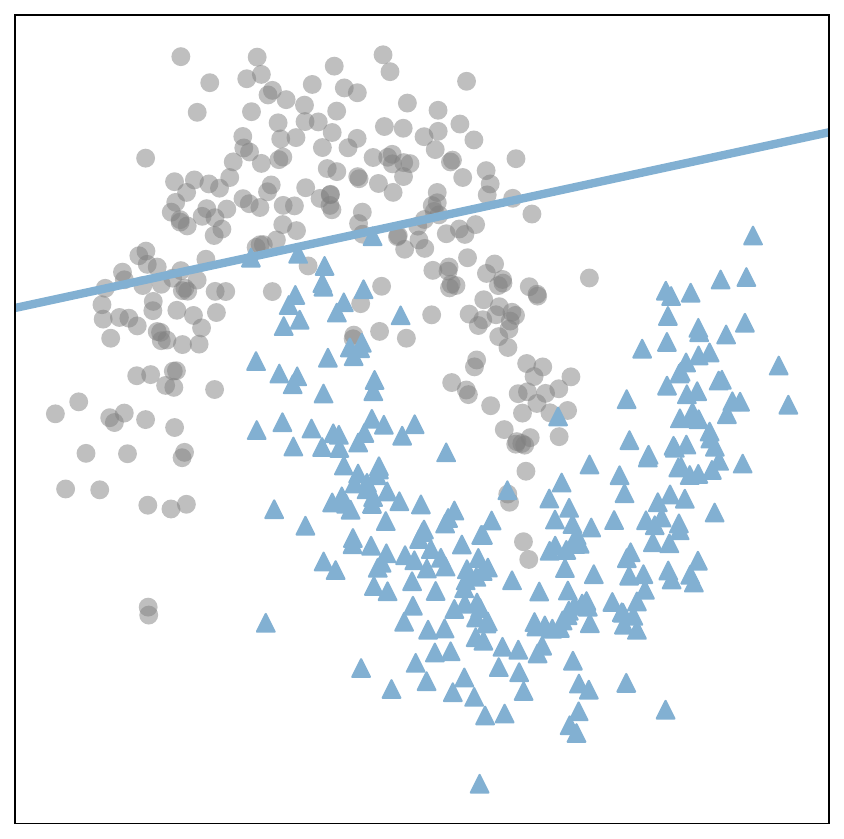}
\small (b) \(B\) as target: \(s(B,A)=0.476\)
\end{minipage}
\caption{Directional dependence of D-LSM under exchanging the target and competing samples.}
\label{fig:directional_lsm}
\end{figure}

Our analysis tracks competing samples relative to an optimal target-supporting boundary. From pre-activation geometry, we bound the transformed target support and competing-sample projections to jointly certify retained exclusions and additional recoveries. We realize these conditions through an injective affine construction that combines a supporting-direction coordinate, paired coordinates, and rank completion. The affine stage preserves D-LSM and the original admitted/excluded partition. After activation, the paired responses enable recovery while the supporting-direction coordinate maintains previous exclusions, under explicit geometric and scaling conditions.

Exact controlled experiments compare certified and realized gains before complete separation, examine certificate coverage, and attain the affine-tube bound on constructed families. For learned Vision Transformer (ViT) representations, we use a feasible lower-bound estimator whose value \(1\) certifies exact separability. The resulting certification profiles and a direct affine boundary-transport audit connect the geometric analysis to learned representations.

Our main contributions are:
\begin{itemize}
\item \textbf{Geometry of target-preserving separability.} We formulate D-LSM under exact target preservation, characterize its supporting-halfspace geometry, and prove invariance under injective affine embeddings, providing a reference for analyzing activation-induced changes in separability.
\item \textbf{Joint preservation and recovery certificates.} For coordinatewise gated activations including ReLU, GELU, and SiLU \citep{nair2010rectified,hendrycks2016gelu,elfwing2018sigmoid}, we derive pre-activation directional certificates for preservation and recovery. When all previous exclusions are certified as preserved, the certified recovery count yields a quantitative D-LSM improvement bound before complete separation.
\item \textbf{Constructive realization with explicit width.} Under an aggregate-tube condition, we construct an injective affine layer with explicit scaling and sufficient width that preserves exclusions and recovers an admitted subset \(R\) after activation, guaranteeing a D-LSM gain of at least \(|R|/|A|\). The construction extends to simultaneous one-vs-rest recovery through a shared layer.
\end{itemize}

\section{Related Work}
\label{sec:related_work}

\paragraph{Linear separability and halfspace geometry.}
Linear separability and maximum-margin classification have long connected affine decision boundaries with convex geometry \citep{cover1965geometrical,cortes1995support,bennett2000duality}. Halfspace-based statistics provide more direct one-sided analogues: Tukey depth measures point centrality by halfspace mass \citep{tukey1975mathematics}, while set-valued halfspace depth extends this construction through directional support functions \citep{cascos2021depth}. Under an empirical singleton-set specialization, the intrusion count underlying D-LSM coincides with the lower halfspace-depth component \(HD^-\) evaluated at \(\operatorname{conv}(A)\). Computational geometry likewise studies bichromatic regions, including halfplanes, that contain all points of one color while minimizing points of the other \citep{glazenburg_et_al:LIPIcs.ISAAC.2024.33}. These connections place D-LSM within the broader literature on one-sided halfspace geometry. We build on this viewpoint to study directional separability in neural representations, with particular emphasis on how it is preserved or altered by representation transformations.

\paragraph{Separability of neural representations.}
Linear probes provide a common readout of label accessibility in intermediate representations \citep{alain2016understanding,belinkov2022probing}, with probe capacity requiring care in interpretation \citep{hewitt2019designing}. Representation geometry has also been studied through manifold capacity, topology, neural collapse, and layerwise separation or compression \citep{chung2018classification,naitzat2020topology,papyan2020prevalence,he2023law,masarczyk2023tunnel,mamou2020emergence}. Within this setting, \citet{zhang2026md} use Minkowski-difference-based measures to monitor hidden-layer linear separability. D-LSM instead conditions the analysis on exact preservation of a designated target class and resolves competing-sample intrusion relative to that class.

\paragraph{Nonlinear transformations and separability.}
Constructive ReLU separation builds on convex-hull decompositions, with distance-preserving extensions \citep{an2015can}, while random-network guarantees link separation to geometry and width \citep{dirksen2022separation,ghosal2022randomly,xu2025linearly}. Related studies examine concentration and dimensionality \citep{zarka2020separation,ansuini2019intrinsic,pope2021intrinsic,gorban2017stochastic}. Our results quantify D-LSM gains before complete separation by jointly certifying retained exclusions and additional recoveries from pre-activation geometry. On a common planar family, the construction of \citet[Theorem~4]{an2015can} gives two ReLU stages of widths \(4\) and \(2\), whereas our affine-tube construction achieves separation in one affine--activation stage of width \(3\). Appendix~\ref{app:comparison_an2015} provides the assumption--conclusion comparison and common-family analysis.

\paragraph{One-sided classification objectives.}
Reliable learning and Neyman--Pearson classification constrain one type of error while optimizing another \citep{kalai2012reliable,diakonikolas2024reliable,scott2005neyman,rigollet2011neyman}. These frameworks formulate asymmetric prediction or population-risk objectives; D-LSM instead uses exact target preservation to characterize the geometry of a fixed finite representation.

\section{Directional Separability: Geometry and Transformations}
\label{sec:directional_separability}

We study linear separability from a directional perspective: how many competing samples can be excluded while preserving every sample of a designated target class. We first characterize this quantity through the supporting geometry of the target convex hull, which provides the geometric basis for studying how representation transformations preserve or alter separability.

\subsection{D-LSM and Supporting Geometry}
\label{sec:lsm_geometry}

Let \(A,B\subset\mathbb{R}^d\) be finite nonempty labeled sample collections, where \(A\) is the target class and \(B\) is the competing class. We use set notation for readability: \(|S|\) counts labeled samples and pointwise transformations retain sample multiplicities, whereas \(\operatorname{conv}(S)\) depends only on the corresponding feature values; see Appendix~\ref{app:sample_collection_convention}.

\begin{definition}[Directional linear separability measure]
\label{def:directional_lsm}
For \(w\neq0\) and \(c\in\mathbb{R}\), let \(H_{w,c}:=\{x\in\mathbb{R}^{d}:w^\top x+c\ge0\}\) be an affine halfspace and define \(\mathcal{H}_A:=\{H_{w,c}:A\subseteq H_{w,c}\}\). The directional linear separability measure (D-LSM) from \(A\) to \(B\) is
\begin{equation}
s(A,B)
:=
1-\frac{1}{|A|}
\min_{H\in\mathcal{H}_A}|B\cap H|.
\label{eq:directional_lsm}
\end{equation}
\end{definition}

Write \(m(A,B):=\min_{H\in\mathcal H_A}|B\cap H|\), so \(s(A,B)=1-m(A,B)/|A|\), where \(m(A,B)\) is the unavoidable competing-sample intrusion under exact target preservation. D-LSM is directional because exchanging the classes changes the preserved collection and normalization; its signed range is discussed in Appendix~\ref{app:target_normalization}. For multiclass samples \(A_1,\ldots,A_K\), we use \(s(A_k):=s\!\left(A_k,\bigsqcup_{\ell\neq k}A_\ell\right)\), and represent \(H_{w,c}\) by \(f_{w,c}(x)=w^\top x+c\) when convenient.

\begin{theorem}[Supporting-halfspace characterization]
\label{thm:supporting_halfspace}
There exists an optimal halfspace \(H^\star\in\mathcal{H}_A\) attaining the minimum in \eqref{eq:directional_lsm} whose boundary supports \(\operatorname{conv}(A)\); that is,
\begin{equation}
\operatorname{conv}(A)\subseteq H^\star,
\qquad
\partial H^\star\cap\operatorname{conv}(A)\neq\varnothing.
\end{equation}
Hence, restricting \eqref{eq:directional_lsm} to target-preserving supporting halfspaces does not change \(s(A,B)\).
\end{theorem}

Indeed, any feasible boundary not touching \(\operatorname{conv}(A)\) can be translated toward the target convex hull until first contact while preserving \(A\) and without introducing additional competing samples. The complete proof is deferred to Appendix~\ref{app:proof_supporting_halfspace}.

Fix such an optimal supporting halfspace \(H^\star\), represented by \(f^\star(x)=w^{\star\top}x+c^\star\). We partition the competing samples according to whether they remain admitted by \(H^\star\):
\begin{align}
B_a(f^\star)
&:=\{b\in B:f^\star(b)\ge0\},\\
B_b(f^\star)
&:=\{b\in B:f^\star(b)<0\}.
\end{align}
Thus,
\[
B=B_a(f^\star)\sqcup B_b(f^\star),
\qquad
s(A,B)=1-\frac{|B_a(f^\star)|}{|A|}.
\]
Importantly, \(B_a(f^\star)\) is defined only by the current optimal halfspace and may contain samples either inside or outside \(\operatorname{conv}(A)\), as illustrated in Figure~\ref{fig:intrusion_decomposition}.

\begin{figure}[t]
\centering
\includegraphics[width=0.49\linewidth]{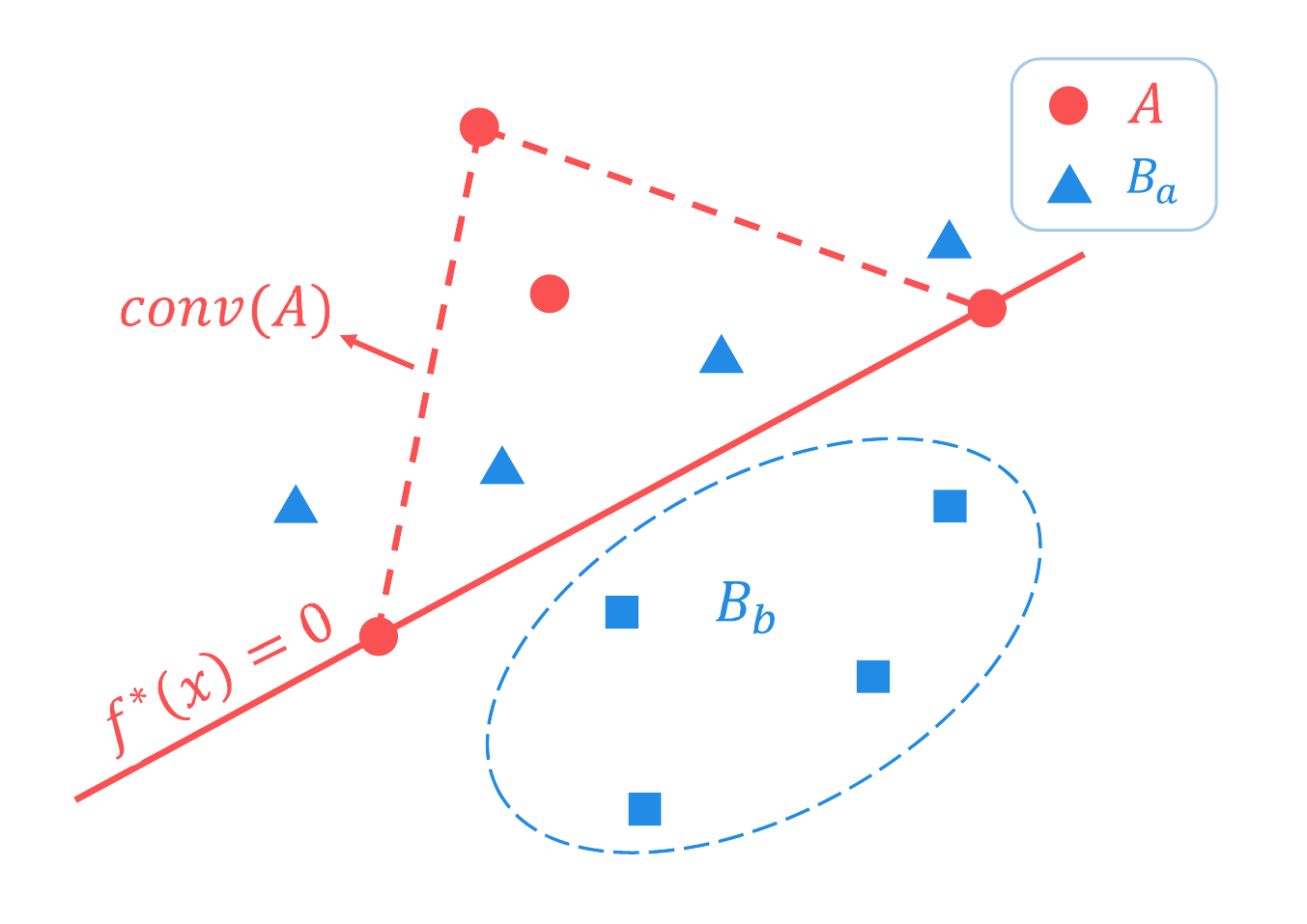}
\caption{Competing samples relative to an optimal target-preserving supporting halfspace. Samples in \(B_a\) remain admitted, whereas samples in \(B_b\) are already excluded.}
\label{fig:intrusion_decomposition}
\end{figure}

\subsection{Basic Properties and Affine Invariance}
\label{sec:lsm_properties}

Definition~\ref{def:directional_lsm} gives \(1-|B|/|A|\le s(A,B)\le1\), so D-LSM can be negative when the unavoidable competing-class intrusion exceeds the number of target samples.

\paragraph{Relation to affine classification accuracy.}
D-LSM is bounded above by the optimal empirical accuracy of an affine binary classifier, with equality exactly under linear separability. Unlike standard accuracy, D-LSM enforces exact target preservation; the formal statement and proof are given in Appendix~\ref{app:proof_accuracy_relation}.

\begin{theorem}[Invariance under injective affine embeddings]
\label{thm:affine_invariance}
Let \(T:\mathbb{R}^d\rightarrow\mathbb{R}^h\) be \(T(x)=Mx+b\), where \(M\in\mathbb{R}^{h\times d}\), \(h\ge d\), and \(\operatorname{rank}(M)=d\). Then
\[
s(A,B)=s(T(A),T(B)).
\]
\end{theorem}

Thus, injective affine reparameterizations can change coordinates without changing exact D-LSM. This separates affine re-expression from the nonlinear geometric changes studied next. The proof is given in Appendix~\ref{app:proof_affine_invariance}.

\subsection{Geometry of Coordinatewise Gated Activations}
\label{sec:activation_geometry}

Consider coordinatewise activations of the form \(\rho_\sigma(t):=t\sigma(t)\), where \(\sigma:\mathbb{R}\rightarrow[0,1]\) is nondecreasing with \(\lim_{t\to-\infty}\sigma(t)=0\) and \(\lim_{t\to+\infty}\sigma(t)=1\). This family includes ReLU, GELU, and SiLU through their corresponding gates \citep{nair2010rectified,hendrycks2016gelu,elfwing2018sigmoid,ramachandran2018searching}.

Assume that \(\rho_\sigma\) attains a global minimum. Since \(\rho_\sigma(0)=0\) and \(\rho_\sigma(t)\ge0\) for \(t\ge0\), we may choose a minimizer \(\tau\le0\) and define
\begin{equation}
\mu:=\rho_\sigma(\tau)\le0,
\qquad
I_+:=[0,\infty),\quad
I_0:=[\tau,0),\quad
I_-:=(-\infty,\tau).
\label{eq:activation_regions}
\end{equation}

\begin{lemma}[Region-wise activation deformation bounds]
\label{lem:activation_envelope}
Let \(\alpha_\sigma^+:=\inf_{t>0}\sigma(t)\in[0,1]\). Then
\begin{equation}
\rho_\sigma(t)\in
\begin{cases}
[\alpha_\sigma^+t,\;t], & t\in I_+,\\
[\max\{t,\mu\},\;\min\{0,t+\mu-\tau\}], & t\in I_0,\\
[\mu,\;0], & t\in I_-.
\end{cases}
\label{eq:activation_envelopes}
\end{equation}
Moreover, \(\rho_\sigma\) is nondecreasing on \(I_+\), and for \(0\le r\le t\), \(\rho_\sigma(t)\ge\rho_\sigma(r)\) and \(\sigma(r)t\le\rho_\sigma(t)\le t\).
\end{lemma}

Let \(\ell_\sigma(t)\) and \(r_\sigma(t)\) denote the lower and upper endpoints in \eqref{eq:activation_envelopes}. The regions \(I_+\), \(I_0\), and \(I_-\) respectively describe order-preserving shrinkage, bounded displacement, and negative saturation, while \(\ell_\sigma(t)\le\rho_\sigma(t)\le r_\sigma(t)\) provides the certificates used below. Proofs and the ReLU, GELU, and SiLU cases are given in Appendix~\ref{app:activation_properties}.

\subsection{Preservation, Recovery, and Guaranteed D-LSM Improvement}
\label{sec:lsm_improvement}

We next ask what pre-activation geometry is sufficient for a known activation to preserve already excluded samples in \(B_b\) and recover additional samples from \(B_a\). Lemma~\ref{lem:activation_envelope} allows this question to be expressed through activation-dependent directional bounds. For \(x\in\mathbb{R}^{d}\) and a direction \(u\in\mathbb{R}^{d}\), define
\begin{equation}
\begin{aligned}
L_u^\sigma(x)
&:=\sum_{j=1}^{d}\min\{u_j\ell_\sigma(x_j),u_jr_\sigma(x_j)\},\\
U_u^\sigma(x)
&:=\sum_{j=1}^{d}\max\{u_j\ell_\sigma(x_j),u_jr_\sigma(x_j)\},\qquad
\bar h_A^\sigma(u):=\max_{a\in A}U_u^\sigma(a).
\end{aligned}
\label{eq:directional_projection_envelopes}
\end{equation}
Then \(L_u^\sigma(x)\le u^\top\rho_\sigma(x)\le U_u^\sigma(x)\), and \(\bar h_A^\sigma(u)\) upper-bounds the transformed target support. Fix the outward unit normal \(u:=-w^\star/\|w^\star\|\) of the optimal supporting halfspace, so \(H^\star=\{x:u^\top x\le h_A(u)\}\) and every \(b\in B_b(f^\star)\) lies strictly beyond \(h_A(u):=\max_{a\in A}u^\top a\).

\begin{proposition}[Sufficient condition for preservation and recovery]
\label{prop:preservation_recovery}
Suppose either \(B_b(f^\star)=\varnothing\), or
\begin{equation}
\min_{b\in B_b(f^\star)}L_u^\sigma(b)
>
\bar h_A^\sigma(u).
\label{eq:preservation_condition}
\end{equation}
Then a post-activation halfspace with normal \(u\) preserves \(\rho_\sigma(A)\) and excludes all samples in \(\rho_\sigma(B_b(f^\star))\).

Moreover, define
\begin{equation}
R_a(u)
:=
\left\{
b\in B_a(f^\star):
L_u^\sigma(b)>\bar h_A^\sigma(u)
\right\}.
\label{eq:recoverable_subset}
\end{equation}
The same halfspace also excludes every sample in \(\rho_\sigma(R_a(u))\).
\end{proposition}

The proof is given in Appendix~\ref{app:proof_preservation_recovery}.

\begin{theorem}[Certified D-LSM improvement]
\label{thm:lsm_improvement}
Let \(\bar A:=\rho_\sigma(A)\) and \(\bar B:=\rho_\sigma(B)\). Under the assumptions of Proposition~\ref{prop:preservation_recovery},
\begin{equation}
s(\bar A,\bar B)
\ge
s(A,B)
+
\frac{|R_a(u)|}{|A|}.
\label{eq:lsm_improvement_bound}
\end{equation}
Hence \(s(\bar A,\bar B)\ge s(A,B)\), with strict improvement whenever \(R_a(u)\neq\varnothing\).
\end{theorem}

The bound follows by counting the preserved exclusions and the additional samples in \(R_a(u)\); see Appendix~\ref{app:proof_lsm_improvement}. These are sufficient certificates: the activation may preserve or recover additional samples beyond those identified by the envelope bounds. To construct pre-activation geometry satisfying the certificate, the next result uses affine invariance and paired coordinates that cancel before activation but yield preservation and recovery afterward under an aggregate-tube condition.

\begin{theorem}[Affine-Tube Recovery]
\label{thm:affine_tube_recovery}
Let \(f^\star\), \(u\), \(B_a(f^\star)\), and \(B_b(f^\star)\) be as in Proposition~\ref{prop:preservation_recovery}, and define \(q(x):=u^\top x-h_A(u)\). Suppose there exist \(m\ge1\) normalized affine functions \(g_k^0(x)=p_k^\top x+\beta_k\), with \(\|p_k\|_2=1\), such that
\[
D_G(x):=\sum_{k=1}^{m}|g_k^0(x)|,\qquad
D_A:=\min_{a\in A}D_G(a)>0,\qquad
R:=\{b\in B_a(f^\star):D_G(b)\le\varepsilon\}\neq\varnothing
\]
for some \(0\le\varepsilon<\alpha_\sigma^+D_A\), where \(\alpha_\sigma^+\) and \(\mu\) are defined in Lemma~\ref{lem:activation_envelope}. Choose
\[
\kappa>\frac{-(m+1)\mu}{\alpha_\sigma^+D_A-\varepsilon},
\qquad
\lambda>
\max\left\{
0,\,
\max_{b\in B_b(f^\star)}
\frac{\kappa D_G(b)-\alpha_\sigma^+\kappa D_A-m\mu}
{\alpha_\sigma^+q(b)}
\right\},
\]
where the second maximum is omitted if \(B_b(f^\star)=\varnothing\).

For every
\[
h\ge 1+2m+d-\dim\operatorname{span}\{u,p_1,\ldots,p_m\},
\]
the coordinates \(\lambda q(x)\) and
\(\{\kappa g_k^0(x),-\kappa g_k^0(x)\}_{k=1}^{m}\) can be augmented by linear rank-completion coordinates, and by zero coordinates when necessary, to form an affine embedding \(T(x)=Mx+c\) with \(\operatorname{rank}(M)=d\). For the lifted outward unit normal
\[
\widetilde u:=
\frac{(1,\underbrace{-1,\ldots,-1}_{2m},
\underbrace{0,\ldots,0}_{h-1-2m})}{\sqrt{1+2m}},
\]
we have
\[
L_{\widetilde u}^{\sigma}(T(b))
>
\bar h_{T(A)}^\sigma(\widetilde u),
\qquad
b\in B_b(f^\star)\cup R.
\]
Consequently, Proposition~\ref{prop:preservation_recovery} preserves the exclusion of all samples in \(B_b(f^\star)\) and recovers every sample in \(R\). Together with Theorems~\ref{thm:affine_invariance} and~\ref{thm:lsm_improvement}, this yields
\[
s\!\left(\rho_\sigma(T(A)),\rho_\sigma(T(B))\right)
\ge
s(T(A),T(B))+\frac{|R|}{|A|}
=
s(A,B)+\frac{|R|}{|A|}.
\]
\end{theorem}

The proof is given in Appendix~\ref{app:proof_affine_tube_recovery}. We next extend this construction to simultaneous recovery across multiple classes.

\begin{corollary}[Multiclass Recovery]
\label{cor:multiclass_recovery}
Let \(A_1,\ldots,A_K\subset\mathbb{R}^d\) be \(K\ge2\) finite nonempty class collections, and define \(B_k:=\bigsqcup_{\ell\neq k}A_\ell\). Suppose that, for the same activation \(\rho_\sigma\), every one-vs-rest pair \((A_k,B_k)\) satisfies the assumptions of Theorem~\ref{thm:affine_tube_recovery}, with all associated quantities indexed by \(k\). In particular, write \(q_k(x):=u_k^\top x-h_{A_k}(u_k)\) and \(g_{k,j}^0(x):=p_{k,j}^\top x+\beta_{k,j}\), where \(\|p_{k,j}\|_2=1\), \(j=1,\ldots,m_k\), and denote the corresponding recovery set and coefficients by \(R_k\), \(\kappa_k\), and \(\lambda_k\).

For every
\[
h
\ge
K+2\sum_{k=1}^{K}m_k+d-
\dim\operatorname{span}\!\left(
\{u_k:1\le k\le K\}
\cup
\{p_{k,j}:1\le k\le K,\ 1\le j\le m_k\}
\right),
\]
the anchor coordinates \(\{\lambda_kq_k(x)\}_{k=1}^{K}\) and paired coordinates
\[
\left\{
\kappa_kg_{k,j}^0(x),
-\kappa_kg_{k,j}^0(x)
:
1\le k\le K,\ 
1\le j\le m_k
\right\}
\]
can be augmented by common linear rank-completion coordinates, and by zero coordinates when necessary, to form a single affine embedding \(T(x)=Mx+c\) with \(\operatorname{rank}(M)=d\). Then, simultaneously for every \(k=1,\ldots,K\),
\[
s\!\left(
\rho_\sigma(T(A_k)),
\rho_\sigma(T(B_k))
\right)
\ge
s(T(A_k),T(B_k))
+\frac{|R_k|}{|A_k|}
=
s(A_k,B_k)
+\frac{|R_k|}{|A_k|}.
\]
\end{corollary}

Each certificate uses only its own classwise coordinate block, so a single affine layer with shared rank-completion coordinates realizes all recovery guarantees simultaneously. The proof is given in Appendix~\ref{app:proof_multiclass_recovery}.

\paragraph{Connection to neural networks.}
At a fixed representation location, one class and the remaining classes instantiate \(A_k\) and \(B_k\). An affine--activation module \(z\mapsto\rho_\sigma(Wz+b)\) then matches the transformation studied above: a full-column-rank \(W\) preserves exact D-LSM by Theorem~\ref{thm:affine_invariance}, while the activation admits the preservation and recovery guarantees of Proposition~\ref{prop:preservation_recovery} and Theorem~\ref{thm:lsm_improvement}. Theorem~\ref{thm:affine_tube_recovery} and Corollary~\ref{cor:multiclass_recovery} further connect layer width and affine-coordinate design to binary and multiclass recovery, motivating the experiments in Section~\ref{sec:experiments}.

\subsection{Practical D-LSM Estimation}
\label{sec:practical_dlsm_estimation}

Exact D-LSM optimization combines a discrete intrusion count with exact target preservation. We guide the affine search using the weighted surrogate \(\mathcal{L}_{\mathrm{sur}}(w,c)=\beta\mathcal{L}_B(w,c)+\alpha\mathcal{L}_A(w,c)\), with adaptive weights prioritizing target feasibility, and retain the feasible candidate halfspaces found during optimization. For the resulting nonempty searched family \(\widehat{\mathcal H}_A\subseteq\mathcal H_A\), we define
\begin{equation}
\widehat{s}(A,B)
=
1-\frac{1}{|A|}
\min_{H\in\widehat{\mathcal H}_A}|B\cap H|.
\label{eq:estimated_lsm}
\end{equation}
Since \(\widehat{\mathcal H}_A\subseteq\mathcal H_A\), we have \(\widehat{s}(A,B)\le s(A,B)\); in particular, \(\widehat{s}(A,B)=1\) certifies exact D-LSM saturation. Optimization and boundary-selection details are given in Appendix~\ref{app:lsm_estimator}.

\section{Experiments}
\label{sec:experiments}

We first evaluate quantitative recovery guarantees and their conservatism in controlled affine--activation modules using exact D-LSM computation. We then test the affine-tube construction that realizes these guarantees through a preceding affine layer, followed by saturation certification and affine boundary transport in learned ViT representations.

\subsection{Controlled Theoretical Validation}
\label{sec:controlled_theoretical_validation}

\paragraph{Activation certificates and quantitative recovery.}
Basic envelope and certificate audits over \(1{,}200\) activation evaluations find no violations; their setups and results appear in Appendices~\ref{app:activation_experiment_setup} and~\ref{app:activation_experiment_results}. We further examine the quantitative guarantee in Theorem~\ref{thm:lsm_improvement} before complete separation. Starting from exact D-LSM \(0.6\), we evaluate \(x\mapsto\rho_\sigma(\gamma(x+t\mathbf{1}))\), with \(\gamma>0\), under ReLU, GELU, and SiLU. Twenty data seeds, three Dirichlet concentrations \(\nu\in\{0.2,1,10\}\), and \(102\) scale--shift settings give \(18{,}360\) evaluations with exact post-activation D-LSM. Positive improvement is certified in \(5{,}841\) evaluations that remain nonseparable.

\begin{figure}[t]
\centering
\includegraphics[width=\linewidth]{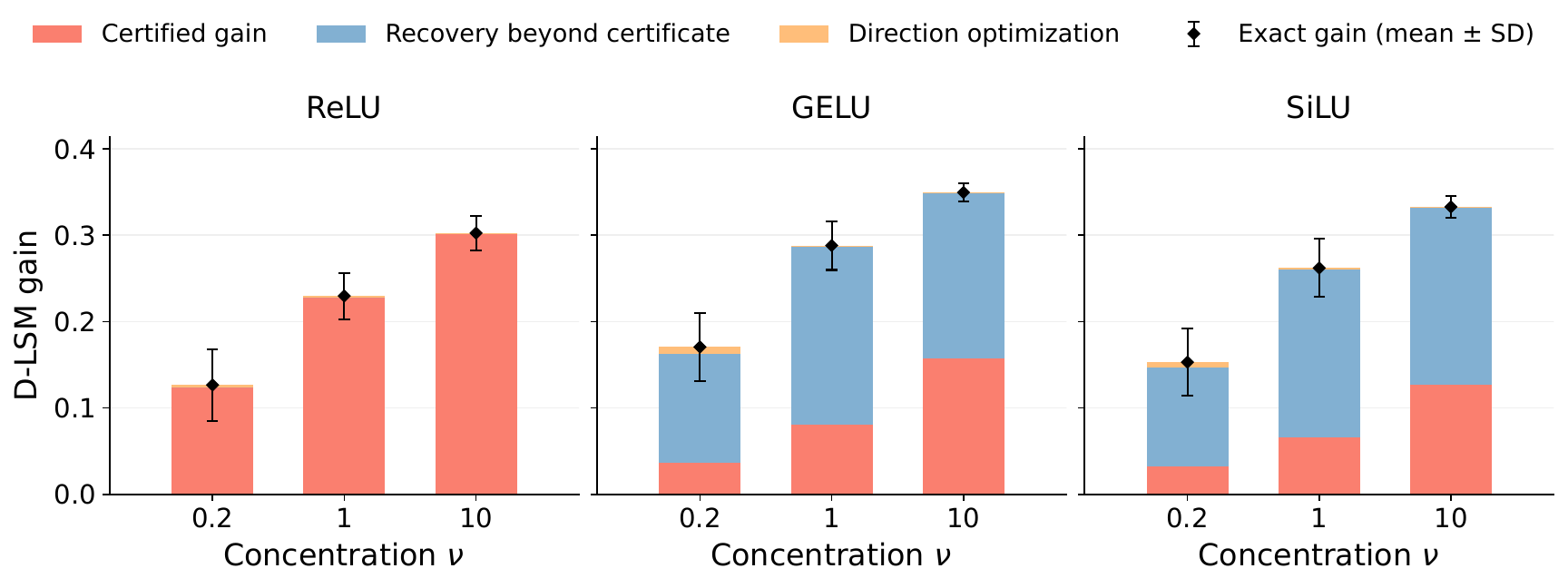}
\caption{
Gain decomposition for the \(6{,}870\) preservation-certified evaluations with \(0.6<s_{\mathrm{post}}<1\), including zero certified gains. Bars stack the certified gain, additional recovery along the same direction, and improvement from optimizing the direction. Eligible affine settings are averaged within each seed. Diamonds and error bars show the exact gain's mean and standard deviation across \(20\) seed means.
}
\label{fig:affine_activation_gain}
\end{figure}

Across all \(8{,}202\) preservation-certified evaluations, the exact gain meets or exceeds the certified bound. Figure~\ref{fig:affine_activation_gain} separates certified gains from additional recovery and direction optimization. ReLU certifies every observed recovery along the transported direction. For GELU and SiLU, additional recovery accounts for most of the mean gap, while some parameter settings nearly attain the bound: at \((\nu,\gamma,t)=(10,1.5,-0.5)\), both activations give mean certified and exact gains of \(0.3725\) and \(0.375\), respectively, with final D-LSM \(0.975\). At this setting, equality holds for all \(20\) ReLU seeds and \(18\) of \(20\) seeds for each smooth activation. These equality cases retain one unavoidable competing-sample intrusion, demonstrating attainment of the quantitative bound before complete separation. Complete settings and certificate-coverage analyses appear in Appendices~\ref{app:affine_activation_setup} and~\ref{app:affine_activation_results}.

\paragraph{Affine-tube recovery.}
We next test the constructive guarantee of Theorem~\ref{thm:affine_tube_recovery}, which specifies affine coordinates, scaling, and sufficient width from the aggregate-tube geometry of the original data. The reference construction uses \(\kappa=2\), \(\lambda=5\), and the sufficient width \(h=3\) given by Theorem~\ref{thm:affine_tube_recovery}. Its affine map has rank \(2=d\), and the paired coordinates cancel exactly before activation. As illustrated for SiLU in Figure~\ref{fig:affine_tube_audit}, the affine stage preserves the original membership pattern, while the activation recovers \(R\) and preserves the exclusion of \(B_b(f^\star)\). For ReLU, GELU, and SiLU, exact D-LSM changes from \(s(A,B)=s(T(A),T(B))=0.6\) to \(s(\rho_\sigma(T(A)),\rho_\sigma(T(B)))=1\).

\begin{figure}[t]
\centering
\includegraphics[width=\linewidth]{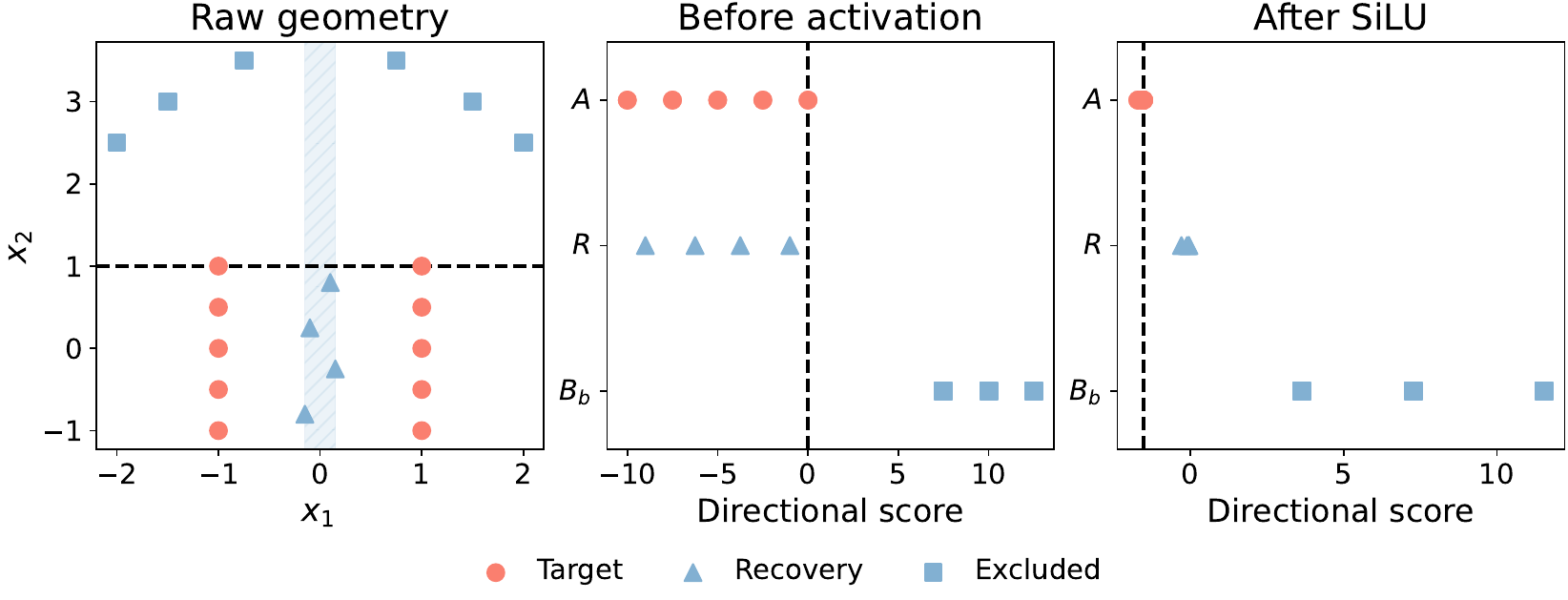}
\caption{Reference SiLU affine-tube construction. The paired coordinates cancel before activation; afterward, \(R\) is recovered and \(B_b(f^\star)\) remains excluded.}
\label{fig:affine_tube_audit}
\end{figure}

We further generate \(50\) datasets for each \(|R|\in\{4,8,12,16\}\), with \(|A|=|B|=40\), and evaluate the same data under all three activations, giving \(600\) evaluations. Table~\ref{tab:affine_tube_randomized} reports the minimum parameter slacks and certificate margins over the \(200\) evaluations for each activation. All theorem conditions remain strictly satisfied.

\begin{table}[htbp]
\centering
\small
\caption{Minimum theorem slacks and certificate margins in the randomized affine-tube experiments.}
\label{tab:affine_tube_randomized}
\setlength{\tabcolsep}{4.5pt}
\begin{tabular}{lccccc}
\toprule
Act. & \(\kappa\) slack & \(\lambda\) slack & Rec. margin & Pres. margin & Valid \\
\midrule
ReLU & \(2.000\) & \(3.730\) & \(0.982\) & \(3.274\) & \(200/200\) \\
GELU & \(1.029\) & \(0.921\) & \(0.209\) & \(0.404\) & \(200/200\) \\
SiLU & \(0.409\) & \(0.778\) & \(0.083\) & \(0.342\) & \(200/200\) \\
\bottomrule
\end{tabular}
\end{table}

Every randomized affine map has rank \(2\), zero cancellation residual, and preserves exact D-LSM. In these families, \(R=B_a(f^\star)\). After activation, all samples in \(R\) are recovered and all samples in \(B_b(f^\star)\) remain excluded, eliminating all competing-sample intrusion and yielding
\[
s(A,B)=s(T(A),T(B))=1-\frac{|R|}{40},
\qquad
s\!\left(\rho_\sigma(T(A)),\rho_\sigma(T(B))\right)=1.
\]
Hence the observed gain equals \(|R|/|A|\) in all \(600\) evaluations, attaining the lower bound in Theorem~\ref{thm:affine_tube_recovery}. Complete constructions and activation-specific margin statistics appear in Appendices~\ref{app:tube_experiment_setup} and~\ref{app:tube_experiment_results}.


\subsection{Directional Separability in Learned ViT Representations}
\label{sec:vit_experiment}

We examine the affine--GELU modules of a ViT-S/32 trained on Tiny ImageNet at epochs \(0\), \(20\), and \(200\). Using the complete training set, we evaluate all \(200\) one-vs-rest problems at the input, affine output, and post-GELU output of each of the \(12\) FFNs. Model and estimation details are provided in Appendices~\ref{app:neural_representation_setup} and~\ref{app:vit_experimental_details}.

\paragraph{D-LSM certification profile.}
The feasible estimator gives \(\widehat{s}\le s\), so \(\widehat{s}=1\) certifies exact directional separability. Every audited affine expansion is full-column-rank; hence its input and output have equal exact D-LSM by Theorem~\ref{thm:affine_invariance}. For each class and estimator seed, we therefore use \(\widehat{s}_{\mathrm{aff},c}:=\max\{\widehat{s}_{\mathrm{in},c},\widehat{s}_{\mathrm{linear},c}\}\) as the affine reference. This combines two feasible lower bounds on the same exact quantity. Table~\ref{tab:vit_dlsm_profile} compares its saturation counts with those obtained after GELU.

\begin{table}[htbp]
\centering
\small
\caption{Affine-reference/post-GELU saturation counts out of \(200\) classes, averaged over five estimator seeds. Ranges summarize FFNs \(2\)--\(11\) at initialization; trained counts are identical across these blocks.}
\label{tab:vit_dlsm_profile}
\setlength{\tabcolsep}{7pt}
\begin{tabular}{cccc}
\toprule
Checkpoint & FFN \(0\) & FFN \(1\) & FFNs \(2\)--\(11\) \\
\midrule
Epoch \(0\) & \(0.0/65.4\) & \(8.0/99.4\) & \((16.0\text{--}46.6)/(101.2\text{--}111.4)\) \\
Epoch \(20\) & \(6/177\) & \(174/177\) & \(177/177\) \\
Epoch \(200\) & \(7/177\) & \(173/177\) & \(177/177\) \\
\bottomrule
\end{tabular}
\end{table}

At initialization, post-GELU representations provide more saturation certificates throughout the network, with counts varying across depth. At epochs \(20\) and \(200\), they certify \(177\) classes from FFN \(0\), whereas the affine reference reaches this count from FFN \(2\). The trained counts agree across all five estimator seeds. Thus, the difference is concentrated in the shallow FFNs: post-GELU representations yield the high-certification profile earlier, while deeper affine representations already attain the same count.

This comparison follows the theoretical distinction between injective affine reparameterization and activation-induced geometric change. It identifies where exact-separability certificates are obtained within the network; estimates below \(1\) retain their lower-bound interpretation. Complete layer-wise results, seed statistics, and the shared class-wise profile are reported in Appendix~\ref{app:vit_full_results}.

\paragraph{Affine boundary transport audit.}
The affine reference above uses independently optimized boundaries. Theorem~\ref{thm:affine_invariance}, however, establishes a correspondence between boundaries representing the same target-preserving decision under an injective affine map. We examine this correspondence directly by fixing each learned feasible boundary, transporting it across the actual FFN expansion in both directions, and comparing its target feasibility and sample membership on the complete representation set.

Across \(72{,}000\) class-wise boundary comparisons, every target-feasibility decision is preserved. Only \(84\) of the \(7.2\times10^9\) sample-wise memberships differ, giving \(99.9999988\%\) agreement. The largest intrusion discrepancy is two competing samples, corresponding to a D-LSM discrepancy of \(0.004\). These results numerically support the boundary correspondence underlying affine invariance. They also distinguish the preserved affine geometry from variation in independently optimized estimates, which depends on the feasible boundaries found during finite search. The transport protocol, detailed error statistics, and finite-precision analysis appear in Appendices~\ref{app:vit_invariance_details} and~\ref{app:vit_invariance_results}.

\section{Conclusion}
\label{sec:conclusion}

In this work, we established a finite-sample geometric framework for directional separability under exact target preservation. The supporting-halfspace characterization and invariance under injective affine embeddings provide the foundation for sufficient preservation and recovery conditions, yielding quantitative D-LSM improvement bounds before complete separation. The Affine-Tube Recovery theorem and its multiclass extension give constructive realizations with explicit scaling and sufficient width. Controlled experiments examine the quantitative bounds, while ViT audits support the affine boundary correspondence. These results concern fixed finite representations and sufficient geometric conditions. Sharpening the recovery conditions and deriving compositional guarantees across multiple affine--activation stages remain natural directions for future work. These results concern fixed finite representations and sufficient geometric conditions. Sharpening the recovery conditions and deriving compositional guarantees across multiple affine--activation stages remain natural directions for future work.

\subsection*{AI use statement}

Generative AI tools were used for language editing, literature search, and limited assistance with mathematical exposition and experimental planning. All mathematical arguments, experiments, and results were independently reviewed and verified by the authors. The authors take full responsibility for the final content of this work.

\section*{Reproducibility Statement}
The formal definitions, assumptions, and theorem statements are presented in Section~\ref{sec:directional_separability}, with complete proofs and supporting properties provided in Appendices~\ref{app:lsm_basic_properties}--\ref{app:preservation_recovery}. The practical D-LSM estimator is described in Section~\ref{sec:practical_dlsm_estimation} and Appendix~\ref{app:lsm_estimator}. Section~\ref{sec:experiments} summarizes the experimental protocols and results, while Appendices~\ref{app:experimental_details} and~\ref{app:additional_results} provide the data-generation and preprocessing procedures, model and estimator configurations, random-seed conventions, evaluation quantities, and complete supplementary results. An anonymized code release is available at \url{https://anonymous.4open.science/r/dlsm-iclr/}; it contains the implementations and entry scripts for the controlled theoretical experiments, ViT training, D-LSM estimation, and affine boundary-transport audit, together with instructions for reproducing the reported outputs.

\section*{Ethics Statement}
This work is theoretical and methodological, uses synthetic data and standard publicly available benchmarks, and involves no human-subject study or newly collected personal or sensitive data; we therefore identify no specific ethical concerns arising from this research.

\bibliography{iclr2027_conference}
\bibliographystyle{iclr2027_conference}

\appendix

\clearpage
\section*{Appendix Contents}
\begingroup
\small
\setlength{\parindent}{0pt}
\setlength{\parskip}{0pt}
\newcommand{\appsectionentry}[2]{\noindent\hyperref[#1]{\textbf{\ref*{#1}\quad #2}}\dotfill\hyperref[#1]{\pageref*{#1}}\par}
\newcommand{\appsubsectionentry}[2]{\noindent\hspace*{1.25em}\hyperref[#1]{\ref*{#1}\quad #2}\dotfill\hyperref[#1]{\pageref*{#1}}\par}

\appsectionentry{app:limitation}{Limitations and Future Directions}
\appsectionentry{app:lsm_basic_properties}{Technical Conventions and Basic Properties of D-LSM}
\appsubsectionentry{app:sample_collection_convention}{Sample-Collection Convention}
\appsubsectionentry{app:target_normalization}{Interpretation of the Target Normalization}
\appsubsectionentry{app:proof_supporting_halfspace}{Proof of the Supporting-Halfspace Characterization}
\appsubsectionentry{app:proof_accuracy_relation}{Relation to Optimal Affine Classification Accuracy}
\appsubsectionentry{app:proof_affine_invariance}{Proof of Affine Invariance}
\vspace{0.35em}

\appsectionentry{app:activation_properties}{Properties of Coordinatewise Gated Activations}
\appsubsectionentry{app:activation_preliminaries}{Scalar Preliminaries}
\appsubsectionentry{app:proof_activation_envelope}{Proof of the Region-wise Activation Deformation Bounds}
\appsubsectionentry{app:activation_examples}{ReLU, GELU, and SiLU}
\vspace{0.35em}

\appsectionentry{app:preservation_recovery}{Preservation, Recovery, and Certified D-LSM Improvement}
\appsubsectionentry{app:directional_projection_bounds}{Directional Projection Bounds}
\appsubsectionentry{app:proof_preservation_recovery}{Proof of the Preservation and Recovery Condition}
\appsubsectionentry{app:proof_lsm_improvement}{Proof of Certified D-LSM Improvement}
\appsubsectionentry{app:proof_affine_tube_recovery}{Proof of Affine-Tube Recovery}
\appsubsectionentry{app:proof_multiclass_recovery}{Proof of Multiclass Recovery}
\appsubsectionentry{app:comparison_an2015}{Comparison with Related Separation Guarantees}
\vspace{0.35em}

\appsectionentry{app:lsm_estimator}{Practical Estimation of D-LSM}
\appsubsectionentry{app:lsm_surrogate}{Penalty-Based Affine Search}
\appsubsectionentry{app:lsm_boundary_selection}{Best Feasible Boundary Selection}
\appsubsectionentry{app:lsm_constraint_control}{Constraint-Priority Control}
\vspace{0.35em}

\appsectionentry{app:experimental_details}{Experimental Details}
\appsubsectionentry{app:activation_experiment_setup}{Activation-Certificate Experiment}
\appsubsectionentry{app:affine_activation_setup}{Quantitative Affine--Activation Experiment}
\appsubsectionentry{app:tube_experiment_setup}{Affine-Tube Experiment}
\appsubsectionentry{app:neural_representation_setup}{Neural Representation Setup}
\appsubsectionentry{app:vit_experimental_details}{ViT Representation and D-LSM Estimation Details}
\appsubsectionentry{app:vit_invariance_details}{Affine Boundary Transport Details}
\vspace{0.35em}

\appsectionentry{app:additional_results}{Additional Experimental Results}
\appsubsectionentry{app:activation_experiment_results}{Activation-Certificate Results}
\appsubsectionentry{app:affine_activation_results}{Quantitative Affine--Activation Results}
\appsubsectionentry{app:tube_experiment_results}{Affine-Tube Results}
\appsubsectionentry{app:vit_full_results}{Expanded ViT D-LSM Results}
\appsubsectionentry{app:vit_invariance_results}{Affine Boundary Transport Results}
\endgroup

\section{Limitations and Future Directions}
\label{app:limitation}

\paragraph{Sufficient certificates and sharpness.}
The activation certificates track an optimal pre-activation supporting direction and use scalar envelopes to bound transformed projections and target support. Their sufficiency reflects both the slack in \(L_u^\sigma\), \(U_u^\sigma\), and \(\bar h_A^\sigma\) and the choice to evaluate recovery along this direction. The post-activation optimum can additionally exploit other directions. Preserving all previous exclusions converts certified recovery into a gain bound; net improvement can also occur when recovery outweighs re-entry. The controlled experiments exhibit bound attainment and quantify contributions from additional recovery and direction optimization (Appendix~\ref{app:affine_activation_results}). Natural refinements include tighter envelopes, optimized certificate directions, and joint bounds on recovery and re-entry.

\paragraph{Activation assumptions and extensions.}
The scalar analysis uses the gated form \(\rho_\sigma(t)=t\sigma(t)\), a nondecreasing gate in \([0,1]\) with limits \(0\) and \(1\) at negative and positive infinity, and an attained global minimum of \(\rho_\sigma\). These properties support a common envelope analysis for ReLU, GELU, and SiLU and enter the affine-tube scaling conditions. Extending the analysis to broader activation families calls for corresponding scalar envelopes and an examination of which paired-coordinate response properties suffice for constructive recovery.

\paragraph{Geometry-dependent constructive width.}
The sufficient width in Theorem~\ref{thm:affine_tube_recovery} depends on the ambient dimension, the number of aggregate-tube functions, and the dimension spanned by their normals together with the supporting direction. The geometry of \(A\) and the recovery set \(R\) enters through the admissible function families. On the common planar family in Appendix~\ref{app:comparison_an2015}, additional samples can be recovered at a fixed width of \(3\) because they satisfy the same tube condition. Further work could optimize the tube functions and coordinate sharing for a prescribed recovery target, and characterize the tradeoff between sufficient width and certifiable recovery.

\paragraph{Composition and broader network modules.}
Tracking supporting directions, exclusion and recovery sets, and certificate margins across successive affine--activation stages could yield compositional guarantees for deeper networks. The same geometric viewpoint also motivates analysis of normalization, residual connections, and attention. Normalization using fixed coordinatewise statistics and nonzero scales is an injective affine transformation and therefore inherits D-LSM invariance. Input-dependent normalization and attention call for bounds accounting for their input-dependent geometric transformations. Developing these bounds and analyzing their interaction with residual additions could extend the framework to target support and competing-sample intrusion within Transformer blocks.

\section{Technical Conventions and Basic Properties of D-LSM}
\label{app:lsm_basic_properties}

\subsection{Sample-Collection Convention}
\label{app:sample_collection_convention}

The main text uses set notation for finite sample collections to keep the geometric presentation concise. Formally, a sample collection \(S\) is regarded as a finite labeled family \(S=(s_i)_{i\in I_S}\subset\mathbb{R}^{d}\), where \(I_S\) is a finite index set. We write \(|S|:=|I_S|\). Thus, cardinalities always count sample labels rather than distinct feature values.

For a pointwise map \(g\), the transformed collection \(g(S):=(g(s_i))_{i\in I_S}\) retains the same labels. Consequently, if \(g(s_i)=g(s_j)\) for two distinct samples \(i\neq j\), both samples remain present and are counted separately. Likewise, for a set \(C\subseteq\mathbb{R}^{d}\), the notation \(S\cap C\) denotes the labeled subcollection \((s_i)_{i\in I_S:\,s_i\in C}\), so that
\[
|S\cap C|
=
\left|
\{i\in I_S:s_i\in C\}
\right|.
\]
In contrast, geometric constructions such as \(\operatorname{conv}(S)\) depend only on the corresponding feature values:
\[
\operatorname{conv}(S)
:=
\operatorname{conv}\{s_i:i\in I_S\}.
\]
We use the simpler set notation \(A,B,\rho_\sigma(A)\), and related expressions throughout the main text under this convention.

\subsection{Interpretation of the Target Normalization}
\label{app:target_normalization}

The normalization by \(|A|\) expresses unavoidable competing-sample intrusion relative to the size of the collection that must be preserved exactly. Since
\[
1-s(A,B)
=
\frac{m(A,B)}{|A|},
\]
the same absolute intrusion has different significance for target collections of different sizes. For example, one unavoidable competing sample relative to \(100\) target samples represents a larger target-relative intrusion than one relative to \(10{,}000\) target samples.

The signed scale has a corresponding direct interpretation:
\[
s(A,B)=1 \iff m(A,B)=0,
\qquad
s(A,B)=0 \iff m(A,B)=|A|,
\]
while
\[
s(A,B)<0 \iff m(A,B)>|A|.
\]
Thus, \(s=1\) is the exact-separability point, whereas \(s=0\) is a target-relative reference level at which the unavoidable competing intrusion equals the target sample count. Negative values indicate that the unavoidable intrusion exceeds this reference level. The normalization does not affect the optimizing halfspace for fixed \(|A|\), although D-LSM remains sample-size dependent across different target multiplicities.

\subsection{Proof of the Supporting-Halfspace Characterization}
\label{app:proof_supporting_halfspace}

\begin{proof}
Recall that
\[
\mathcal{H}_A
=
\left\{
H_{w,c}
=
\{x\in\mathbb{R}^{d}:w^\top x+c\ge0\}
:
w\neq0,\ A\subseteq H_{w,c}
\right\}.
\]
Since \(\mathcal{H}_A\) is nonempty and the intrusion count \(|B\cap H|\) takes values in the finite set \(\{0,\ldots,|B|\}\), there exists at least one \(H\in\mathcal{H}_A\) attaining the minimum in Definition~\ref{def:directional_lsm}. Let
\[
H
=
\{x\in\mathbb{R}^{d}:f(x)\ge0\},
\qquad
f(x)=w^\top x+c,
\]
be any optimal target-preserving halfspace, with \(w\neq0\). Since \(A\subseteq H\), define its minimum target slack
\[
\gamma
:=
\min_{a\in A}f(a)
\ge0.
\]
Because \(A\) is finite and nonempty, the minimum is attained by some \(a^\star\in A\).

Now translate the boundary toward the target samples by defining
\[
\widetilde f(x)
:=
f(x)-\gamma,
\qquad
\widetilde H
:=
\{x:\widetilde f(x)\ge0\}.
\]
For every \(a\in A\),
\[
\widetilde f(a)
=
f(a)-\gamma
\ge0,
\]
so \(A\subseteq\widetilde H\). Moreover,
\[
\widetilde f(a^\star)=0,
\]
and hence \(a^\star\in\partial\widetilde H\).

Since \(\gamma\ge0\),
\[
\widetilde H
=
\{x:f(x)\ge\gamma\}
\subseteq
\{x:f(x)\ge0\}
=
H.
\]
Therefore,
\[
|B\cap\widetilde H|
\le
|B\cap H|.
\]
Because \(H\) was already optimal, the inequality cannot be strict; otherwise \(\widetilde H\) would achieve a smaller competing-sample intrusion. Hence \(\widetilde H\) is also optimal.

Finally, \(\widetilde H\) is convex and contains \(A\), so
\[
\operatorname{conv}(A)\subseteq\widetilde H.
\]
At the same time, \(a^\star\in A\subseteq\operatorname{conv}(A)\) and \(a^\star\in\partial\widetilde H\). Thus,
\[
\partial\widetilde H
\cap
\operatorname{conv}(A)
\neq
\varnothing.
\]
Therefore, \(\partial\widetilde H\) is a supporting hyperplane of \(\operatorname{conv}(A)\), and \(\widetilde H\) is an optimal target-preserving supporting halfspace. 
\end{proof}

\subsection{Relation to Optimal Affine Classification Accuracy}
\label{app:proof_accuracy_relation}

Let \(m:=|A|\) and \(n:=|B|\). For an affine halfspace \(H\subset\mathbb{R}^{d}\), regard samples in \(A\cap H\) and \(B\setminus H\) as correctly classified, and define the optimal empirical affine classification accuracy by
\begin{equation}
\operatorname{acc}(A,B)
:=
\max_{H}
\frac{|A\cap H|+|B\setminus H|}{m+n},
\label{eq:appendix_affine_accuracy}
\end{equation}
where the maximum ranges over affine halfspaces with nonzero normal vectors. Since the objective takes values in the finite set \(\{0,1/(m+n),\ldots,1\}\), the maximum is attained.

\begin{proof}
Let
\[
k^\star
:=
\min_{H\in\mathcal{H}_A}|B\cap H|.
\]
By Definition~\ref{def:directional_lsm},
\[
s(A,B)=1-\frac{k^\star}{m}.
\]
Choose an optimal target-preserving halfspace \(H^\star\in\mathcal{H}_A\) attaining \(k^\star\). When \(H^\star\) is evaluated as a binary classifier, every target sample is classified correctly and exactly \(n-k^\star\) competing samples are classified correctly. Therefore,
\[
\operatorname{acc}(A,B)
\ge
\frac{m+n-k^\star}{m+n}
=
1-\frac{k^\star}{m+n}.
\]
Since \(m,n>0\),
\[
1-\frac{k^\star}{m+n}
\ge
1-\frac{k^\star}{m}
=
s(A,B),
\]
and hence
\begin{equation}
\operatorname{acc}(A,B)\ge s(A,B).
\label{eq:appendix_accuracy_lsm}
\end{equation}

If \(A\) and \(B\) are linearly separable, then \(k^\star=0\), so \(s(A,B)=1\). Since classification accuracy cannot exceed one, \(\operatorname{acc}(A,B)=s(A,B)=1\). Conversely, if \(A\) and \(B\) are not linearly separable, then \(k^\star>0\), and because \(n>0\),
\[
1-\frac{k^\star}{m+n}
>
1-\frac{k^\star}{m}
=
s(A,B).
\]
Together with the lower bound above, this gives \(\operatorname{acc}(A,B)>s(A,B)\). Therefore,
\[
\operatorname{acc}(A,B)=s(A,B)
\quad\Longleftrightarrow\quad
A\ \text{and}\ B\ \text{are linearly separable}.
\]
\end{proof}

\subsection{Proof of Affine Invariance}
\label{app:proof_affine_invariance}

The result follows from two elementary invariance properties: full-column-rank linear embeddings preserve all affine membership patterns relevant to D-LSM, and a common translation can be absorbed into the offset of an affine halfspace.

\begin{lemma}[Linear invariance]
\label{lem:linear_invariance}
Let \(L:\mathbb{R}^d\rightarrow\mathbb{R}^h\) be
\[
L(x)=Mx,
\]
where \(M\in\mathbb{R}^{h\times d}\), \(h\ge d\), and \(\operatorname{rank}(M)=d\). Then
\[
s(A,B)=s(L(A),L(B)).
\]
\end{lemma}

\begin{proof}
Since \(\operatorname{rank}(M)=d\), the map \(L\) is injective and
\[
M^\top:\mathbb{R}^h\rightarrow\mathbb{R}^d
\]
is surjective. Define
\[
k^\star
:=
\min_{H\in\mathcal H_A}|B\cap H|,
\qquad
\widetilde{k}^\star
:=
\min_{\widetilde H\in\mathcal H_{L(A)}}
|L(B)\cap\widetilde H|.
\]
It suffices to show \(k^\star=\widetilde{k}^\star\).

We first prove \(\widetilde{k}^\star\le k^\star\). Let
\[
H_{w,c}
=
\{x\in\mathbb{R}^d:w^\top x+c\ge0\},
\qquad
w\neq0,
\]
be any target-preserving halfspace in the original representation. Since \(M^\top\) is surjective, there exists \(v\in\mathbb{R}^h\) such that
\[
M^\top v=w.
\]
Because \(w\neq0\), necessarily \(v\neq0\). Define
\[
\widetilde H_{v,c}
=
\{y\in\mathbb{R}^h:v^\top y+c\ge0\}.
\]
For every \(x\in\mathbb{R}^d\),
\[
v^\top L(x)+c
=
v^\top Mx+c
=
(M^\top v)^\top x+c
=
w^\top x+c.
\]
Hence
\[
x\in H_{w,c}
\quad\Longleftrightarrow\quad
L(x)\in\widetilde H_{v,c}.
\]
Thus the two halfspaces induce exactly the same membership pattern on \(A\cup B\) and \(L(A)\cup L(B)\), respectively. Applying this construction to an optimal original-space halfspace gives
\[
\widetilde{k}^\star\le k^\star.
\]

For the reverse inequality, let
\[
\widetilde H_{v,c}
=
\{y\in\mathbb{R}^h:v^\top y+c\ge0\},
\qquad
v\neq0,
\]
be any target-preserving halfspace for \(L(A)\), and define
\[
w:=M^\top v.
\]
If \(w\neq0\), then
\[
H_{w,c}
=
\{x\in\mathbb{R}^d:w^\top x+c\ge0\}
\]
is an admissible original-space affine halfspace, and
\[
v^\top L(x)+c=w^\top x+c
\]
shows that it induces exactly the same membership pattern.

It remains to consider \(w=M^\top v=0\). In this case,
\[
v^\top L(x)+c=c
\]
for every \(x\in\mathbb{R}^d\). Since \(\widetilde H_{v,c}\) preserves the nonempty target collection \(L(A)\), necessarily \(c\ge0\). Therefore every transformed sample lies in \(\widetilde H_{v,c}\), including every sample in \(L(B)\). The same all-included membership pattern can be realized in the original representation by choosing any nonzero \(w_0\in\mathbb{R}^d\) and a sufficiently large \(c_0\) such that
\[
w_0^\top x+c_0\ge0
\qquad
\text{for all }x\in A\cup B,
\]
which is possible because \(A\cup B\) is finite. Hence no transformed-space halfspace achieves a competing-sample intrusion unavailable in the original representation, and
\[
k^\star\le\widetilde{k}^\star.
\]

Combining the two inequalities yields
\[
k^\star=\widetilde{k}^\star.
\]
Under the sample-collection convention of Appendix~\ref{app:sample_collection_convention}, the embedding preserves target-sample multiplicities, so \(|L(A)|=|A|\). Therefore,
\[
s(L(A),L(B))
=
1-\frac{\widetilde{k}^\star}{|A|}
=
1-\frac{k^\star}{|A|}
=
s(A,B).
\]
\end{proof}

\begin{lemma}[Translation invariance]
\label{lem:translation_invariance}
For any \(b\in\mathbb{R}^d\), let
\[
\tau_b(x)=x+b.
\]
Then
\[
s(A,B)=s(\tau_b(A),\tau_b(B)).
\]
\end{lemma}

\begin{proof}
Consider any affine halfspace
\[
H_{w,c}
=
\{x\in\mathbb{R}^d:w^\top x+c\ge0\},
\qquad
w\neq0.
\]
Define
\[
\widetilde H_{w,\widetilde c}
=
\{y\in\mathbb{R}^d:w^\top y+\widetilde c\ge0\},
\qquad
\widetilde c:=c-w^\top b.
\]
For every \(x\in\mathbb{R}^d\),
\[
w^\top\tau_b(x)+\widetilde c
=
w^\top(x+b)+c-w^\top b
=
w^\top x+c.
\]
Hence
\[
x\in H_{w,c}
\quad\Longleftrightarrow\quad
\tau_b(x)\in\widetilde H_{w,\widetilde c}.
\]
Thus every target-preserving affine membership pattern on \(A\cup B\) is reproduced exactly after translation. Applying the same argument to the inverse translation \(\tau_{-b}\) gives the converse correspondence. Therefore the minimum competing-sample intrusion is unchanged, and
\[
s(A,B)=s(\tau_b(A),\tau_b(B)).
\]
\end{proof}

\begin{proof}[Proof of Theorem~\ref{thm:affine_invariance}]
Let
\[
T(x)=Mx+b
\]
with \(M\in\mathbb{R}^{h\times d}\), \(h\ge d\), and \(\operatorname{rank}(M)=d\). Write
\[
T=\tau_b\circ L,
\qquad
L(x)=Mx,
\]
where \(\tau_b(y)=y+b\) acts on \(\mathbb{R}^h\).

By Lemma~\ref{lem:linear_invariance},
\[
s(A,B)=s(L(A),L(B)).
\]
Applying Lemma~\ref{lem:translation_invariance} in \(\mathbb{R}^h\) gives
\[
s(L(A),L(B))
=
s(\tau_b(L(A)),\tau_b(L(B)))
=
s(T(A),T(B)).
\]
Therefore,
\[
s(A,B)=s(T(A),T(B)).
\]
\end{proof}

\section{Properties of Coordinatewise Gated Activations}
\label{app:activation_properties}

\subsection{Scalar Preliminaries}
\label{app:activation_preliminaries}

Recall that
\[
\rho_\sigma(t)=t\sigma(t),
\]
where \(\sigma:\mathbb{R}\rightarrow[0,1]\) is nondecreasing, with \(\sigma(t)\rightarrow0\) as \(t\rightarrow-\infty\) and \(\sigma(t)\rightarrow1\) as \(t\rightarrow+\infty\). We first collect several scalar properties used throughout the subsequent analysis.

\begin{lemma}[Basic scalar properties]
\label{lem:activation_scalar_properties}
For any coordinatewise gated activation \(\rho_\sigma(t)=t\sigma(t)\), the following properties hold:
\begin{enumerate}
\item If \(t\le0\), then
\[
t\le\rho_\sigma(t)\le0.
\]
\item If \(t\ge0\), then
\[
0\le\rho_\sigma(t)\le t.
\]
\item The function \(\rho_\sigma\) is nondecreasing on \([0,\infty)\). More precisely, for \(0\le r\le t\),
\[
\rho_\sigma(t)\ge t\sigma(r)\ge\rho_\sigma(r).
\]
\item Define the displacement
\[
\delta_\sigma(t):=\rho_\sigma(t)-t.
\]
Then \(\delta_\sigma\) is nonincreasing on \((-\infty,0]\).
\end{enumerate}
\end{lemma}

\begin{proof}
For \(t\le0\), the condition \(0\le\sigma(t)\le1\) gives
\[
t\le t\sigma(t)\le0,
\]
which proves the first claim. Similarly, for \(t\ge0\),
\[
0\le t\sigma(t)\le t,
\]
proving the second claim.

For the third claim, let \(0\le r\le t\). Since \(\sigma\) is nondecreasing,
\[
\sigma(t)\ge\sigma(r).
\]
Multiplying by \(t\ge0\) gives
\[
\rho_\sigma(t)
=
t\sigma(t)
\ge
t\sigma(r).
\]
Since \(t\ge r\) and \(\sigma(r)\ge0\),
\[
t\sigma(r)\ge r\sigma(r)=\rho_\sigma(r).
\]
Thus \(\rho_\sigma\) is nondecreasing on \([0,\infty)\).

Finally, for \(t\le0\),
\[
\delta_\sigma(t)
=
t\sigma(t)-t
=
(-t)(1-\sigma(t)).
\]
Let \(s\le t\le0\). Then \(-s\ge-t\ge0\), while monotonicity of \(\sigma\) gives
\[
1-\sigma(s)\ge1-\sigma(t)\ge0.
\]
Therefore,
\[
\delta_\sigma(s)
=
(-s)(1-\sigma(s))
\ge
(-t)(1-\sigma(t))
=
\delta_\sigma(t),
\]
so \(\delta_\sigma\) is nonincreasing on \((-\infty,0]\).
\end{proof}

We additionally assume, as in Section~\ref{sec:activation_geometry}, that \(\rho_\sigma\) attains a global minimum. Since \(\rho_\sigma(0)=0\) and \(\rho_\sigma(t)\ge0\) for all \(t\ge0\), at least one global minimizer can be chosen in \((-\infty,0]\). Fix such a minimizer \(\tau\le0\) and let
\[
\mu:=\rho_\sigma(\tau)\le0.
\]
By Lemma~\ref{lem:activation_scalar_properties},
\[
\delta_\sigma(\tau)=\mu-\tau.
\]
Since \(\delta_\sigma\) is nonincreasing on \((-\infty,0]\), we further have
\begin{equation}
0\le\delta_\sigma(t)\le\mu-\tau
\qquad
\text{for all }t\in[\tau,0].
\label{eq:appendix_transition_displacement}
\end{equation}
This displacement bound is the key ingredient for controlling the transition region \(I_0\).

\subsection{Proof of the Region-wise activation deformation bounds}
\label{app:proof_activation_envelope}

\begin{proof}
Recall
\[
I_+:=[0,\infty),
\qquad
I_0:=[\tau,0),
\qquad
I_-:=(-\infty,\tau),
\]
and let
\[
\alpha_\sigma^+
:=
\inf_{t>0}\sigma(t)\in[0,1].
\]
We consider the three regions separately.

\paragraph{Positive region \(I_+\).}
For \(t=0\), we have \(\rho_\sigma(0)=0\), so the claimed bounds hold with equality. Now let \(t>0\). By the definition of \(\alpha_\sigma^+\), we have \(\alpha_\sigma^+\le\sigma(t)\le1\), and therefore
\[
\alpha_\sigma^+t
\le
t\sigma(t)
=
\rho_\sigma(t)
\le
t.
\]
Hence
\[
\rho_\sigma(t)\in[\alpha_\sigma^+t,t]
\qquad
\text{for every }t\in I_+.
\]

Moreover, Lemma~\ref{lem:activation_scalar_properties} shows that \(\rho_\sigma\) is nondecreasing on \(I_+\). In particular, for any \(0\le r\le t\),
\[
\rho_\sigma(t)
=
t\sigma(t)
\ge
t\sigma(r)
\ge
r\sigma(r)
=
\rho_\sigma(r),
\]
which also yields the anchored bound
\[
\sigma(r)t
\le
\rho_\sigma(t)
\le
t.
\]

\paragraph{Transition region \(I_0\).}
Let \(t\in[\tau,0)\). Since \(t\le0\), Lemma~\ref{lem:activation_scalar_properties} gives
\[
t\le\rho_\sigma(t)\le0.
\]
Because \(\mu\) is the global minimum of \(\rho_\sigma\),
\[
\rho_\sigma(t)\ge\mu.
\]
Combining these lower bounds yields
\[
\rho_\sigma(t)\ge\max\{t,\mu\}.
\]

For the remaining upper bound, write
\[
\rho_\sigma(t)=t+\delta_\sigma(t).
\]
Since \(t\ge\tau\) and \(\delta_\sigma\) is nonincreasing on the negative half-line,
\[
\delta_\sigma(t)
\le
\delta_\sigma(\tau)
=
\mu-\tau.
\]
Therefore,
\[
\rho_\sigma(t)
\le
t+\mu-\tau.
\]
Together with \(\rho_\sigma(t)\le0\), this gives
\[
\rho_\sigma(t)
\le
\min\{0,t+\mu-\tau\}.
\]
Hence
\[
\rho_\sigma(t)
\in
\left[
\max\{t,\mu\},
\min\{0,t+\mu-\tau\}
\right].
\]

\paragraph{Negative region \(I_-\).}
Let \(t<\tau\). Since \(\mu\) is the global minimum,
\[
\rho_\sigma(t)\ge\mu.
\]
On the other hand, \(t<0\) and \(\sigma(t)\ge0\), so
\[
\rho_\sigma(t)=t\sigma(t)\le0.
\]
Thus
\[
\rho_\sigma(t)\in[\mu,0].
\]

Combining the three cases establishes the region-wise envelopes in \eqref{eq:activation_envelopes}, together with the positive-region monotonicity and anchored bound stated in Lemma~\ref{lem:activation_envelope}.
\end{proof}

\subsection{ReLU, GELU, and SiLU}
\label{app:activation_examples}

We verify that ReLU, GELU, and SiLU satisfy the assumptions used in Section~\ref{sec:activation_geometry}. In each case, the activation can be written as \(\rho_\sigma(t)=t\sigma(t)\) with a nondecreasing gate \(\sigma:\mathbb{R}\rightarrow[0,1]\) satisfying the required limits.

\paragraph{ReLU.}
For ReLU,
\[
\rho_{\mathrm{ReLU}}(t)=\max\{0,t\}=t\,\sigma_{\mathrm{ReLU}}(t),
\qquad
\sigma_{\mathrm{ReLU}}(t):=
\begin{cases}
0, & t<0,\\
1, & t\ge0.
\end{cases}
\]
The gate is nondecreasing, with limits \(0\) and \(1\) at \(-\infty\) and \(+\infty\), respectively. Moreover, \(\rho_{\mathrm{ReLU}}(t)\ge0\) for all \(t\), and its global minimum is \(0\). We choose
\[
\tau=0,
\qquad
\mu=0,
\qquad
\alpha_{\mathrm{ReLU}}^+=1.
\]
With this choice, \(I_0=[\tau,0)=\varnothing\), while \(I_-=(-\infty,0)\) is mapped exactly to \(0\). Thus the region-wise envelopes in Lemma~\ref{lem:activation_envelope} reduce to the exact ReLU mapping.

\paragraph{SiLU.}
The SiLU activation is
\[
\rho_{\mathrm{SiLU}}(t)
=
t\,\sigma_{\mathrm{SiLU}}(t),
\qquad
\sigma_{\mathrm{SiLU}}(t)
=
\frac{1}{1+e^{-t}}.
\]
The logistic gate is strictly increasing, takes values in \((0,1)\), and satisfies
\[
\lim_{t\to-\infty}\sigma_{\mathrm{SiLU}}(t)=0,
\qquad
\lim_{t\to+\infty}\sigma_{\mathrm{SiLU}}(t)=1.
\]
Since the gate is continuous and increasing,
\[
\alpha_{\mathrm{SiLU}}^+
=
\inf_{t>0}\sigma_{\mathrm{SiLU}}(t)
=
\frac{1}{2}.
\]

It remains to verify the existence of a global minimum. The function \(\rho_{\mathrm{SiLU}}\) is continuous, satisfies \(\rho_{\mathrm{SiLU}}(t)<0\) for every \(t<0\), and obeys
\[
\lim_{t\to-\infty}t\sigma_{\mathrm{SiLU}}(t)=0,
\qquad
\lim_{t\to+\infty}t\sigma_{\mathrm{SiLU}}(t)=+\infty.
\]
The first limit follows from \(\sigma_{\mathrm{SiLU}}(t)\sim e^t\) as \(t\to-\infty\). Hence, choosing any \(t_0<0\), we have \(\rho_{\mathrm{SiLU}}(t_0)<0\), while sufficiently large positive and negative inputs have activation values strictly greater than \(\rho_{\mathrm{SiLU}}(t_0)\). Continuity therefore implies that the global minimum is attained on a bounded interval. We denote a chosen minimizer by \(\tau_{\mathrm{SiLU}}<0\) and its value by \(\mu_{\mathrm{SiLU}}<0\).

\paragraph{GELU.}
For the exact GELU activation,
\[
\rho_{\mathrm{GELU}}(t)
=
t\,\Phi(t),
\qquad
\sigma_{\mathrm{GELU}}(t)
=
\Phi(t),
\]
where \(\Phi\) is the standard Gaussian cumulative distribution function. Since \(\Phi\) is strictly increasing with range \((0,1)\),
\[
\lim_{t\to-\infty}\Phi(t)=0,
\qquad
\lim_{t\to+\infty}\Phi(t)=1,
\qquad
\alpha_{\mathrm{GELU}}^+
=
\inf_{t>0}\Phi(t)
=
\frac{1}{2}.
\]

The function \(\rho_{\mathrm{GELU}}\) is continuous and negative for every \(t<0\). Furthermore,
\[
\lim_{t\to+\infty}t\Phi(t)=+\infty.
\]
For the negative tail, write \(x=-t>0\). Then
\[
|t\Phi(t)|
=
x\Phi(-x)
=
x\bigl(1-\Phi(x)\bigr).
\]
Using the standard Gaussian tail bound \(1-\Phi(x)\le\phi(x)/x\), where \(\phi\) is the standard normal density, gives
\[
0
\le
x\bigl(1-\Phi(x)\bigr)
\le
\phi(x)
\longrightarrow
0,
\]
and therefore
\[
\lim_{t\to-\infty}t\Phi(t)=0.
\]
As for SiLU, any \(t_0<0\) satisfies \(\rho_{\mathrm{GELU}}(t_0)<0\), while the activation approaches \(0\) from below as \(t\to-\infty\) and diverges to \(+\infty\) as \(t\to+\infty\). Continuity therefore guarantees the existence of a global minimizer \(\tau_{\mathrm{GELU}}<0\), with \(\mu_{\mathrm{GELU}}:=\rho_{\mathrm{GELU}}(\tau_{\mathrm{GELU}})<0\).

Thus, ReLU, exact GELU, and SiLU all satisfy the assumptions underlying Lemma~\ref{lem:activation_envelope}. Their positive-region constants are
\[
\alpha_{\mathrm{ReLU}}^+=1,
\qquad
\alpha_{\mathrm{GELU}}^+
=
\alpha_{\mathrm{SiLU}}^+
=
\frac{1}{2}.
\]

\section{Preservation, Recovery, and Certified D-LSM Improvement}
\label{app:preservation_recovery}

\subsection{Directional Projection Bounds}
\label{app:directional_projection_bounds}

We first translate the scalar activation envelopes of Lemma~\ref{lem:activation_envelope} into bounds on arbitrary linear projections. For each coordinate of \(x\in\mathbb{R}^{d}\), let
\[
\ell_\sigma(x_j)
\le
\rho_\sigma(x_j)
\le
r_\sigma(x_j)
\]
denote the corresponding lower and upper endpoints in \eqref{eq:activation_envelopes}. For a direction \(u\in\mathbb{R}^{d}\), recall
\[
L_u^\sigma(x)
:=
\sum_{j=1}^{d}
\min\{u_j\ell_\sigma(x_j),u_jr_\sigma(x_j)\},
\qquad
U_u^\sigma(x)
:=
\sum_{j=1}^{d}
\max\{u_j\ell_\sigma(x_j),u_jr_\sigma(x_j)\}.
\]

\begin{lemma}[Directional projection bounds]
\label{lem:directional_projection_bounds}
For every \(x,u\in\mathbb{R}^{d}\),
\begin{equation}
L_u^\sigma(x)
\le
u^\top\rho_\sigma(x)
\le
U_u^\sigma(x).
\label{eq:appendix_directional_projection_bounds}
\end{equation}
Consequently, for any finite target sample collection \(A\),
\begin{equation}
h_{\rho_\sigma(A)}(u)
:=
\max_{a\in A}u^\top\rho_\sigma(a)
\le
\bar h_A^\sigma(u)
:=
\max_{a\in A}U_u^\sigma(a).
\label{eq:appendix_transformed_support_bound}
\end{equation}
\end{lemma}

\begin{proof}
Fix \(x,u\in\mathbb{R}^{d}\). For each coordinate \(j\),
\[
\ell_\sigma(x_j)
\le
\rho_\sigma(x_j)
\le
r_\sigma(x_j).
\]
If \(u_j\ge0\), multiplication preserves the order and gives
\[
u_j\ell_\sigma(x_j)
\le
u_j\rho_\sigma(x_j)
\le
u_jr_\sigma(x_j).
\]
If \(u_j<0\), multiplication reverses the order:
\[
u_jr_\sigma(x_j)
\le
u_j\rho_\sigma(x_j)
\le
u_j\ell_\sigma(x_j).
\]
Hence, in either case,
\[
\min\{u_j\ell_\sigma(x_j),u_jr_\sigma(x_j)\}
\le
u_j\rho_\sigma(x_j)
\le
\max\{u_j\ell_\sigma(x_j),u_jr_\sigma(x_j)\}.
\]
Summing these inequalities over \(j=1,\ldots,d\) yields
\[
L_u^\sigma(x)
\le
\sum_{j=1}^{d}u_j\rho_\sigma(x_j)
=
u^\top\rho_\sigma(x)
\le
U_u^\sigma(x),
\]
which proves \eqref{eq:appendix_directional_projection_bounds}.

Applying the upper bound to each \(a\in A\) gives
\[
u^\top\rho_\sigma(a)
\le
U_u^\sigma(a).
\]
Taking the maximum over \(A\),
\[
h_{\rho_\sigma(A)}(u)
=
\max_{a\in A}u^\top\rho_\sigma(a)
\le
\max_{a\in A}U_u^\sigma(a)
=
\bar h_A^\sigma(u),
\]
which proves \eqref{eq:appendix_transformed_support_bound}.
\end{proof}

The quantities \(L_u^\sigma(x)\) and \(U_u^\sigma(x)\) are therefore certified projection bounds obtained solely from the pre-activation point and the structural activation envelopes. In particular, \(\bar h_A^\sigma(u)\) is an upper bound on the true support of the transformed target samples along \(u\). These bounds need not be tight for a particular sample, but their validity does not depend on observing the post-activation representation, which allows them to be used as sufficient certificates for preservation and recovery.

\subsection{Proof of the Preservation and Recovery Condition}
\label{app:proof_preservation_recovery}
\begin{proof}
Recall from Lemma~\ref{lem:directional_projection_bounds} that, for every \(x\in\mathbb{R}^{d}\),
\[
L_u^\sigma(x)
\le
u^\top\rho_\sigma(x)
\le
U_u^\sigma(x),
\]
and that
\[
\max_{a\in A}u^\top\rho_\sigma(a)
\le
\bar h_A^\sigma(u)
:=
\max_{a\in A}U_u^\sigma(a).
\]
Consider the post-activation halfspace
\begin{equation}
\bar H_u^\sigma
:=
\left\{
y\in\mathbb{R}^{d}:
u^\top y\le\bar h_A^\sigma(u)
\right\}.
\label{eq:appendix_post_activation_halfspace}
\end{equation}
Since \(u\neq0\), \(\bar H_u^\sigma\) is an admissible affine halfspace.

For every \(a\in A\),
\[
u^\top\rho_\sigma(a)
\le
U_u^\sigma(a)
\le
\bar h_A^\sigma(u),
\]
and hence
\begin{equation}
\rho_\sigma(A)\subseteq\bar H_u^\sigma.
\label{eq:appendix_target_preservation}
\end{equation}

Suppose first that \(B_b(f^\star)\neq\varnothing\). Under the preservation condition \eqref{eq:preservation_condition}, every \(b\in B_b(f^\star)\) satisfies
\[
L_u^\sigma(b)
\ge
\min_{b'\in B_b(f^\star)}L_u^\sigma(b')
>
\bar h_A^\sigma(u).
\]
Therefore,
\[
u^\top\rho_\sigma(b)
\ge
L_u^\sigma(b)
>
\bar h_A^\sigma(u),
\]
so
\begin{equation}
\rho_\sigma(B_b(f^\star))
\cap
\bar H_u^\sigma
=
\varnothing.
\label{eq:appendix_preserved_exclusion}
\end{equation}
If \(B_b(f^\star)=\varnothing\), this conclusion holds trivially.

Now let \(b\in R_a(u)\). By definition,
\[
L_u^\sigma(b)
>
\bar h_A^\sigma(u).
\]
Again,
\[
u^\top\rho_\sigma(b)
\ge
L_u^\sigma(b)
>
\bar h_A^\sigma(u),
\]
and therefore
\begin{equation}
\rho_\sigma(R_a(u))
\cap
\bar H_u^\sigma
=
\varnothing.
\label{eq:appendix_recovered_exclusion}
\end{equation}

Combining \eqref{eq:appendix_target_preservation}, \eqref{eq:appendix_preserved_exclusion}, and \eqref{eq:appendix_recovered_exclusion}, the same halfspace \(\bar H_u^\sigma\) preserves all transformed target samples while excluding all transformed samples originating from
\[
B_b(f^\star)\cup R_a(u).
\]
This proves Proposition~\ref{prop:preservation_recovery}.
\end{proof}

\begin{remark}[An intuitive interpretation]
\label{rem:margin_interpretation}
Imagine viewing all samples along the chosen direction. The activation bounds tell us how far the target samples could extend in that direction after activation. We place a common
cutoff at this conservative upper bound, so every transformed target sample is guaranteed to remain on the target side. For a competing sample, we then check its lower bound: if even this bound lies strictly beyond the cutoff, the sample is guaranteed to be excluded.

Applying this check to every previously excluded sample preserves all existing exclusions. Any previously admitted sample that also passes the check is additionally recovered. The gap between the bounds provides a safety margin for these conclusions. Because the bounds are conservative, however, a sample may still be excluded even when it does not pass the check. The condition therefore guarantees separation when satisfied, but does not imply a loss of separation when it fails.
\end{remark}

\subsection{Proof of Certified D-LSM Improvement}
\label{app:proof_lsm_improvement}

\begin{proof}
Let
\[
k^\star
:=
\min_{H\in\mathcal{H}_A}|B\cap H|.
\]
Since \(H^\star\) is an optimal target-preserving halfspace and
\[
B_a(f^\star)
=
B\cap H^\star,
\]
we have
\begin{equation}
k^\star
=
|B_a(f^\star)|.
\label{eq:appendix_original_intrusion}
\end{equation}

Let
\[
\bar A:=\rho_\sigma(A),
\qquad
\bar B:=\rho_\sigma(B),
\]
and define
\[
\bar k^\star
:=
\min_{\bar H\in\mathcal{H}_{\bar A}}
|\bar B\cap\bar H|.
\]

By Proposition~\ref{prop:preservation_recovery}, the halfspace
\[
\bar H_u^\sigma
=
\{y\in\mathbb{R}^{d}:u^\top y\le\bar h_A^\sigma(u)\}
\]
contains all samples in \(\bar A\), preserves the exclusion of all samples in \(B_b(f^\star)\), and additionally excludes all samples in \(R_a(u)\). Since
\[
R_a(u)\subseteq B_a(f^\star)
\]
and
\[
B=B_b(f^\star)\sqcup B_a(f^\star),
\]
the only competing samples that may remain inside
\(\bar H_u^\sigma\) originate from
\[
B_a(f^\star)\setminus R_a(u).
\]
Therefore,
\[
|\bar B\cap\bar H_u^\sigma|
\le
|B_a(f^\star)|-|R_a(u)|.
\]
Using \eqref{eq:appendix_original_intrusion},
\[
|\bar B\cap\bar H_u^\sigma|
\le
k^\star-|R_a(u)|.
\]
Since \(\bar H_u^\sigma\in\mathcal{H}_{\bar A}\),
\[
\bar k^\star
\le
k^\star-|R_a(u)|.
\]

Under the sample-collection convention of Appendix~\ref{app:sample_collection_convention},
\[
|\bar A|=|A|.
\]
Hence
\[
\begin{aligned}
s(\bar A,\bar B)
&=
1-\frac{\bar k^\star}{|A|}\\
&\ge
1-\frac{k^\star-|R_a(u)|}{|A|}\\
&=
s(A,B)
+
\frac{|R_a(u)|}{|A|}.
\end{aligned}
\]
Thus,
\[
s(\bar A,\bar B)\ge s(A,B),
\]
with strict improvement whenever \(R_a(u)\neq\varnothing\).
\end{proof}

\subsection{Proof of Affine-Tube Recovery}
\label{app:proof_affine_tube_recovery}
\begin{proof}
The proof is constructive. The affine map combines an anchor coordinate \(\lambda q\), paired coordinates \(\pm\kappa g_k^0\), and rank-completion coordinates; the paired coordinates cancel along \(v\) before activation, so the lifted halfspace reproduces the original optimum. After activation, the paired envelopes separate \(A\) from \(R\): \(\kappa\) establishes the recovery gap, while \(\lambda\) preserves the exclusion of \(B_b(f^\star)\). Proposition~\ref{prop:preservation_recovery} and Theorem~\ref{thm:lsm_improvement} then convert these certificates into the claimed D-LSM improvement.

Recall that \(q(x)=u^\top x-h_A(u)\). Since \(H^\star\) is the optimal supporting halfspace from Theorem~\ref{thm:supporting_halfspace}, its score satisfies
\(f^\star(x)=-\|w^\star\|q(x)\). Consequently,
\[
q(x)\le0
\quad\text{for }x\in A\cup B_a(f^\star),
\qquad
q(x)>0
\quad\text{for }x\in B_b(f^\star),
\]
and \(\max_{a\in A}q(a)=0\). Moreover, \(0\le\varepsilon<\alpha_\sigma^+D_A\) implies \(\alpha_\sigma^+>0\).

We first derive the paired-coordinate bounds used below. For \(t\ge0\), Lemma~\ref{lem:activation_envelope} gives
\(\ell_\sigma(t)=\alpha_\sigma^+t\), \(r_\sigma(t)=t\),
\(\ell_\sigma(-t)\ge\mu\), and \(r_\sigma(-t)\le0\).
Exchanging \(t\) and \(-t\) therefore yields, for every \(t\in\mathbb R\),
\begin{equation}
\ell_\sigma(t)+\ell_\sigma(-t)
\ge
\alpha_\sigma^+|t|+\mu,
\qquad
r_\sigma(t)+r_\sigma(-t)
\le
|t|.
\label{eq:affine_tube_pair_bounds}
\end{equation}

Construct \(T\) using the coordinates specified in
Theorem~\ref{thm:affine_tube_recovery}. The coefficient vectors of the additional rank-completion coordinates are chosen to form an orthonormal basis of
\(\operatorname{span}\{u,p_1,\ldots,p_m\}^{\perp}\); any remaining coordinates are set to zero. Hence the row space of the linear part \(M\) contains both
\(\operatorname{span}\{u,p_1,\ldots,p_m\}\) and its orthogonal complement. Since \(\lambda,\kappa>0\),
\[
\operatorname{rank}(M)=d.
\]
Thus \(T\) is injective, and Theorem~\ref{thm:affine_invariance} gives
\begin{equation}
s(T(A),T(B))=s(A,B).
\label{eq:affine_tube_affine_invariance}
\end{equation}

Define
\[
v:=
\bigl(
1,
\underbrace{-1,\ldots,-1}_{2m},
\underbrace{0,\ldots,0}_{h-1-2m}
\bigr)^\top,
\qquad
\widetilde u:=\frac{v}{\|v\|},
\qquad
\|v\|=\sqrt{1+2m}.
\]
The paired coordinates cancel before activation, so
\begin{equation}
v^\top T(x)
=
\lambda q(x)
-\sum_{k=1}^{m}\kappa g_k^0(x)
+\sum_{k=1}^{m}\kappa g_k^0(x)
=
\lambda q(x).
\label{eq:affine_tube_pre_activation_score}
\end{equation}
Therefore the score \(-v^\top y\) induces on \(T(A)\cup T(B)\) exactly the same membership pattern as \(f^\star\) induces on \(A\cup B\). Its halfspace
\(\{y:v^\top y\le0\}\) contains \(T(A)\), admits \(T(B_a(f^\star))\), and excludes \(T(B_b(f^\star))\). Since
\(\max_{a\in A}v^\top T(a)=0\), its boundary supports \(T(A)\); by
\eqref{eq:affine_tube_affine_invariance}, it is also optimal. Its outward unit normal is precisely \(\widetilde u\).

Since \(\widetilde u=v/\|v\|\) and \(\|v\|>0\), the definitions in \eqref{eq:directional_projection_envelopes} imply
\[
L_{\widetilde u}^\sigma(T(x))
=
\frac{L_v^\sigma(T(x))}{\|v\|},
\qquad
U_{\widetilde u}^\sigma(T(x))
=
\frac{U_v^\sigma(T(x))}{\|v\|},
\qquad
\bar h_{T(A)}^\sigma(\widetilde u)
=
\frac{\bar h_{T(A)}^\sigma(v)}{\|v\|}.
\]
Hence a strict certificate along \(v\) is equivalent to the corresponding certificate along \(\widetilde u\). We may therefore work directly with \(v\).

Using the signs of the coordinates of \(v\) in
\eqref{eq:directional_projection_envelopes}, we obtain, for every \(x\in\mathbb R^d\),
\begin{align}
L_v^\sigma(T(x))
&=
\ell_\sigma(\lambda q(x))
-
\sum_{k=1}^{m}
\left[
r_\sigma(\kappa g_k^0(x))
+
r_\sigma(-\kappa g_k^0(x))
\right],
\label{eq:affine_tube_lower_projection}\\
U_v^\sigma(T(x))
&=
r_\sigma(\lambda q(x))
-
\sum_{k=1}^{m}
\left[
\ell_\sigma(\kappa g_k^0(x))
+
\ell_\sigma(-\kappa g_k^0(x))
\right].
\label{eq:affine_tube_upper_projection}
\end{align}

For \(a\in A\), we have \(q(a)\le0\), and therefore
\(r_\sigma(\lambda q(a))\le0\). Applying
\eqref{eq:affine_tube_pair_bounds} to
\eqref{eq:affine_tube_upper_projection} gives
\[
U_v^\sigma(T(a))
\le
-\alpha_\sigma^+\kappa D_G(a)-m\mu
\le
-\alpha_\sigma^+\kappa D_A-m\mu.
\]
Taking the maximum over \(A\), we obtain
\begin{equation}
\bar h_{T(A)}^\sigma(v)
\le
-\alpha_\sigma^+\kappa D_A-m\mu.
\label{eq:affine_tube_target_bound}
\end{equation}

Now let \(r\in R\). Since \(R\subseteq B_a(f^\star)\), we have
\(q(r)\le0\), so \(\ell_\sigma(\lambda q(r))\ge\mu\). Equations
\eqref{eq:affine_tube_pair_bounds} and
\eqref{eq:affine_tube_lower_projection} yield
\[
L_v^\sigma(T(r))
\ge
\mu-\kappa D_G(r)
\ge
\mu-\kappa\varepsilon.
\]
Combining this with \eqref{eq:affine_tube_target_bound} gives
\begin{equation}
L_v^\sigma(T(r))
-
\bar h_{T(A)}^\sigma(v)
\ge
\kappa(\alpha_\sigma^+D_A-\varepsilon)
+(m+1)\mu.
\label{eq:affine_tube_recovery_gap}
\end{equation}
To make the right-hand side strictly positive, it is sufficient to require
\[
\kappa(\alpha_\sigma^+D_A-\varepsilon)
>
-(m+1)\mu.
\]
Since \(\alpha_\sigma^+D_A-\varepsilon>0\), division preserves the inequality and gives exactly
\begin{equation}
\kappa>
\frac{-(m+1)\mu}
{\alpha_\sigma^+D_A-\varepsilon}.
\label{eq:affine_tube_kappa_bound}
\end{equation}
Under \eqref{eq:affine_tube_kappa_bound},
\[
L_v^\sigma(T(r))
>
\bar h_{T(A)}^\sigma(v),
\qquad r\in R.
\]
Positive homogeneity then gives the same strict inequality along
\(\widetilde u\).

It remains to preserve the exclusion of \(B_b(f^\star)\). Let
\(b\in B_b(f^\star)\). Since \(q(b)>0\), the positive-region envelope gives
\(\ell_\sigma(\lambda q(b))
=\alpha_\sigma^+\lambda q(b)\). Therefore,
\[
L_v^\sigma(T(b))
\ge
\alpha_\sigma^+\lambda q(b)-\kappa D_G(b).
\]
Combining this inequality with
\eqref{eq:affine_tube_target_bound}, we obtain
\begin{equation}
L_v^\sigma(T(b))
-
\bar h_{T(A)}^\sigma(v)
\ge
\alpha_\sigma^+\lambda q(b)
-\kappa D_G(b)
+\alpha_\sigma^+\kappa D_A
+m\mu.
\label{eq:affine_tube_preservation_gap}
\end{equation}
A sufficient condition for the right-hand side of
\eqref{eq:affine_tube_preservation_gap} to be strictly positive is
\[
\alpha_\sigma^+\lambda q(b)
>
\kappa D_G(b)-\alpha_\sigma^+\kappa D_A-m\mu.
\]
Because \(\alpha_\sigma^+q(b)>0\), this is equivalent to
\[
\lambda>
\frac{
\kappa D_G(b)-\alpha_\sigma^+\kappa D_A-m\mu
}{
\alpha_\sigma^+q(b)
}.
\]
Requiring this inequality simultaneously for all
\(b\in B_b(f^\star)\), together with \(\lambda>0\), gives
\begin{equation}
\lambda>
\max\left\{
0,\,
\max_{b\in B_b(f^\star)}
\frac{
\kappa D_G(b)-\alpha_\sigma^+\kappa D_A-m\mu
}{
\alpha_\sigma^+q(b)
}
\right\}.
\label{eq:affine_tube_lambda_bound}
\end{equation}
Thus \eqref{eq:affine_tube_lambda_bound} implies
\[
L_v^\sigma(T(b))
>
\bar h_{T(A)}^\sigma(v),
\qquad b\in B_b(f^\star),
\]
and positive homogeneity again transfers the certificate to
\(\widetilde u\). If \(B_b(f^\star)=\varnothing\), the preservation condition is vacuous and any \(\lambda>0\) is admissible. Thus, \(\kappa\) controls the recovery gap between \(A\) and \(R\), whereas \(\lambda\) independently preserves the exclusion of \(B_b(f^\star)\).

We have therefore established
\[
L_{\widetilde u}^\sigma(T(b))
>
\bar h_{T(A)}^\sigma(\widetilde u),
\qquad
b\in B_b(f^\star)\cup R.
\]
Proposition~\ref{prop:preservation_recovery} consequently preserves the exclusion of every sample originating from \(B_b(f^\star)\), while the transformed recoverable subset defined in
\eqref{eq:recoverable_subset} contains \(T(R)\). It may contain additional samples, but its cardinality is at least \(|R|\).

Finally, Theorem~\ref{thm:lsm_improvement}, the sample-collection convention of Appendix~\ref{app:sample_collection_convention}, and
\eqref{eq:affine_tube_affine_invariance} give
\[
\begin{aligned}
s\!\left(\rho_\sigma(T(A)),\rho_\sigma(T(B))\right)
&\ge
s(T(A),T(B))+\frac{|R|}{|A|}\\
&=
s(A,B)+\frac{|R|}{|A|}.
\end{aligned}
\]
\end{proof}

\subsection{Proof of Multiclass Recovery}
\label{app:proof_multiclass_recovery}
\begin{proof}
For each \(k\), form one coordinate block consisting of the anchor
\(\lambda_k q_k(x)\) and the \(m_k\) pairs
\[
\kappa_k g_{k,j}^0(x),\quad
-\kappa_k g_{k,j}^0(x),
\qquad j=1,\ldots,m_k.
\]
Concatenate these \(K\) blocks and choose the coefficient vectors of the common rank-completion coordinates to form an orthonormal basis of
\[
\operatorname{span}\!\left(
\{u_k:1\le k\le K\}
\cup
\{p_{k,j}:1\le k\le K,\ 1\le j\le m_k\}
\right)^{\perp}.
\]
The number of required coordinates is therefore
\[
K+2\sum_{k=1}^{K}m_k+d-
\dim\operatorname{span}\!\left(
\{u_k:1\le k\le K\}
\cup
\{p_{k,j}:1\le k\le K,\ 1\le j\le m_k\}
\right),
\]
which does not exceed \(h\). Any remaining coordinates are set to zero. The row space of the linear part \(M\) contains the displayed span and its orthogonal complement, so
\[
\operatorname{rank}(M)=d.
\]
Thus \(T\) is a single shared injective affine embedding. Applying Theorem~\ref{thm:affine_invariance} separately to every one-vs-rest pair gives
\begin{equation}
s(T(A_k),T(B_k))=s(A_k,B_k),
\qquad k=1,\ldots,K.
\label{eq:multiclass_affine_invariance}
\end{equation}

For each \(k\), let \(v_k\in\mathbb R^h\) take coefficient \(1\) on the anchor coordinate \(\lambda_kq_k(x)\), coefficient \(-1\) on the \(2m_k\) paired coordinates in the \(k\)-th block, and coefficient \(0\) on every other classwise block, rank-completion coordinate, and zero coordinate. Then
\[
\|v_k\|=\sqrt{1+2m_k},
\qquad
\widetilde u_k:=\frac{v_k}{\|v_k\|}.
\]
Before activation, the paired coordinates cancel and give
\begin{equation}
v_k^\top T(x)
=
\lambda_kq_k(x)
-\sum_{j=1}^{m_k}\kappa_kg_{k,j}^0(x)
+\sum_{j=1}^{m_k}\kappa_kg_{k,j}^0(x)
=
\lambda_kq_k(x).
\label{eq:multiclass_block_cancellation}
\end{equation}

Hence the halfspace with score \(-v_k^\top y\) induces exactly the same sample membership as the optimal supporting halfspace \(f_k^\star\) for \((A_k,B_k)\). Its boundary supports \(T(A_k)\), and it attains the same competing-sample intrusion as \(f_k^\star\). Together with
\eqref{eq:multiclass_affine_invariance}, this shows that the lifted halfspace remains optimal and that \(\widetilde u_k\) is its outward unit normal.

Because \(v_k\) vanishes outside the \(k\)-th block, all other classwise blocks and the common rank-completion coordinates contribute zero to
\(L_{v_k}^\sigma\), \(U_{v_k}^\sigma\), and
\(\bar h_{T(A_k)}^\sigma(v_k)\). Thus the paired bounds in
\eqref{eq:affine_tube_pair_bounds} apply independently to each classwise block.

Write \(D_{G,k}\), \(D_{A,k}\), and \(\varepsilon_k\) for the class-indexed quantities in Theorem~\ref{thm:affine_tube_recovery}. For every \(r\in R_k\), the classwise version of
\eqref{eq:affine_tube_recovery_gap} gives
\[
L_{v_k}^\sigma(T(r))
-
\bar h_{T(A_k)}^\sigma(v_k)
\ge
\kappa_k
\bigl(\alpha_\sigma^+D_{A,k}-\varepsilon_k\bigr)
+(m_k+1)\mu
>
0,
\]
where the strict inequality follows from the classwise bound on
\(\kappa_k\).

Similarly, for every sample \(b\in B_b(f_k^\star)\), the classwise version of
\eqref{eq:affine_tube_preservation_gap} gives
\[
L_{v_k}^\sigma(T(b))
-
\bar h_{T(A_k)}^\sigma(v_k)
\ge
\alpha_\sigma^+\lambda_kq_k(b)
-\kappa_kD_{G,k}(b)
+\alpha_\sigma^+\kappa_kD_{A,k}
+m_k\mu
>
0,
\]
where the final inequality follows from the classwise bound on
\(\lambda_k\). When \(B_b(f_k^\star)=\varnothing\), this preservation condition is vacuous.

The directional bounds are positively homogeneous in their direction. Therefore, division by \(\|v_k\|>0\) yields
\[
L_{\widetilde u_k}^\sigma(T(b))
>
\bar h_{T(A_k)}^\sigma(\widetilde u_k),
\qquad
b\in B_b(f_k^\star)\cup R_k.
\]
Proposition~\ref{prop:preservation_recovery} consequently preserves the exclusion of all previously excluded samples for the \(k\)-th one-vs-rest problem and certifies the recovery of every sample in \(T(R_k)\). The remaining classwise blocks do not affect this certificate because the corresponding entries of \(v_k\) are zero.

By Theorem~\ref{thm:lsm_improvement} and the sample-collection convention of Appendix~\ref{app:sample_collection_convention},
\[
s\!\left(
\rho_\sigma(T(A_k)),
\rho_\sigma(T(B_k))
\right)
\ge
s(T(A_k),T(B_k))
+\frac{|R_k|}{|A_k|}.
\]
Substituting \eqref{eq:multiclass_affine_invariance} gives
\[
s\!\left(
\rho_\sigma(T(A_k)),
\rho_\sigma(T(B_k))
\right)
\ge
s(A_k,B_k)+\frac{|R_k|}{|A_k|},
\qquad k=1,\ldots,K.
\]
Since the same embedding \(T\) is used for every \(k\), these one-vs-rest guarantees hold simultaneously.
\end{proof}

\subsection{Comparison with Related Separation Guarantees}
\label{app:comparison_an2015}

Table~\ref{tab:separation_guarantee_comparison} organizes related results by their geometric assumptions, transformations, and conclusions. Our results address both certification for a given representation and the construction of an affine map that realizes certified recovery.

\begin{table}[t]
\centering
\small
\setlength{\tabcolsep}{3pt}
\renewcommand{\arraystretch}{1.2}
\caption{Assumptions and conclusions of selected separation results.
A stage denotes an affine map followed by an activation; a final affine readout is not counted.
Width statements describe sufficient sizes for the specified constructions or probability guarantees.}
\label{tab:separation_guarantee_comparison}
\begin{tabular}{@{}p{0.17\linewidth}p{0.28\linewidth}p{0.20\linewidth}p{0.29\linewidth}@{}}
\toprule
Result &
Data and geometric assumptions &
Transformation &
Guarantee and parameter dependence \\
\midrule

\citet{an2015can}, Theorem~2 &
Finite sets; one class's convex hull avoids the other class, with a separating convex-hull decomposition. &
One constructed ReLU stage. &
Complete separation; width determined by the number of decomposition pieces. \\
\addlinespace

\citet{an2015can}, Theorem~4 &
Disjoint finite sets with a disjoint convex-hull decomposition of sizes \(L_1,L_2\). &
Two constructed ReLU stages. &
Complete separation; the displayed construction uses widths \(L_1L_2\) and \(L_1\). \\
\addlinespace

\citet{dirksen2022separation} &
Bounded sets at positive distance; geometric complexity controls the required widths and bias scales. &
Two random ReLU stages with Gaussian weights and uniform biases. &
High-probability separation with a margin; bounds depend on mutual covering complexity and Gaussian widths. \\
\addlinespace

\citet{ghosal2022randomly} &
Bounded sets at positive distance; finite cardinality or mutual complexity enters the bounds. &
One ReLU stage, with constructive and random variants. &
Constructive finite-set separation and high-probability random separation, with explicit width and margin guarantees. \\
\addlinespace

\citet{xu2025linearly}, Theorem~3.1 &
Nonzero points on two equal-dimensional subspaces with strictly positive principal angles. &
One linear map with Gaussian weights followed by coordinatewise squaring. &
High-probability separation of the subspaces; width is polynomial in intrinsic dimension, with angle and confidence dependence. \\
\midrule

This work: activation certificates &
Finite collections, an optimal supporting halfspace, and the directional preservation condition. &
Coordinatewise gated activation on a given representation. &
Preserves existing exclusions, certifies \(R_a(u)\), and guarantees D-LSM gain at least \(|R_a(u)|/|A|\). \\
\addlinespace

This work: affine-tube recovery &
Finite collections, an optimal supporting halfspace, and the aggregate-tube condition. &
One constructed injective affine map followed by ReLU, GELU, or SiLU. &
Preserves pre-activation D-LSM and recovers \(R\), with gain at least \(|R|/|A|\) and explicit width and scale conditions. \\
\bottomrule
\end{tabular}
\end{table}

\paragraph{Relation to constructive ReLU separation.}
Paired rectification,
\[
g(x)\longmapsto
\bigl(
\rho_{\mathrm{ReLU}}(g(x)),
\rho_{\mathrm{ReLU}}(-g(x))
\bigr),
\]
also appears in the bidirectional construction of \citet{an2015can}, which develops distance-preserving extensions. Our affine-tube construction combines paired coordinates with a supporting-direction coordinate and rank completion. These components preserve D-LSM before activation and provide explicit conditions for retaining existing exclusions and recovering a specified subset afterward.

\paragraph{A common data family.}
Let \(d=2\), and let
\[
A\subseteq\{-1,1\}\times[-1,1]
\]
be a finite collection containing the four corners \((\pm1,\pm1)\). Thus,
\[
\operatorname{conv}(A)=[-1,1]^2.
\]
Fix \(0<\varepsilon<1/2\), so that
\(\varepsilon<\alpha_\sigma^+\) for each of ReLU, GELU, and SiLU. Choose finite collections
\[
(0,0)\in R
\subset(-\varepsilon,\varepsilon)\times(-1,1)
\]
and
\[
\{(-2,2),(2,2)\}\subseteq E
\subset\{x\in\mathbb{R}^2:x_2>1\}.
\]
Define \(B:=R\sqcup E\).

The supporting halfspace
\[
H^\star=\{x:x_2\le1\}
\]
contains \(A\cup R\) and excludes \(E\). Every affine halfspace containing \(A\) also contains \(\operatorname{conv}(A)\), hence every sample in \(R\). Consequently, \(H^\star\) is optimal, and
\[
u=(0,1),\qquad
q(x)=x_2-1,\qquad
B_a(f^\star)=R,\qquad
B_b(f^\star)=E,
\]
with
\[
s(A,B)=1-\frac{|R|}{|A|}.
\]

\paragraph{Convex-separability conditions under both class assignments.}
The origin gives
\[
(0,0)\in B\cap\operatorname{conv}(A).
\]
Conversely,
\[
(\pm1,1)
=
\frac12(0,0)+\frac12(\pm2,2)
\in A\cap\operatorname{conv}(B).
\]
Therefore,
\[
\operatorname{conv}(A)\cap B\neq\varnothing,
\qquad
\operatorname{conv}(B)\cap A\neq\varnothing.
\]
The one-stage hypothesis of \citet[Theorem~2]{an2015can} fails under either assignment of \(A\) and \(B\).

Their general construction applies through the decomposition
\[
A=A_-\sqcup A_+,\qquad B=R\sqcup E,
\]
where
\[
A_\pm:=A\cap\bigl(\{\pm1\}\times[-1,1]\bigr).
\]
Indeed, \(\operatorname{conv}(R)\) lies strictly between the two target sides, while every point of \(\operatorname{conv}(E)\) has second coordinate greater than \(1\). Hence,
\[
\operatorname{conv}(A_\pm)\cap\operatorname{conv}(R)
=
\operatorname{conv}(A_\pm)\cap\operatorname{conv}(E)
=
\varnothing.
\]
Substituting \(L_1=L_2=2\) into the construction underlying \citet[Theorem~4]{an2015can} gives two ReLU stages with sufficient widths \(4\) and \(2\).

\paragraph{Affine-tube guarantee on the same family.}
Take \(m=1\) and
\[
g_1^0(x)=x_1,\qquad p_1=(1,0).
\]
Then
\[
D_G(x)=|x_1|,
\qquad
D_A=\min_{a\in A}|a_1|=1,
\qquad
D_G(r)<\varepsilon
\quad\text{for all }r\in R.
\]
The aggregate-tube condition therefore holds. Moreover,
\[
\dim\operatorname{span}\{u,p_1\}=2,
\]
so Theorem~\ref{thm:affine_tube_recovery} supplies the sufficient width
\[
h_0=1+2m+d-\dim\operatorname{span}\{u,p_1\}=3.
\]

For each activation, choose
\[
\kappa>
\frac{-2\mu}{\alpha_\sigma^+-\varepsilon}
\]
and
\[
\lambda>
\max\left\{
0,\,
\max_{b\in E}
\frac{\kappa|b_1|-\alpha_\sigma^+\kappa-\mu}
{\alpha_\sigma^+(b_2-1)}
\right\}.
\]
The affine map
\[
T(x)=Mx+c
=
\begin{bmatrix}
\lambda(x_2-1)\\
\kappa x_1\\
-\kappa x_1
\end{bmatrix}
\]
has \(\operatorname{rank}(M)=2\). For the lifted outward normal
\[
\widetilde u=\frac{(1,-1,-1)}{\sqrt{3}},
\]
the paired coordinates cancel before activation:
\[
\widetilde u^\top T(x)
=
\frac{\lambda}{\sqrt{3}}q(x).
\]
Thus the lifted supporting halfspace has the same admitted and excluded samples, and affine invariance gives
\[
s(T(A),T(B))
=
s(A,B)
=
1-\frac{|R|}{|A|}.
\]

Theorem~\ref{thm:affine_tube_recovery} guarantees that activation preserves every exclusion in \(E\) and recovers every sample in \(R\). Therefore,
\[
s\!\left(
\rho_\sigma(T(A)),
\rho_\sigma(T(B))
\right)
\ge
s(A,B)+\frac{|R|}{|A|}
=
1.
\]
Since D-LSM is at most \(1\), equality holds:
\[
s\!\left(
\rho_\sigma(T(A)),
\rho_\sigma(T(B))
\right)=1.
\]

This family consequently admits a single affine--activation stage of width \(3\), with unchanged D-LSM at the affine stage and an exact activation gain of \(|R|/|A|\). The same construction is certified for ReLU, GELU, and SiLU through their respective values of \(\alpha_\sigma^+\) and \(\mu\).

\paragraph{Dependence on recovery-set size.}
For fixed \(A\), \(E\), and \(\varepsilon\), the same admissible \(\kappa\) and \(\lambda\) work for every finite collection \(R\subset(-\varepsilon,\varepsilon)\times(-1,1)\) containing the origin. The decomposition sizes above and the width of the affine-tube construction remain fixed as \(|R|\) increases. Each additional sample in \(R\) contributes \(1/|A|\) to the certified D-LSM gain without changing the affine-tube width or its scale parameters.

\section{Practical Estimation of D-LSM}
\label{app:lsm_estimator}

Exact D-LSM computation can be approached through combinatorial searches coupled with linear-program feasibility tests, or through mixed-integer optimization. Although these procedures can identify the optimal target-preserving boundary, their cost grows rapidly with the number and dimensionality of the representations, making repeated evaluation across classes, layers, checkpoints, and random seeds impractical for modern neural networks. We therefore employ a trainable affine estimator that operates directly within standard neural-network frameworks and supports efficient batched optimization on accelerators. Each learned feasible boundary provides a valid lower estimate of D-LSM, while an estimate attaining one certifies the exact value.

\subsection{Penalty-Based Affine Search}
\label{app:lsm_surrogate}

For a fixed finite pair \((A,B)\), exact D-LSM requires solving
\[
\min_{H\in\mathcal{H}_A}|B\cap H|,
\]
where the feasibility condition \(H\in\mathcal{H}_A\) requires every target sample to remain on the target side. The resulting objective combines a hard preservation constraint with a discrete intrusion count. We use a penalty-based differentiable objective to guide the search over affine halfspaces. Let
\[
f_{w,c}(x):=w^\top x+c,
\]
where \(f_{w,c}(x)\ge0\) denotes the target side. Let
\[
\mathcal{L}_A(w,c)
=
\frac{1}{|A|}
\sum_{a\in A}
\log\!\left(1+\exp\!\left(-f_{w,c}(a)\right)\right),
\]
and
\[
\mathcal{L}_B(w,c)
=
\frac{1}{|B|}
\sum_{b\in B}
\log\!\left(1+\exp\!\left(f_{w,c}(b)\right)\right),
\]
which are the binary cross-entropy losses encouraging target and competing samples to lie on opposite sides of the affine boundary. We optimize
\begin{equation}
\mathcal{L}_{\mathrm{sur}}(w,c)
=
\beta\,\mathcal{L}_B(w,c)
+
\alpha\,\mathcal{L}_A(w,c),
\qquad
\alpha>0,\qquad\beta\ge0.
\label{eq:appendix_lsm_surrogate}
\end{equation}

The two weights reflect the asymmetric roles of the two sample collections. Reducing \(\mathcal{L}_B\) searches for boundaries that exclude more competing samples, whereas the target term acts as a penalty for violating the preservation constraint on \(A\). During optimization, the relative weighting is controlled so that satisfying the target constraint takes priority whenever it is violated; once the constraint is satisfied, the search can further reduce competing-sample intrusion. The constraint-priority weighting rule is given in Appendix~\ref{app:lsm_constraint_control}, and implementation details are provided in Appendix~\ref{app:vit_experimental_details}.

The surrogate guides boundary search and is not itself a D-LSM score. The final estimate is defined by the feasible-boundary selection procedure in Appendix~\ref{app:lsm_boundary_selection}.

\subsection{Best Feasible Boundary Selection}
\label{app:lsm_boundary_selection}

The penalty-based optimization in Appendix~\ref{app:lsm_surrogate} generates a sequence of candidate affine halfspaces
\[
H_t
:=
\{x\in\mathbb{R}^{d}:w_t^\top x+c_t\ge0\},
\]
where \(t\) indexes estimator epochs and each epoch-end candidate is evaluated on the complete target and competing collections.

An iterate is regarded as feasible whenever
\[
A\subseteq H_t.
\]
For a run that encounters at least one feasible candidate, define the searched feasible family
\begin{equation}
\widehat{\mathcal H}_A
:=
\{H_t:A\subseteq H_t,\; H_t \text{ is evaluated after an estimator epoch}\}.
\label{eq:appendix_searched_halfspaces}
\end{equation}
Among these candidates, we retain the halfspace with the smallest discrete competing-sample intrusion,
\begin{equation}
\widehat H
\in
\arg\min_{H\in\widehat{\mathcal H}_A}
|B\cap H|,
\label{eq:appendix_best_feasible_halfspace}
\end{equation}
and report
\begin{equation}
\widehat{s}(A,B)
:=
1-\frac{|B\cap\widehat H|}{|A|}
=
1-\frac{1}{|A|}
\min_{H\in\widehat{\mathcal H}_A}|B\cap H|.
\label{eq:appendix_estimated_lsm}
\end{equation}

For multiclass estimation, this selection is performed independently for each target class. A stored boundary is replaced only when a feasible candidate has a strictly smaller competing-sample intrusion count.

\begin{proposition}[Feasible lower-bound property]
\label{prop:estimated_lsm_lower_bound}
Whenever \(\widehat{\mathcal H}_A\neq\varnothing\),
\begin{equation}
\widehat{s}(A,B)
\le
s(A,B).
\label{eq:appendix_estimated_lsm_lower_bound}
\end{equation}
\end{proposition}

\begin{proof}
Every halfspace in \(\widehat{\mathcal H}_A\) preserves the target collection by construction. Hence
\[
\widehat{\mathcal H}_A
\subseteq
\mathcal H_A.
\]
Therefore,
\[
\min_{H\in\mathcal H_A}|B\cap H|
\le
\min_{H\in\widehat{\mathcal H}_A}|B\cap H|.
\]
Using Definition~\ref{def:directional_lsm} and \eqref{eq:appendix_estimated_lsm},
\[
\begin{aligned}
\widehat{s}(A,B)
&=
1-\frac{1}{|A|}
\min_{H\in\widehat{\mathcal H}_A}|B\cap H|
\\
&\le
1-\frac{1}{|A|}
\min_{H\in\mathcal H_A}|B\cap H|
\\
&=
s(A,B).
\end{aligned}
\]
\end{proof}

The approximation therefore arises from searching only a subset of the complete feasible halfspace family, rather than from replacing the D-LSM score by the surrogate loss. Equality is attained whenever the searched family contains a halfspace attaining the exact D-LSM optimum; otherwise, the best feasible iterate provides a lower bound on the exact D-LSM.

\subsection{Constraint-Priority Control}
\label{app:lsm_constraint_control}

For a \(K\)-class representation, the joint estimator maintains an independent pair of penalty weights \((\alpha_i,\beta_i)\) for every one-vs-rest problem \(i\). Let
\[
H_{i,t}
=
\{x:g_{i,t}(x)\ge0\}
\]
denote the affine halfspace associated with class \(i\) after estimator epoch \(t\). Mini-batches are used to optimize the surrogate objective within each epoch, while feasibility and discrete competing-sample intrusion are evaluated after the epoch using the complete representation dataset.

The penalty weights are initialized as
\[
\alpha_i=1,
\qquad
\beta_i=0,
\]
so the initial estimator epoch prioritizes target preservation for every class. After each epoch, if the full-data evaluation shows
\[
A_i\not\subseteq H_{i,t},
\]
we update only the corresponding penalty state according to
\begin{equation}
\alpha_i\leftarrow\gamma\alpha_i,
\qquad
\beta_i\leftarrow0,
\qquad
\gamma>1.
\label{eq:appendix_constraint_priority_update}
\end{equation}
This suppresses the competing-sample objective for class \(i\) during the next estimator epoch and prioritizes recovery of target feasibility. If instead
\[
A_i\subseteq H_{i,t},
\]
we keep the accumulated value of \(\alpha_i\) unchanged and set
\[
\beta_i\leftarrow1.
\]
The full-data feasibility test and corresponding penalty update are performed independently for all \(K\) outputs after every estimator epoch. Whenever a class-wise boundary is feasible, its intrusion is evaluated on the complete competing collection and the boundary is considered for best-feasible selection as described in Appendix~\ref{app:lsm_boundary_selection}.

\paragraph{Joint multi-class estimator.}
For a \(K\)-class representation, we optimize all \(K\) one-vs-rest affine searches jointly using a single linear estimator. A conventional binary implementation could assign two logits to each target class \(i\), \(f_{i,a}(x)\) and \(f_{i,b}(x)\), with the binary decision determined only by their difference
\[
g_i(x)=f_{i,a}(x)-f_{i,b}(x).
\]
Since adding the same scalar to both logits does not change this decision, we fix the reference logit to zero,
\[
f_{i,b}(x)\equiv0,
\]
and optimize only
\[
g_i(x)=f_{i,a}(x)=w_i^\top x+c_i.
\]
The resulting estimator consists of a single linear layer with \(K\) outputs,
\[
F(x)
=
\bigl(
g_1(x),\ldots,g_K(x)
\bigr),
\]
where output \(i\) defines the affine boundary for the one-vs-rest D-LSM problem of target class \(A_i\). As illustrated in Figure~\ref{fig:lsm_estimator}, this reduces the output dimension from \(2K\) to \(K\) without changing any of the \(K\) binary decision boundaries.

Although the affine parameters are optimized jointly, the target-preservation control is maintained independently for each one-vs-rest problem. Each target class \(i\) has its own penalty weights \(\alpha_i\) and \(\beta_i\), so violation of the preservation constraint for one class changes only the corresponding \((\alpha_i,\beta_i)\) pair.

\begin{figure}[t]
\centering
\includegraphics[width=0.68\linewidth]{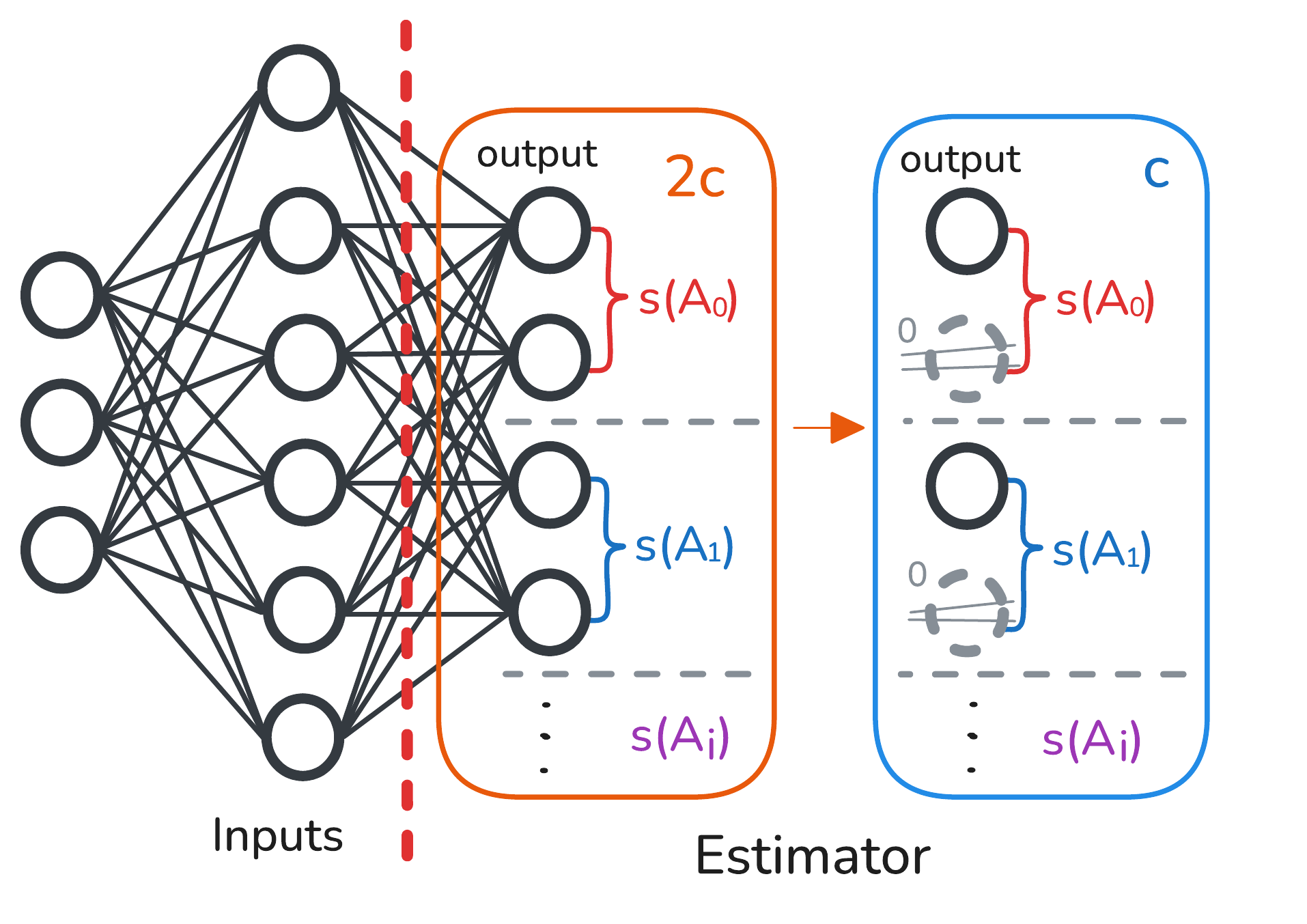}
\caption{Joint multi-class D-LSM estimator. A conventional formulation assigns two logits to each of the \(K\) one-vs-rest problems, whereas the equivalent implementation fixes one reference logit to zero and retains a single affine score per target class. The resulting estimator therefore uses \(K\) affine outputs, with an independently maintained constraint-control pair \((\alpha_i,\beta_i)\) for each class \(A_i\).}
\label{fig:lsm_estimator}
\end{figure}

\section{Experimental Details}
\label{app:experimental_details}

\subsection{Activation-Certificate Experiment}
\label{app:activation_experiment_setup}

\paragraph{Numerical protocol.}
All activation-certificate experiments use double-precision arithmetic. ReLU, exact GELU \(\rho_{\mathrm{GELU}}(t)=t\Phi(t)\), and SiLU are evaluated coordinatewise. The activation constants are
\[
\begin{aligned}
(\alpha_{\mathrm{ReLU}}^+,\tau_{\mathrm{ReLU}},\mu_{\mathrm{ReLU}})
&=(1,0,0),\\
(\alpha_{\mathrm{GELU}}^+,\tau_{\mathrm{GELU}},\mu_{\mathrm{GELU}})
&=(0.5,-0.7517915247,-0.1699712075),\\
(\alpha_{\mathrm{SiLU}}^+,\tau_{\mathrm{SiLU}},\mu_{\mathrm{SiLU}})
&=(0.5,-1.2784645428,-0.2784645428).
\end{aligned}
\]
The GELU and SiLU minimizers are computed by \(100\) bisection iterations on their scalar derivatives over \([-2,0]\) and \([-3,0]\), respectively. All strict certificate comparisons use tolerance \(10^{-10}\). The implementation assigns \(t=0\) to \(I_+\); for ReLU, \(I_0\) is empty.

\paragraph{Target construction.}
We fix three target anchors
\[
c_1=(-2,1),\qquad c_2=(1,-2),\qquad c_3=(1,1),
\]
and the orthonormal directions
\[
u=\frac{(-1,-1)}{\sqrt{2}},
\qquad
u_\perp=\frac{(1,-1)}{\sqrt{2}}.
\]
Each target collection contains \(|A|=40\) samples. The three anchors are retained exactly. The remaining \(37\) samples are assigned cyclically to the anchors, giving \(13\), \(12\), and \(12\) additional samples around \(c_1\), \(c_2\), and \(c_3\), respectively; the assignments are shuffled before perturbation. Each additional sample is generated as
\[
a=c_j+\xi,
\qquad
\xi\sim\operatorname{Uniform}([-0.05,0.05]^2),
\]
with independent coordinatewise noise.

\paragraph{Admitted samples.}
Each competing collection contains \(|B|=40\) samples, of which \(16\) are initially admitted. The origin is included explicitly,
\[
b_0=(0,0)=\frac{1}{3}c_1+\frac{1}{3}c_2+\frac{1}{3}c_3,
\]
and the remaining \(15\) samples are generated as
\[
b_i=\sum_{j=1}^{3}\omega_{ij}c_j,
\qquad
(\omega_{i1},\omega_{i2},\omega_{i3})
\sim\operatorname{Dirichlet}(\nu,\nu,\nu).
\]
We use
\[
\nu\in\{100,10,1,0.2\}.
\]
Larger \(\nu\) concentrates the samples near the barycenter, whereas smaller \(\nu\) distributes them more broadly over the anchor triangle. Since the three anchors belong to \(A\), all \(16\) samples lie in \(\operatorname{conv}(A)\) and are admitted by every target-preserving affine halfspace.

\paragraph{Excluded samples.}
Let
\[
h_A(u)=\max_{a\in A}u^\top a.
\]
The remaining \(24\) competing samples are generated as
\[
b_i=\bigl(h_A(u)+\eta_i\bigr)u+t_i u_\perp,
\]
where
\[
\eta_i\sim\operatorname{Uniform}(0.05,0.75),
\qquad
t_i\sim\operatorname{Uniform}(-1.5,1.5).
\]
All margins and tangent offsets are sampled independently. By construction,
\[
u^\top b_i=h_A(u)+\eta_i>h_A(u),
\]
so the halfspace supported by \(A\) along \(u\) excludes these \(24\) samples. Because the other \(16\) samples belong to \(\operatorname{conv}(A)\), \(u\) attains the exact minimum intrusion and
\[
|B_a(f^\star)|=16,\qquad
|B_b(f^\star)|=24,\qquad
s(A,B)=1-\frac{16}{40}=0.6.
\]

\paragraph{Randomization.}
Data are generated with \texttt{numpy.random.default\_rng}. For each concentration, we generate \(100\) datasets. With local index \(i\in\{0,\ldots,99\}\) and concentration index \(j\in\{0,1,2,3\}\), the random seed is
\[
2027+100000j+i,
\]
where \(j=0,1,2,3\) corresponds to \(\nu=100,10,1,0.2\), respectively. The same \(400\) datasets are evaluated under all three activations, producing \(1{,}200\) paired activation evaluations.

\paragraph{Scalar and directional envelopes.}
For every \(x\in A\cup B\), we compute
\[
L_u^\sigma(x),\qquad
u^\top\rho_\sigma(x),\qquad
U_u^\sigma(x),
\]
and record the maximum violation of
\[
L_u^\sigma(x)
\le
u^\top\rho_\sigma(x)
\le
U_u^\sigma(x).
\]
We also compare the exact transformed target support
\[
h_{\rho_\sigma(A)}(u)
=
\max_{a\in A}u^\top\rho_\sigma(a)
\]
with its upper certificate
\[
\bar h_A^\sigma(u)
=
\max_{a\in A}U_u^\sigma(a).
\]
The directional envelope width of \(x\) is \(U_u^\sigma(x)-L_u^\sigma(x)\). We first average this width over \(A\cup B\) within each evaluation and then aggregate it across evaluations.

The scalar envelopes are separately audited on \(12{,}001\) equally spaced values in \([-6,6]\), augmented by the exact boundary values \(t=\tau\) and \(t=0\). The envelope visualization uses a separate \(3{,}001\)-point grid over the same interval.

\paragraph{Preservation and recovery.}
Preservation is evaluated at the dataset level. An evaluation is certified when
\[
\min_{b\in B_b(f^\star)}L_u^\sigma(b)
>
\bar h_A^\sigma(u)+10^{-10},
\]
and is successful when
\[
\min_{b\in B_b(f^\star)}
u^\top\rho_\sigma(b)
>
h_{\rho_\sigma(A)}(u)+10^{-10}.
\]
Recovery is evaluated sample-wise over \(B_a(f^\star)\). A sample is certified when
\[
L_u^\sigma(b)
>
\bar h_A^\sigma(u)+10^{-10},
\]
and is successful along \(u\) when
\[
u^\top\rho_\sigma(b)
>
h_{\rho_\sigma(A)}(u)+10^{-10}.
\]
Certificate precision is the fraction of certified events that are successful, and coverage is the fraction of successful events that are certified.

\paragraph{Exact post-activation D-LSM.}
Post-activation D-LSM is computed exactly in two dimensions. We first compute the convex hull of \(\rho_\sigma(A)\) using the monotone-chain algorithm. Critical directions are orthogonal to \(b-a\), where \(a\) is a transformed target-hull vertex and \(b\in\rho_\sigma(B)\). Duplicate critical angles are merged after rounding to \(14\) decimal places, and the evaluated candidates consist of every critical angle and one midpoint between each consecutive pair. The competing-sample intrusion is constant within every resulting open angular interval.

For each candidate direction \(u'\), we compute
\[
h_{\rho_\sigma(A)}(u')
=
\max_{a\in A}u'^\top\rho_\sigma(a)
\]
and count the samples satisfying
\[
u'^\top\rho_\sigma(b)
\le
h_{\rho_\sigma(A)}(u')+10^{-10}.
\]
The minimum count gives exact post-activation D-LSM. Whenever the preservation certificate holds, we additionally evaluate
\[
s\!\left(\rho_\sigma(A),\rho_\sigma(B)\right)
\ge
0.6+\frac{|R_a(u)|}{40}.
\]

\paragraph{Reference visualization.}
The reference visualization uses \(\nu=1\). For every seed at this concentration, we aggregate the number of certified recovery samples across the three activations and divide it by the corresponding total number of recovery candidates. We select the seed whose rate is closest to the median, breaking ties by the smaller seed. This procedure selects seed \(202039\). The visualization reports the actual directional scores, the intervals \([L_u^\sigma,U_u^\sigma]\), the exact target support, and the certified target support.

\subsection{Quantitative Affine--Activation Experiment}
\label{app:affine_activation_setup}

This experiment evaluates the quantitative guarantee in Theorem~\ref{thm:lsm_improvement} over a prescribed family of affine--activation modules, comparing certified gains with exact post-activation D-LSM and measuring certificate coverage.

\paragraph{Data construction.}
We fix the target anchors and orthonormal directions
\[
c_1=(-2,1),\qquad c_2=(1,-2),\qquad c_3=(1,1),
\qquad
u=\frac{(-1,-1)}{\sqrt{2}},
\qquad
u_\perp=\frac{(1,-1)}{\sqrt{2}}.
\]
The target collection \(A\) contains these three anchors and \(37\) additional samples, with \(13\), \(12\), and \(12\) samples around \(c_1\), \(c_2\), and \(c_3\), respectively. Each additional sample receives independent coordinatewise perturbations from \(\operatorname{Uniform}(-0.05,0.05)\), giving \(|A|=40\).

The competing collection \(B=B_a\sqcup B_b\) has \(|B_a|=16\) and \(|B_b|=24\). The collection \(B_a\) contains \((0,0)\), \((1,0.5)\), and \(14\) random convex combinations
\[
b_i=\sum_{j=1}^{3}\omega_{ij}c_j,
\qquad
(\omega_{i1},\omega_{i2},\omega_{i3})
\sim\operatorname{Dirichlet}(\nu,\nu,\nu),
\qquad
\nu\in\{0.2,1,10\}.
\]
Writing \(h_A(u)=\max_{a\in A}u^\top a\), we generate the points in \(B_b\) as
\[
b_i=\bigl(h_A(u)+\delta_i\bigr)u+\xi_i u_\perp,
\qquad
\delta_i\sim\operatorname{Uniform}(0.05,0.75),
\qquad
\xi_i\sim\operatorname{Uniform}(-0.5,0.5),
\]
with independent margins and tangent offsets.

Every point in \(B_a\) lies in \(\operatorname{conv}(A)\), so every target-preserving halfspace admits these \(16\) samples. The supporting halfspace along \(u\) excludes all points in \(B_b\), establishing the exact initial value
\[
s_0:=s(A,B)=1-\frac{16}{40}=0.6.
\]
Sample identities and multiplicities are retained after transformation.

\paragraph{Affine--activation configurations.}
We evaluate the modules
\[
F_{\sigma,\gamma,t}(x)=\rho_\sigma(T_{\gamma,t}(x)),
\qquad
T_{\gamma,t}(x)=\gamma(x+t\mathbf{1}),
\]
using coordinatewise ReLU, exact GELU, and SiLU, with
\[
\begin{aligned}
\gamma&\in\{0.25,0.5,0.75,1,1.25,1.5\},\\
t&\in\{-1.5+0.25j:j=0,\ldots,16\}.
\end{aligned}
\]
Since \(\gamma>0\), Theorem~\ref{thm:affine_invariance} gives \(s(T_{\gamma,t}(A),T_{\gamma,t}(B))=s_0\), and the transported outward unit normal remains \(u\).

We use \(20\) data seeds, indexed by \(0,\ldots,19\). At a fixed seed, \(A\) and \(B_b\) remain unchanged across concentrations, while each dataset is shared across all activations and affine settings. This gives \(60\) datasets and \(20\times3\times3\times102=18{,}360\) evaluations.

\paragraph{Certified and observed events.}
We use the activation envelopes in Eq.~\eqref{eq:activation_envelopes}, with the activation constants listed in Appendix~\ref{app:activation_experiment_setup}. For a fixed configuration, abbreviate \(T=T_{\gamma,t}\) and \(F=F_{\sigma,\gamma,t}\), and define
\[
\begin{aligned}
h&=\max_{a\in A}u^\top F(a),
&
\bar h&=\max_{a\in A}U_u^\sigma(T(a)),\\
m_{\mathrm{cert}}(x)&=L_u^\sigma(T(x))-\bar h,
&
m_{\mathrm{obs}}(x)&=u^\top F(x)-h.
\end{aligned}
\]
An exclusion is certified when \(m_{\mathrm{cert}}(x)>0\) and observed when \(m_{\mathrm{obs}}(x)>0\). Boundary points are counted as admitted. Let \(r_{\mathrm{cert}}\) and \(r_{\mathrm{obs}}\) count certified and observed recoveries in \(B_a\), respectively, and let \(q\) count points in \(B_b\) with \(m_{\mathrm{obs}}(b)\le0\).

Write \(P_{\mathrm{cert}}\) for the event that \(m_{\mathrm{cert}}(b)>0\) for every \(b\in B_b\), and \(P_{\mathrm{obs}}\) for the corresponding observed event. When \(P_{\mathrm{cert}}\) holds, Theorem~\ref{thm:lsm_improvement} gives
\[
g_{\mathrm{cert}}=\frac{r_{\mathrm{cert}}}{40},
\qquad
s(F(A),F(B))-s_0\ge g_{\mathrm{cert}}.
\]
The certified gain is undefined otherwise. The observed value along the transported direction is
\[
s_u=s_0+\frac{r_{\mathrm{obs}}-q}{40},
\]
which accounts for both recovery from \(B_a\) and re-entry from \(B_b\).

\paragraph{Exact post-activation D-LSM.}
Let \(A'=F(A)\) and \(B'=F(B)=\{b'_1,\ldots,b'_{40}\}\). We compute the minimum intrusion through the mixed-integer optimization problem
\begin{equation}
\begin{aligned}
\min_{\substack{
w\in\mathbb{R}^{2},\,\beta\in\mathbb{R}\\
z\in\{0,1\}^{40}}}
&\quad \sum_{i=1}^{40}z_i\\
\text{subject to}
&\quad w^\top a'+\beta\ge1,
&& a'\in A',\\
&\quad z_i=0\ \Longrightarrow\
w^\top b'_i+\beta\le-1,
&& i=1,\ldots,40.
\end{aligned}
\label{eq:affine_activation_mip}
\end{equation}
Here, \(z_i=0\) enforces exclusion of the corresponding competing sample. The coefficients \(w,\beta\) are unrestricted, and unit margins can be obtained by shifting and rescaling the affine score on the finite sample collections. For the certified optimal objective \(n^\star\), we obtain
\[
s_{\mathrm{post}}=1-\frac{n^\star}{40},
\qquad
\Delta^\star=s_{\mathrm{post}}-s_0.
\]

\paragraph{Gain comparison.}
Certified and exact gains are compared on evaluations satisfying \(P_{\mathrm{cert}}\) and certified optimality. Preservation implies \(q=0\), yielding the decomposition
\[
\Delta^\star
=
\underbrace{g_{\mathrm{cert}}}_{\text{certified gain}}
+
\underbrace{\frac{r_{\mathrm{obs}}-r_{\mathrm{cert}}}{40}}_{\text{additional recovery}}
+
\underbrace{s_{\mathrm{post}}-s_u}_{\text{direction optimization}}.
\]
Figure~\ref{fig:affine_activation_gain} further restricts this comparison to evaluations with \(s_0<s_{\mathrm{post}}<1\), retaining cases with zero certified gain. This subset measures quantitative improvement before complete separation.

\paragraph{Certificate coverage and parameter regions.}
Recovery is evaluated sample-wise over \(B_a\), and preservation is evaluated both sample-wise over \(B_b\) and jointly over all of \(B_b\). Certificate precision is the fraction of certified events that are observed; coverage is the fraction of observed events that are certified.

For each seed, activation, and concentration, define
\[
\begin{aligned}
\mathcal{C}_k
&=\{(\gamma,t):
P_{\mathrm{cert}}\ \text{holds},
\ r_{\mathrm{cert}}\ge k\},\\
\mathcal{O}_k
&=\{(\gamma,t):
P_{\mathrm{obs}}\ \text{holds},
\ r_{\mathrm{obs}}\ge k\}.
\end{aligned}
\]
The region analysis uses \(k=8\) over the full affine parameter grid. Both regions describe preservation and recovery along the same transported direction; either joint event implies \(s_{\mathrm{post}}\ge s_0+8/40=0.8\). We report their grid proportions and the coverage \(|\mathcal{C}_8\cap\mathcal{O}_8|/|\mathcal{O}_8|\). At each grid point, the four maps report frequencies across \(20\) seeds for membership in \(\mathcal{C}_8\), \(\mathcal{O}_8\), and \(\mathcal{O}_8\setminus\mathcal{C}_8\), followed by the D-LSM decrease event \(s_{\mathrm{post}}<s_0=0.6\). The last event uses the certified optimal post-activation D-LSM. All region statistics use the full parameter grid.

\paragraph{Aggregation.}
For each activation and concentration, statistics are first computed within each seed over the relevant affine settings, then summarized by their mean and sample standard deviation across seeds. In Figure~\ref{fig:affine_activation_gain}, each seed contributes the mean over its eligible settings, and error bars show the standard deviation across the \(20\) seed means.

\subsection{Affine-Tube Experiment}
\label{app:tube_experiment_setup}

\paragraph{Numerical protocol.}
The affine-tube experiments use PyTorch double-precision tensors and numerical tolerance \(10^{-10}\). ReLU is evaluated with \texttt{torch.relu}, GELU with \texttt{gelu(..., approximate="none")}, and SiLU with \texttt{silu}. We use
\[
\varepsilon=0.15,\qquad
\kappa=2,\qquad
\lambda=5
\]
in every reference and randomized evaluation.

\paragraph{Reference geometry.}
The deterministic reference target is
\[
A
=
\{-1,1\}
\times
\{-1,-0.5,0,0.5,1\}.
\]
The recovery and previously excluded samples are
\[
\begin{aligned}
R=\{&
(-0.15,-0.8),(0.15,-0.25),
(-0.10,0.25),(0.10,0.8)\},\\
B_b(f^\star)=\{&
(-2,2.5),(2,2.5),(-1.5,3),(1.5,3),
(-0.75,3.5),(0.75,3.5)\}.
\end{aligned}
\]
We use \(u=(0,1)\), so
\[
h_A(u)=1,
\qquad
q(x)=x_2-1.
\]
The halfspace \(\{x:x_2\le1\}\) contains \(A\cup R\) and excludes \(B_b(f^\star)\). Moreover,
\[
\operatorname{conv}(A)=[-1,1]\times[-1,1],
\]
so all four samples in \(R\) are unavoidable before activation. Therefore,
\[
s(A,B)
=
1-\frac{|R|}{|A|}
=
1-\frac{4}{10}
=
0.6.
\]

\paragraph{Tube and affine embedding.}
We use one normalized tube function,
\[
g_1^0(x)=x_1,
\qquad
D_G(x)=|x_1|.
\]
Consequently,
\[
D_A=\min_{a\in A}D_G(a)=1,
\qquad
D_G(r)\le0.15
\quad\text{for all }r\in R.
\]
Since \(d=2\), \(m=1\), and
\[
\dim\operatorname{span}\{(0,1),(1,0)\}=2,
\]
the width condition in Theorem~\ref{thm:affine_tube_recovery} gives
\[
h\ge1+2m+d-2=3.
\]
We use the sufficient width \(h=3\) given by Theorem~\ref{thm:affine_tube_recovery} and define
\[
T(x)
=
\begin{pmatrix}
\lambda(x_2-1)\\
\kappa x_1\\
-\kappa x_1
\end{pmatrix}
=
Mx+c,
\]
where
\[
M=
\begin{pmatrix}
0 & \lambda\\
\kappa & 0\\
-\kappa & 0
\end{pmatrix},
\qquad
c=
\begin{pmatrix}
-\lambda\\
0\\
0
\end{pmatrix}.
\]
Because \(\lambda,\kappa>0\), \(\operatorname{rank}(M)=2=d\).

For
\[
v=(1,-1,-1),
\qquad
\widetilde u=\frac{v}{\sqrt{3}},
\]
the paired coordinates cancel exactly:
\[
v^\top T(x)
=
\lambda q(x)-\kappa x_1+\kappa x_1
=
\lambda q(x).
\]
Thus the lifted halfspace has the same pre-activation membership pattern as the original supporting halfspace. The reference visualization uses the unnormalized score \(v^\top T(x)\), while all certificate and actual margins are computed along \(\widetilde u\).

\paragraph{Reference parameter bounds.}
For ReLU, GELU, and SiLU, respectively, the lower bounds on \(\kappa\) are
\[
0,\qquad
0.9712640427,\qquad
1.5912259586.
\]
After setting \(\kappa=2\), the corresponding reference lower bounds on \(\lambda\) are
\[
1.3333333333,\qquad
4.2266282766,\qquad
4.3712860570.
\]
Hence the common choices \(\kappa=2\) and \(\lambda=5\) satisfy the strict theorem conditions for all three activations.

\paragraph{Random target pool.}
For each seed, we first generate one target collection containing the four fixed anchors
\[
C=\{(-1,-1),(-1,1),(1,-1),(1,1)\}.
\]
We then independently generate \(18\) left and \(18\) right target samples:
\[
X_i^{\mathrm{L}}\sim\operatorname{Uniform}(-2,-1.1),
\qquad
Y_i^{\mathrm{L}}\sim\operatorname{Uniform}(-0.9,0.9),
\]
and
\[
X_i^{\mathrm{R}}\sim\operatorname{Uniform}(1.1,2),
\qquad
Y_i^{\mathrm{R}}\sim\operatorname{Uniform}(-0.9,0.9).
\]
The resulting collection is
\[
A
=
C
\cup
\{(X_i^{\mathrm{L}},Y_i^{\mathrm{L}})\}_{i=1}^{18}
\cup
\{(X_i^{\mathrm{R}},Y_i^{\mathrm{R}})\}_{i=1}^{18},
\]
and therefore contains \(40\) samples. The anchors ensure that \(h_A(u)=1\) and \(D_A=1\), while every additional target sample satisfies \(q(a)<0\) and \(D_G(a)>1\).

\paragraph{Random competing pools.}
For the same seed, we generate a pool of \(16\) recovery samples,
\[
R_{\max}
=
\{(X_i^{\mathrm{rec}},Y_i^{\mathrm{rec}})\}_{i=1}^{16},
\]
where
\[
X_i^{\mathrm{rec}}\sim\operatorname{Uniform}(-0.15,0.15),
\qquad
Y_i^{\mathrm{rec}}\sim\operatorname{Uniform}(-0.8,0.8),
\]
and a pool of \(36\) previously excluded samples,
\[
B_{b,\max}
=
\{(X_i^{\mathrm{exc}},Y_i^{\mathrm{exc}})\}_{i=1}^{36},
\]
where
\[
X_i^{\mathrm{exc}}\sim\operatorname{Uniform}(-2,2),
\qquad
Y_i^{\mathrm{exc}}\sim\operatorname{Uniform}(2.5,3.5).
\]
All coordinates are sampled sequentially from a CPU \texttt{torch.Generator} initialized by the corresponding seed.

For
\[
r\in\{4,8,12,16\},
\]
the recovery collection consists of the first \(r\) samples of \(R_{\max}\), and the excluded collection consists of the first \(40-r\) samples of \(B_{b,\max}\). Thus the four recovery settings for a fixed seed share the same target and nested competing-sample pools. Every recovery sample lies in the anchor box \([-1,1]^2\subseteq\operatorname{conv}(A)\), while every excluded sample satisfies \(x_2>1\). It follows that \(u=(0,1)\) is exactly optimal and
\[
s(A,B)=1-\frac{r}{40}.
\]

\paragraph{Randomization and paired evaluation.}
We use a CPU PyTorch generator with seeds \(0,\ldots,49\). One complete target, recovery, and excluded pool is generated for each seed and reused across all four recovery sizes. Each resulting dataset is then evaluated without modification under ReLU, GELU, and SiLU. The experiment therefore contains
\[
50\times4\times3=600
\]
paired evaluations.

\paragraph{Parameter and certificate audit.}
For each dataset and activation, we compute the activation-dependent theorem bounds on \(\kappa\) and \(\lambda\), followed by the slacks
\[
\kappa-\kappa_{\min},
\qquad
\lambda-\lambda_{\min}.
\]
We verify the output width, the rank of \(M\), and the maximum cancellation residual
\[
\max_{x\in A\cup B}
\left|
v^\top T(x)-\lambda q(x)
\right|.
\]
The normalized recovery and preservation certificate margins are
\[
\min_{r\in R}
\left[
L_{\widetilde u}^\sigma(T(r))
-
\bar h_{T(A)}^\sigma(\widetilde u)
\right]
\]
and
\[
\min_{b\in B_b(f^\star)}
\left[
L_{\widetilde u}^\sigma(T(b))
-
\bar h_{T(A)}^\sigma(\widetilde u)
\right].
\]
We additionally compute the corresponding actual margins using the exact activated scores relative to
\[
\max_{a\in A}
\widetilde u^\top\rho_\sigma(T(a)).
\]

The affine-stage D-LSM is obtained from Theorem~\ref{thm:affine_invariance}. After activation, the explicit direction \(\widetilde u\) is used to verify that every sample in \(R\cup B_b(f^\star)\) lies strictly beyond the transformed target support. Zero competing intrusion certifies
\[
s\!\left(\rho_\sigma(T(A)),\rho_\sigma(T(B))\right)=1.
\]
Finally, we compare the observed gain with \(r/40\) and record their difference as the gain residual.

\subsection{Neural Representation Setup}
\label{app:neural_representation_setup}

\paragraph{Dataset.} We use all \(200\) Tiny ImageNet classes for both model training and representation analysis. Each class contains \(500\) training images, giving \(N=100{,}000\) samples in total. For the one-vs-rest problem associated with class \(c\), the target collection \(A_c\) contains \(500\) samples and the competing collection \(B_c\) contains \(99{,}500\) samples. The class set, sample set, and sample ordering are fixed across all model checkpoints, network locations, and estimator runs.

\paragraph{ViT architecture.} We use a ViT-S/32 with input resolution \(224\times224\) and patch size \(32\times32\), producing a \(7\times7\) patch grid. A learnable class token is prepended to the \(49\) patch tokens, giving \(50\) tokens in total. The model contains \(12\) Transformer encoder blocks, hidden dimension \(384\), \(6\) attention heads per block, and FFN hidden dimension \(1{,}536\). Each FFN applies an affine expansion from \(\mathbb{R}^{384}\) to \(\mathbb{R}^{1536}\), followed by GELU and an affine projection back to \(\mathbb{R}^{384}\). The classification head has \(200\) outputs.

\paragraph{ViT training configuration.} The ViT-S/32 is initialized from scratch and trained for \(200\) epochs using cross-entropy loss and AdamW \citep{loshchilov2017decoupled}. Training is performed on a single GPU with a per-step and effective batch size of \(512\). The base learning rate is \(1\times10^{-4}\), and no weight decay is applied. We use a linear learning-rate warm-up for the first \(5\) epochs with start factor \(0.033\), followed by cosine annealing \citep{loshchilov2016sgdr} over the remaining \(195\) epochs. Label smoothing, Mixup, CutMix, and model exponential moving average are disabled. Gradient clipping uses the default maximum norm of \(1.0\). The final epoch-\(200\) checkpoint attains \(99.98\%\) accuracy on the complete training set. The complete configuration is summarized in Table~\ref{tab:vit_model_training_config}. All model training, representation extraction, and D-LSM estimation are implemented in PyTorch \citep{paszke2019pytorch}.

\begin{table}[t]
\centering
\small
\caption{Training configuration of the ViT-S/32 used in the representation experiments.}
\label{tab:vit_model_training_config}
\begin{tabular}{lc}
\toprule
Configuration & Value \\
\midrule
Training epochs & \(200\) \\
Batch size & \(512\) \\
Effective batch size & \(512\) \\
Gradient accumulation steps & \(1\) \\
Optimizer & AdamW \\
Adam coefficients & \((0.9,0.999)\) \\
Adam \(\epsilon\) & \(10^{-8}\) \\
Weight decay & \(0\) \\
Base learning rate & \(1\times10^{-4}\) \\
Warm-up scheduler & LinearLR \\
Warm-up epochs & \(5\) \\
Warm-up start factor & \(0.033\) \\
Main scheduler & CosineAnnealingLR \\
\(T_{\max}\) & \(195\) \\
Label smoothing & \(0\) \\
Mixup / CutMix & \(0/0\) \\
Gradient clipping norm & \(1.0\) \\
Model EMA & Disabled \\
Mixed precision & Enabled \\
\bottomrule
\end{tabular}
\end{table}

\paragraph{Training implementation.} The run uses one training process on a single GPU. Automatic mixed precision and the channels-last memory format are enabled. Matrix-multiplication precision is set to \texttt{high}, and TF32 execution is enabled on supported CUDA hardware. The training DataLoader uses \(16\) workers, pinned memory, persistent workers, and a prefetch factor of \(4\).

\paragraph{Image preprocessing.} We apply the same deterministic preprocessing to every image during representation analysis. Each image is resized to \(256\) pixels using bicubic interpolation, center-cropped to \(224\times224\), converted to a tensor, and normalized using the ImageNet mean \(\mu=(0.485,0.456,0.406)\) and standard deviation \(\sigma=(0.229,0.224,0.225)\). The resulting pipeline is \(\operatorname{Resize}(256)\rightarrow\operatorname{CenterCrop}(224)\rightarrow\operatorname{ToTensor}\rightarrow\operatorname{Normalize}\).

\paragraph{Evaluation checkpoints.} We analyze the model checkpoints at epochs \(0\), \(20\), and \(200\). The epoch-\(0\) checkpoint is saved after model and optimizer initialization but before any parameter update, epoch \(20\) represents an intermediate training stage, and epoch \(200\) is the final trained model. At every checkpoint, we evaluate all \(12\) Transformer FFNs and all \(200\) one-vs-rest target classes. Each checkpoint is evaluated using the same five estimator seeds and the same DataLoader batch sequence described below. Thus, the experimental protocol contains \(3\times5=15\) joint full-depth D-LSM estimation runs.

\subsection{ViT Representation and D-LSM Estimation Details}
\label{app:vit_experimental_details}

\paragraph{Representation extraction.} For every FFN, we extract the representations at \texttt{ln\_2}, \texttt{mlp.0}, and \texttt{mlp.1}. These locations correspond respectively to the FFN input, the output of the affine expansion immediately before GELU, and the post-GELU representation. The evaluated model uses the exact GELU activation \(\rho_{\mathrm{GELU}}(t)=t\Phi(t)\) analyzed in Appendix~\ref{app:activation_examples}. Evaluating all \(12\) FFNs therefore gives \(12\times3=36\) representation locations per checkpoint. During D-LSM estimation, the ViT is fixed in evaluation mode and all representations are computed without gradient tracking; optimization updates only the affine D-LSM estimators.

\paragraph{Representation dimensionality.} We flatten the complete \(50\)-token representation before applying an estimator. Flattening is a fixed bijective coordinate identification of the token tensor with a Euclidean vector and therefore does not alter its affine-halfspace geometry. The resulting dimensions are
\[
d_{\mathrm{in}}=50\times384=19{,}200,\qquad d_{\mathrm{FFN}}=50\times1{,}536=76{,}800,
\]
where \(d_{\mathrm{in}}\) is used at \texttt{ln\_2}, while \(d_{\mathrm{FFN}}\) is used at \texttt{mlp.0} and \texttt{mlp.1}. Since the analysis uses all \(N=100{,}000\) training samples,
\[
d_{\mathrm{in}}+1=19{,}201<N,\qquad d_{\mathrm{FFN}}+1=76{,}801<N.
\]
Thus, the corresponding augmented feature matrices cannot have full row rank. In particular, arbitrary sample-wise affine labels cannot be interpolated solely because the representation dimension is at least the number of samples. This observation does not itself determine the separability of the natural Tiny ImageNet one-vs-rest labels.

We analyze the complete training representation set because the objective is to characterize the finite-sample geometry of the representations learned from these samples, rather than to estimate out-of-sample classification performance. The ViT results should therefore be interpreted as training-representation geometry.

\paragraph{Affine-rank verification.} For each checkpoint and FFN, we verify that the affine-expansion weight matrix \(W\in\mathbb{R}^{1536\times384}\) has rank \(384\). Because the same affine transformation is applied independently to all \(50\) tokens, its flattened linear component is
\[
\mathcal{W}=I_{50}\otimes W,
\qquad
\operatorname{rank}(\mathcal{W})=50\,\operatorname{rank}(W)=19{,}200.
\]
The flattened FFN expansion is therefore full-column-rank. Theorem~\ref{thm:affine_invariance} consequently gives
\[
s_{\mathrm{in},c}=s_{\mathrm{linear},c}
\]
for every checkpoint, FFN, and target class \(c\). The practical estimates at the two locations may nevertheless differ because their affine boundaries are optimized through separate finite searches.

\paragraph{Joint one-vs-rest estimators.} For each extracted representation \(q\), we construct a PyTorch \texttt{Linear} module from the flattened feature space to \(200\) affine outputs. Its output for class \(c\) is an affine score
\[
g_{q,c}(z)=w_{q,c}^{\top}z+b_{q,c},
\]
which is compared with the fixed reference score \(f_b=0\). The corresponding target-preserving halfspace is \(H_{q,c}=\{z:g_{q,c}(z)\ge0\}\). The \(200\) outputs represent the \(200\) one-vs-rest problems at location \(q\). All \(36\) representation-specific linear modules are collected into a joint estimator and optimized in the same run, while their parameters, constraint states, feasibility decisions, and best feasible boundaries remain separate.

\paragraph{Constraint-priority optimization.} We follow the rule in Appendix~\ref{app:lsm_constraint_control}, initializing \(\alpha_{q,c}=1\) and \(\beta_{q,c}=0\) for every representation \(q\) and target class \(c\), with penalty multiplier \(\gamma=1.01\). Mini-batches optimize the surrogate objective, while target feasibility and competing-sample intrusion are evaluated on all \(100{,}000\) samples after every estimator epoch. Penalty states are updated independently for each representation--class pair.

\paragraph{Estimator seeds and batch order.} Every checkpoint is evaluated using the five PyTorch seeds
\[
5,\qquad10,\qquad1024,\qquad2026,\qquad2027.
\]
For each run, the corresponding seed is set before constructing the affine estimators, thereby controlling their initial weights and biases. The neural-network checkpoint and representation data remain unchanged across these estimator seeds. Separately, every estimator run uses a dedicated DataLoader generator with the fixed seed \(5\). This produces the same sequence of shuffled mini-batches in all runs, checkpoints, and representation locations. The estimator seed and DataLoader seed therefore serve different purposes: the former varies the initial affine search state, whereas the latter holds the batch order fixed. The five seed-specific results are retained separately for subsequent statistical analysis.

\paragraph{Optimization configuration.} Every D-LSM estimator is optimized for \(300\) epochs using Adam \citep{kingma2014adam} with learning rate \(5\times10^{-5}\), \(\beta_1=0.9\), \(\beta_2=0.999\), \(\epsilon=10^{-8}\), and no weight decay. The mini-batch size is \(1{,}024\), and no learning-rate scheduler is used. Each estimator epoch contains one complete pass over the \(100{,}000\) samples, followed by a full-data feasibility and intrusion evaluation. The configuration is summarized in Table~\ref{tab:dlsm_estimator_config}.

\begin{table}[t]
\centering
\small
\caption{Configuration of the D-LSM estimator used for all ViT representation evaluations.}
\label{tab:dlsm_estimator_config}
\begin{tabular}{lc}
\toprule
Configuration & Value \\
\midrule
Model checkpoints & \(0,20,200\) \\
FFN indices & \(0,\ldots,11\) \\
Locations per FFN & \texttt{ln\_2}, \texttt{mlp.0}, \texttt{mlp.1} \\
One-vs-rest outputs & \(200\) \\
Estimator seeds & \(5,10,1024,2026,2027\) \\
DataLoader seed & \(5\) \\
Estimator epochs & \(300\) \\
Mini-batch size & \(1{,}024\) \\
Optimizer & Adam \\
Learning rate & \(5\times10^{-5}\) \\
Adam coefficients & \((0.9,0.999)\) \\
Adam \(\epsilon\) & \(10^{-8}\) \\
Weight decay & \(0\) \\
Learning-rate scheduler & None \\
Full-data evaluation & After every epoch \\
\bottomrule
\end{tabular}
\end{table}

\paragraph{Best feasible boundary selection.} We apply the selection procedure in Appendix~\ref{app:lsm_boundary_selection} independently to each representation \(q\), target class \(c\), and estimator seed \(r\), using the epoch-end candidates from all \(300\) estimator epochs. For each \((q,c,r)\), we retain the selected boundary and its score \(\widehat{s}_{q,c}^{(r)}\), computed according to \eqref{eq:appendix_estimated_lsm}, for subsequent analysis. The lower-bound guarantee follows from Proposition~\ref{prop:estimated_lsm_lower_bound}.

\paragraph{Computation.} Model training, representation evaluation, and D-LSM optimization are performed on a single NVIDIA RTX 6000 Ada Generation GPU with \(49{,}140\) MiB of available memory. Estimator optimization uses mini-batches of \(1{,}024\) samples with pinned host memory and \(8\) DataLoader workers. The model is held fixed in evaluation mode throughout each estimator run. Combining three model checkpoints with five estimator seeds gives \(15\) full-depth estimation runs, each jointly evaluating \(36\) representation locations and \(200\) one-vs-rest problems per location.

\subsection{Affine Boundary Transport Details}
\label{app:vit_invariance_details}

\paragraph{Audit objective.}
The D-LSM estimates at \texttt{ln\_2} and \texttt{mlp.0} are produced by independently optimized affine boundaries and may therefore differ even though their exact D-LSM values are equal by Theorem~\ref{thm:affine_invariance}. To evaluate the affine correspondence directly, we fix each learned boundary, transport it across the actual FFN affine expansion, and compare the original and transported boundaries on the complete representation set. No estimator optimization is performed during the audit.

\paragraph{Boundary transport.}
For FFN \(j\), let \(T_j(x)=M_jx+c_j\) denote the affine expansion from the flattened \texttt{ln\_2} representation to the flattened \texttt{mlp.0} representation. Here, \(M_j=I_{50}\otimes W_j\), where the same FFN weight matrix \(W_j\) is applied to all \(50\) tokens. As established in Appendix~\ref{app:vit_experimental_details}, every \(W_j\) and hence every \(M_j\) is full-column-rank.

Consider an input-space boundary \(f(x)=u^\top x+\beta\). Full column rank of \(M_j\) guarantees the existence of \(v\) satisfying \(M_j^\top v=u\). Setting \(\eta=\beta-v^\top c_j\) gives the output-space boundary \(g(y)=v^\top y+\eta\), for which
\[
g(T_j(x))
=
v^\top(M_jx+c_j)+\eta
=
u^\top x+\beta
=
f(x).
\]
Conversely, any output-space boundary \(g(y)=v^\top y+\eta\) can be transported to the input space by composition,
\[
f(x)
=
g(T_j(x))
=
(M_j^\top v)^\top x+v^\top c_j+\eta.
\]
Thus, both transport directions preserve the affine score and therefore preserve sample membership, target feasibility, competing-sample intrusion, and the D-LSM value in exact arithmetic. When transporting an input boundary to the higher-dimensional output space, we use the minimum-norm solution of \(M_j^\top v=u\) returned by \texttt{torch.linalg.lstsq}.

\paragraph{Evaluation protocol.}
We audit the epoch-\(0\), epoch-\(20\), and epoch-\(200\) checkpoints used in Section~\ref{sec:vit_experiment}. For each checkpoint, we load the D-LSM estimators trained with seeds \(5\), \(10\), \(1024\), \(2026\), and \(2027\). At every FFN, the learned \texttt{ln\_2} boundary is transported to \texttt{mlp.0}, and the learned \texttt{mlp.0} boundary is transported to \texttt{ln\_2}. Each original boundary is then compared with its transported realization on the corresponding representations. All \(12\) FFNs, both transport directions, and all \(200\) one-vs-rest boundaries are evaluated using the complete \(100{,}000\)-sample training set.

The ViT and estimators are fixed in evaluation mode, and all representations and boundary outputs are computed without gradient tracking. The audit uses an evaluation batch size of \(1{,}024\), \(8\) DataLoader workers, pinned memory, persistent workers, and no shuffling. Boundary transport and evaluation use FP32 tensors without automatic mixed precision.

\begin{table}[t]
\centering
\small
\caption{Configuration of the affine boundary transport audit.}
\label{tab:vit_invariance_config}
\begin{tabular}{lc}
\toprule
Configuration & Value \\
\midrule
Model checkpoints & \(0,20,200\) \\
Estimator seeds & \(5,10,1024,2026,2027\) \\
FFN indices & \(0,\ldots,11\) \\
Representation pair & \texttt{ln\_2}, \texttt{mlp.0} \\
Transport directions & Input-to-output, output-to-input \\
One-vs-rest boundaries & \(200\) \\
Evaluation samples & \(100{,}000\) \\
Evaluation batch size & \(1{,}024\) \\
DataLoader workers & \(8\) \\
DataLoader shuffling & Disabled \\
Numerical precision & FP32 \\
Score tolerance & \(10^{-4}\) \\
\bottomrule
\end{tabular}
\end{table}

\paragraph{Audit scale.}
The experiment contains
\[
3\text{ checkpoints}
\times
5\text{ seeds}
\times
12\text{ FFNs}
\times
2\text{ transport directions}
=
360
\]
transport comparisons. Each comparison contains \(200\) class-wise boundaries evaluated on \(100{,}000\) samples. The complete audit therefore includes
\[
360\times200=72{,}000
\]
class-wise boundary comparisons and
\[
360\times200\times100{,}000
=
7.2\times10^9
\]
sample-wise membership comparisons.

\paragraph{Audit quantities.}
For each original boundary \(f_c\) and its transported realization \(g_c\), we record the maximum absolute score error
\[
\epsilon_{\mathrm{score}}
=
\max_{n,c}
\left|
f_c(x_n)-g_c(T_j(x_n))
\right|
\]
and the number of sample--class pairs for which the two scores produce different target-side memberships. Membership is defined by comparison with the fixed reference score \(f_b=0\), so a sample belongs to the target side exactly when its affine score is nonnegative.

At the class level, we compare whether the two boundaries preserve every target sample, the number of competing samples retained by each feasible boundary, and their resulting D-LSM values. We report the number of feasibility mismatches, the number of classes with different intrusion counts, the maximum intrusion-count difference, and the maximum absolute D-LSM difference. Since every Tiny ImageNet target class contains \(500\) samples, an intrusion-count difference of one corresponds to a D-LSM difference of \(1/500=0.002\). The score tolerance \(10^{-4}\) is used only as an auxiliary numerical diagnostic and does not affect membership, feasibility, intrusion, or D-LSM computation.

\section{Additional Experimental Results}
\label{app:additional_results}

\subsection{Activation-Certificate Results}
\label{app:activation_experiment_results}

\paragraph{Certificate and realized outcomes.} For each dataset and fixed direction \(u\), we distinguish an envelope-certified event from a realized separation event. The former compares \(L_u^\sigma(x)\) with the certified target support \(\bar h_A^\sigma(u)\), as required by Proposition~\ref{prop:preservation_recovery}; the latter compares the exact score \(u^\top\rho_\sigma(x)\) with the exact support \(\max_{a\in A}u^\top\rho_\sigma(a)\). Reporting both quantities tests the validity and tightness of the sufficient conditions: every certified event must be realized, while realized but uncertified events quantify separation beyond the guaranteed region.

We first examine the scalar envelopes underlying Lemma~\ref{lem:activation_envelope}. Figure~\ref{fig:activation_envelopes} compares each activation with its certified interval \([\ell_\sigma(t),r_\sigma(t)]\) over \(t\in[-6,6]\). The ReLU envelope is exact, whereas the GELU and SiLU envelopes widen on the positive region because their common lower slope is \(\alpha_\sigma^+=0.5\). On the negative region, the lower endpoint is controlled by the global activation minimum \(\mu_\sigma\).

\begin{figure}[htbp]
\centering
\includegraphics[width=\linewidth]{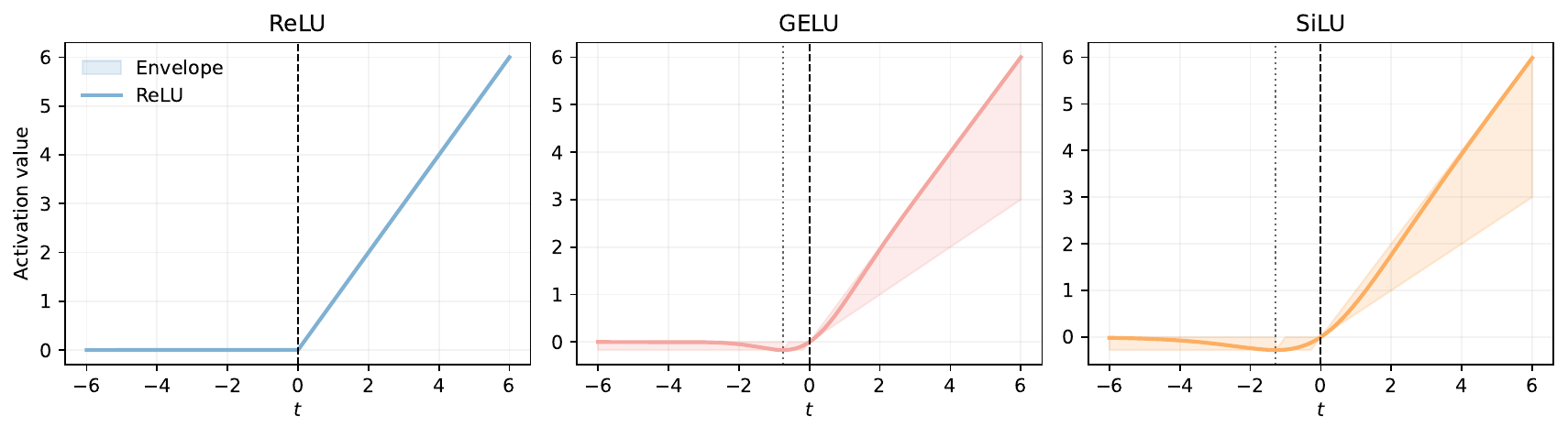}
\caption{Scalar activation envelopes used by the certificates. The shaded region represents \([\ell_\sigma(t),r_\sigma(t)]\), the dashed vertical line marks \(t=0\), and the dotted line marks the negative-region minimizer \(t=\tau_\sigma\).}
\label{fig:activation_envelopes}
\end{figure}

Table~\ref{tab:activation_envelope_results} summarizes the numerical envelope audit. Scalar widths are computed on the grid described in Appendix~\ref{app:activation_experiment_setup}. Directional widths are first summarized within each dataset and then averaged over the \(400\) datasets evaluated for each activation.

\begin{table}[htbp]
\centering
\small
\caption{Scalar and directional envelope audit. ``Max viol.'' is the maximum violation among the scalar bounds, directional bounds, and transformed target-support bound.}
\label{tab:activation_envelope_results}
\setlength{\tabcolsep}{5pt}
\begin{tabular}{lccccc}
\toprule
& \multicolumn{2}{c}{Scalar width} & \multicolumn{2}{c}{Directional width} & \\
\cmidrule(lr){2-3}\cmidrule(lr){4-5}
Act. & Mean & \(q_{.95}\) & Mean & Mean \(q_{.95}\) & Max viol. \\
\midrule
ReLU & \(0.000\) & \(0.000\) & \(0.000\) & \(0.000\) & \(0.000\) \\
GELU & \(0.833\) & \(2.700\) & \(0.382\) & \(0.712\) & \(0.000\) \\
SiLU & \(0.883\) & \(2.700\) & \(0.445\) & \(0.712\) & \(0.000\) \\
\bottomrule
\end{tabular}
\end{table}

No envelope or target-support violation is observed. The zero ReLU widths confirm that its certificates coincide with the exact directional scores. GELU and SiLU produce nonzero widths, but the resulting intervals remain valid throughout the scalar and dataset-level audits.

Figure~\ref{fig:activation_reference_audit} gives a representative geometric audit at \(\nu=1\). The raw geometry contains \(16\) samples in \(B_a(f^\star)\) and \(24\) samples in \(B_b(f^\star)\). After activation, the horizontal coordinate of each marker is its exact directional score, while the gray segment gives the corresponding certified interval \([L_u^\sigma,U_u^\sigma]\). The ReLU intervals collapse to the exact scores. For GELU and SiLU, the certified target support is more conservative than the exact support, explaining why some successfully separated samples are not certified.

\begin{figure}[htbp]
\centering
\includegraphics[width=\linewidth]{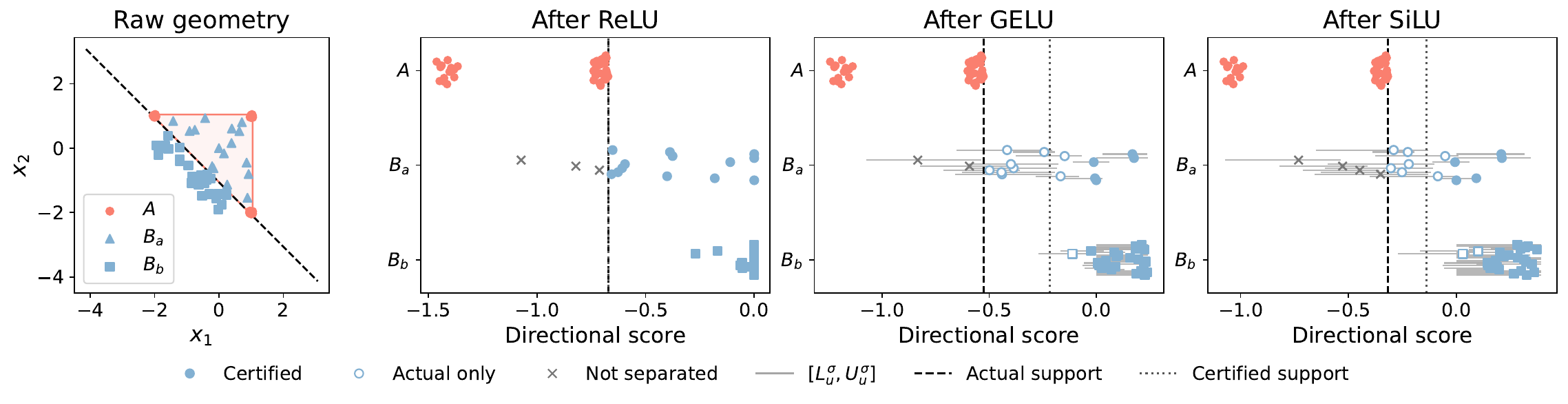}
\caption{Representative activation-certificate audit at \(\nu=1\). Filled markers are certified, hollow markers are separated but not certified, and crosses remain unseparated along \(u\). Gray segments show \([L_u^\sigma,U_u^\sigma]\); dashed and dotted lines denote the exact and certified target supports, respectively.}
\label{fig:activation_reference_audit}
\end{figure}

Table~\ref{tab:activation_certificate_audit} aggregates the preservation and recovery results over all four concentration settings.

\begin{table}[htbp]
\centering
\small
\caption{Aggregate activation-certificate results. Preservation is evaluated over \(400\) datasets per activation and recovery over \(6{,}400\) admitted samples. ``Cert./succ.'' reports certified and actually successful events.}
\label{tab:activation_certificate_audit}
\setlength{\tabcolsep}{5pt}
\begin{tabular}{lcccc}
\toprule
Act. & Mean width & Pres. cert./succ. & Rec. cert./succ. & Bound viol./eval. \\
\midrule
ReLU & \(0.000\) & \(400/400\) & \(5244/5244\) & \(0/400\) \\
GELU & \(0.382\) & \(177/400\) & \(3636/5494\) & \(0/177\) \\
SiLU & \(0.445\) & \(41/400\) & \(3198/5159\) & \(0/41\) \\
\bottomrule
\end{tabular}
\end{table}

Table~\ref{tab:activation_concentration_results} reports the concentration-specific results. Each row contains \(100\) datasets and \(1{,}600\) recovery candidates. ``Pres. cert.'' is the dataset-level preservation-certificate rate, while ``Rec. cert.'' and ``Rec. succ.'' are the sample-level recovery-certificate and actual directional-separation rates. The mean bound slack is \(s(\rho_\sigma(A),\rho_\sigma(B))-s(A,B)-|R_a(u)|/|A|\), averaged over the datasets satisfying the preservation condition.

\begin{table}[htbp]
\centering
\small
\caption{Concentration-specific activation-certificate and exact D-LSM results.}
\label{tab:activation_concentration_results}
\setlength{\tabcolsep}{3.5pt}
\begin{tabular}{lccccccc}
\toprule
Act. & \(\nu\) & Pres. cert. & Rec. cert. & Rec. succ. & Post D-LSM & Bound eval./viol. & Mean slack \\
\midrule
ReLU & \(0.2\) & \(1.000\) & \(0.430\) & \(0.430\) & \(0.773\) & \(100/0\) & \(0.001\) \\
ReLU & \(1\) & \(1.000\) & \(0.848\) & \(0.848\) & \(0.940\) & \(100/0\) & \(0.001\) \\
ReLU & \(10\) & \(1.000\) & \(1.000\) & \(1.000\) & \(1.000\) & \(100/0\) & \(0.000\) \\
ReLU & \(100\) & \(1.000\) & \(1.000\) & \(1.000\) & \(1.000\) & \(100/0\) & \(0.000\) \\
\midrule
GELU & \(0.2\) & \(0.560\) & \(0.174\) & \(0.536\) & \(0.821\) & \(56/0\) & \(0.151\) \\
GELU & \(1\) & \(0.390\) & \(0.336\) & \(0.898\) & \(0.960\) & \(39/0\) & \(0.228\) \\
GELU & \(10\) & \(0.400\) & \(0.763\) & \(1.000\) & \(1.000\) & \(40/0\) & \(0.084\) \\
GELU & \(100\) & \(0.420\) & \(1.000\) & \(1.000\) & \(1.000\) & \(42/0\) & \(0.000\) \\
\midrule
SiLU & \(0.2\) & \(0.100\) & \(0.162\) & \(0.437\) & \(0.787\) & \(10/0\) & \(0.100\) \\
SiLU & \(1\) & \(0.080\) & \(0.281\) & \(0.791\) & \(0.919\) & \(8/0\) & \(0.200\) \\
SiLU & \(10\) & \(0.090\) & \(0.570\) & \(0.997\) & \(0.999\) & \(9/0\) & \(0.153\) \\
SiLU & \(100\) & \(0.140\) & \(0.986\) & \(1.000\) & \(1.000\) & \(14/0\) & \(0.000\) \\
\bottomrule
\end{tabular}
\end{table}

The preservation-success rate is \(1\) in every setting, and both preservation and recovery certificates have precision \(1\). Increasing \(\nu\) concentrates the admitted samples near the target-hull center and produces progressively stronger recovery: at \(\nu=100\), all three activations attain mean post-activation D-LSM \(1\). The exact ReLU envelope certifies every observed directional recovery. For GELU and SiLU, direct evaluation identifies additional recovered samples beyond the sufficient certificate region, quantifying the gap between guaranteed and realized separation. Across the \(618\) datasets satisfying the preservation condition, every exact post-activation D-LSM satisfies the lower bound in Theorem~\ref{thm:lsm_improvement}; the nonnegative mean slacks further show that the certified improvement is attained or exceeded in every evaluated setting.

\subsection{Quantitative Affine--Activation Results}
\label{app:affine_activation_results}

We report results for the two-dimensional datasets and affine parameter grid specified in Appendix~\ref{app:affine_activation_setup}. All \(18{,}360\) evaluations have certified optimal post-activation D-LSM values. The preservation certificate holds in \(8{,}202\) evaluations, and every corresponding exact gain satisfies the bound in Theorem~\ref{thm:lsm_improvement}.

\paragraph{Quantitative improvement before complete separation.}
Positive gains are certified in \(5{,}841\) evaluations that remain nonseparable. Figure~\ref{fig:affine_activation_comparisons} compares certified recovery with observed recovery over the full grid, and certified gains with exact gains on the preservation-certified subset satisfying \(0.6<s_{\mathrm{post}}<1\). The latter contains \(6{,}870\) evaluations, including \(1{,}029\) with zero certified gain.

The grid also contains settings attaining the quantitative bound before complete separation. At \((\nu,\gamma,t)=(10,1.5,-0.5)\), equality holds for all \(20\) ReLU evaluations and for \(18\) of the \(20\) evaluations of each smooth activation. All \(60\) evaluations have \(s_{\mathrm{post}}=0.975\), corresponding to one remaining competing-sample intrusion. For GELU and SiLU, the mean certified and exact gains are \(0.3725\) and \(0.375\), respectively.

\begin{figure}[htbp]
\centering
\includegraphics[width=\linewidth]{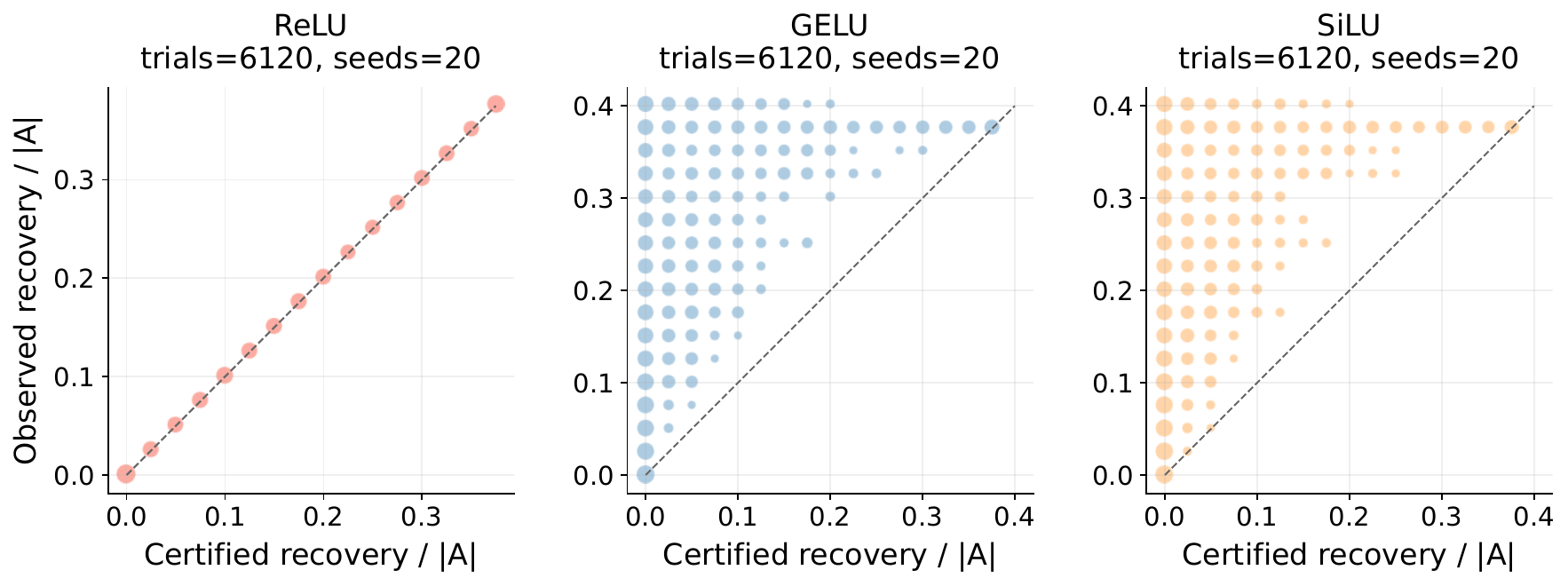}
\par\medskip
\includegraphics[width=\linewidth]{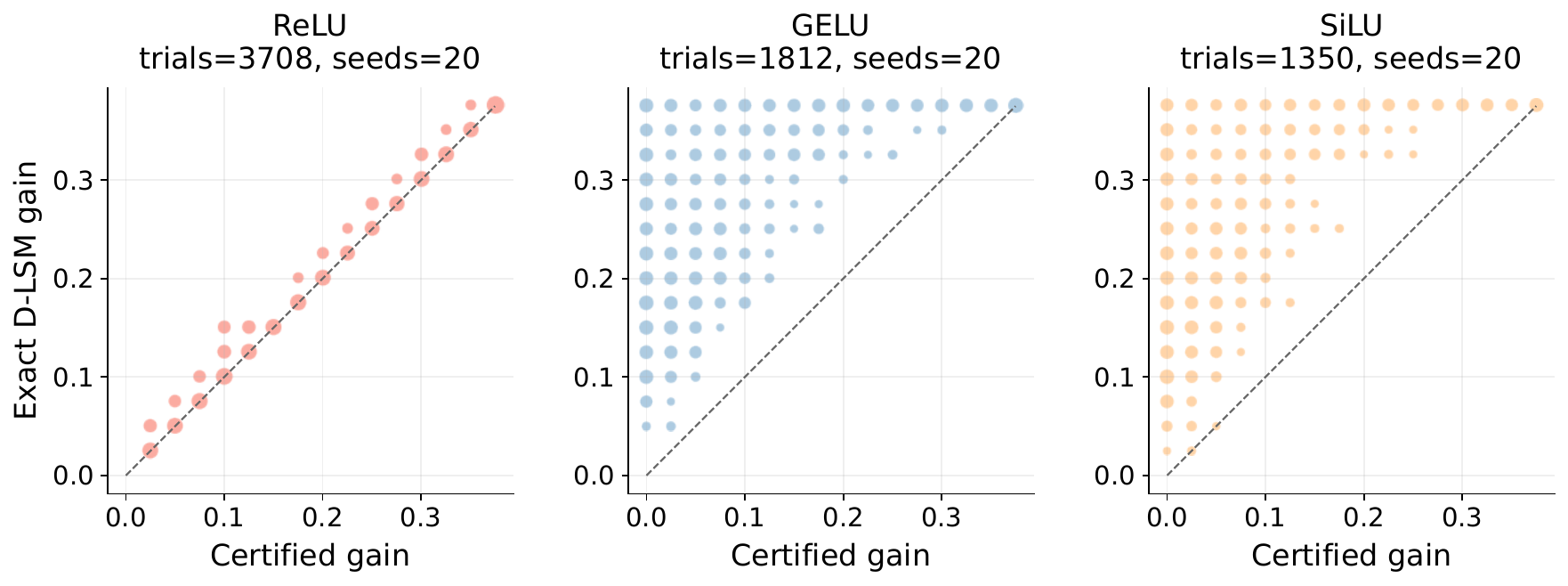}
\caption{Certified and realized outcomes. Top: certified versus observed recovery counts, normalized by \(|A|=40\), over all \(18{,}360\) evaluations. Bottom: certified versus exact gains for the \(6{,}870\) preservation-certified evaluations with \(0.6<s_{\mathrm{post}}<1\), including zero certified gains. Columns correspond to ReLU, GELU, and SiLU. Outcomes are pooled across concentrations and affine settings; marker size increases with the number of coincident evaluations. Dashed lines indicate equality.}
\label{fig:affine_activation_comparisons}
\end{figure}

\paragraph{Recovery beyond the certificates.}
For ReLU, certified and observed recoveries coincide throughout the grid, consistent with its exact scalar envelopes. For GELU and SiLU, additional samples satisfy \(u^\top F(b)>h\) while \(L_u^\sigma(T(b))\le\bar h\). These observed but uncertified recoveries account for most of the gain gap in Figure~\ref{fig:affine_activation_gain}. On its preservation-certified, nonsaturated subset at \(\nu=1\), additional recovery and direction optimization contribute mean gaps of \(0.20586\) and \(0.00142\) for GELU, and \(0.19426\) and \(0.00190\) for SiLU, respectively. Thus, in these evaluations, most improvement beyond the certificate is already realized along the transported direction.

\paragraph{Certified and observed parameter regions.}
Figure~\ref{fig:affine_activation_regions} compares the joint event of preserving every original exclusion and recovering at least eight admitted samples along the transported direction at \(\nu=1\). The certified and observed regions coincide for ReLU. For GELU and SiLU, the certified regions occupy subsets of the observed regions and concentrate around negative shifts and larger scales within the evaluated grid. The third column locates settings where the joint event occurs without certification.

The final column identifies settings where the exact post-activation D-LSM falls below \(0.6\). At \(\nu=1\), such decreases occur in \(17.65\%\), \(3.92\%\), and \(2.94\%\) of evaluations for ReLU, GELU, and SiLU, respectively. For ReLU, the shifts \(t\in\{-1.5,-1.25\}\) map all samples to zero, accounting for \(11.76\%\) of evaluations and yielding \(s_{\mathrm{post}}=0\). At \(t=-1\), the exact value is also zero although the outputs are not all identical. Thus, the decrease region includes both complete representation collapse and losses of separability with distinct output representations.

\begin{figure}[t]
\centering
\includegraphics[width=\linewidth]{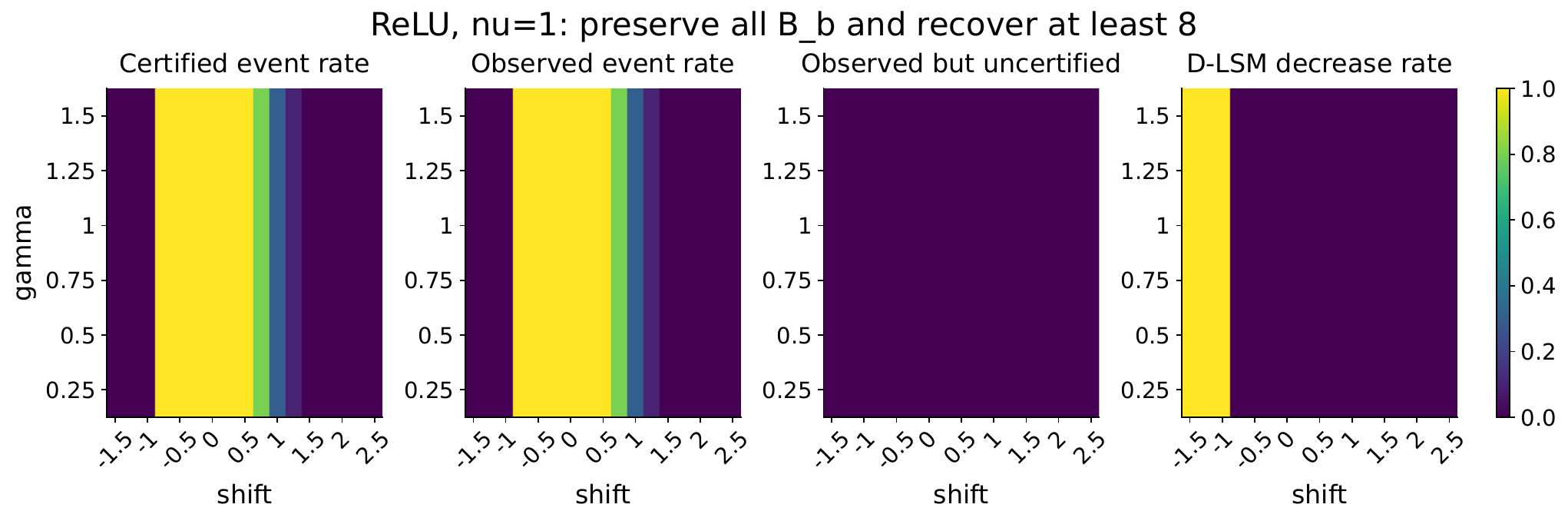}
\par\smallskip
\includegraphics[width=\linewidth]{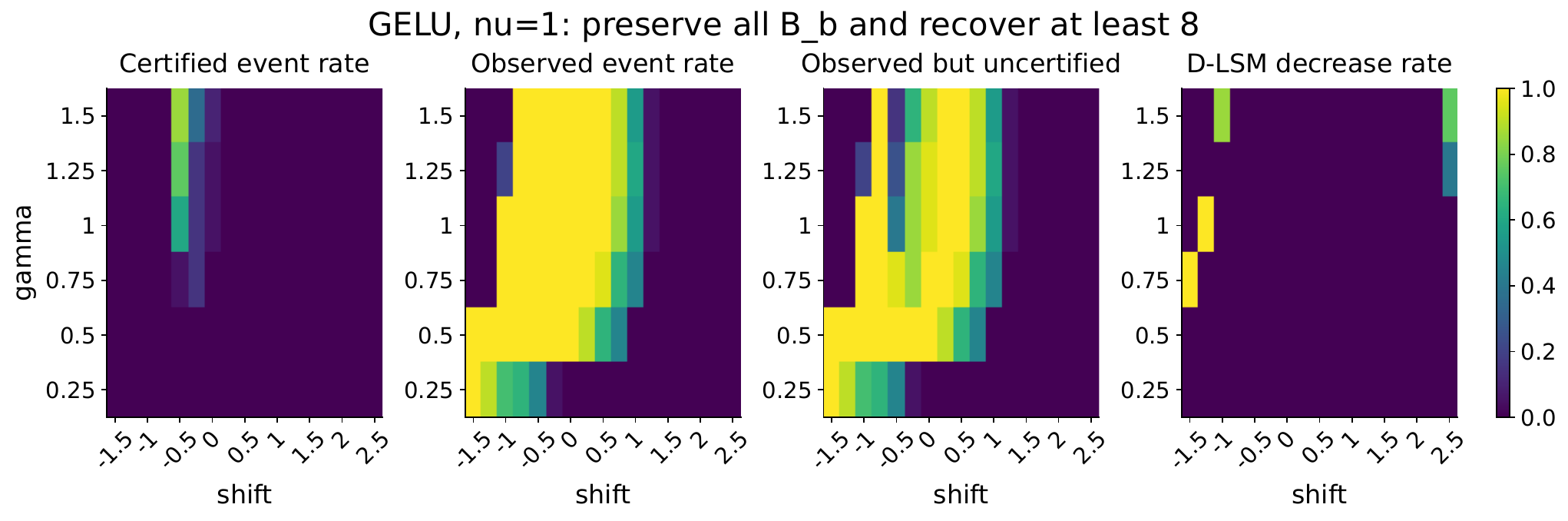}
\par\smallskip
\includegraphics[width=\linewidth]{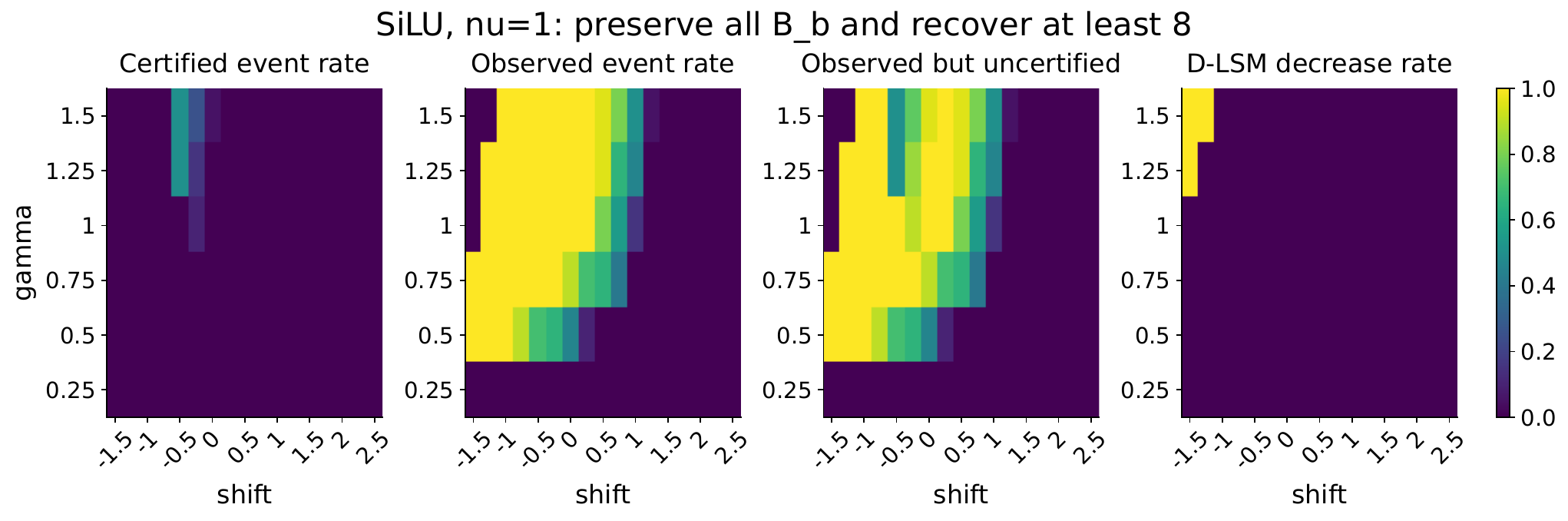}
\caption{Parameter-region comparison at \(\nu=1\), with ReLU, GELU, and SiLU shown from top to bottom. The first three columns report the frequencies of certified, observed, and observed but uncertified joint events, each requiring preservation of all \(24\) samples in \(B_b\) and recovery of at least eight samples from \(B_a\) along the transported direction. The fourth column reports the frequency of exact D-LSM decrease, \(s_{\mathrm{post}}<s_0=0.6\). Each frequency is computed across \(20\) seeds at the corresponding scale--shift setting, using the full parameter grid.}
\label{fig:affine_activation_regions}
\end{figure}

Table~\ref{tab:affine_activation_coverage} summarizes certificate coverage and D-LSM decrease across all three concentrations. For GELU and SiLU, increasing \(\nu\) raises both pointwise recovery coverage and the fraction of settings certifying the joint eight-sample recovery event. At \(\nu=10\), their pointwise recovery coverage reaches \(25.96\%\) and \(16.96\%\), respectively. Together with the decrease proportions, these full-grid statistics describe how certification and post-activation separability vary across the sampled geometries; the gain comparisons above quantify the improvement guaranteed on preservation-certified evaluations.

\begin{table}[htbp]
\centering
\small
\caption{Certificate coverage and D-LSM decrease over the full affine parameter grid. Values are means \(\pm\) sample standard deviations across \(20\) seeds, in percent. Recovery coverage is the fraction of observed sample-wise recoveries that are certified. Certified and observed region proportions are \(|\mathcal{C}_8|/102\) and \(|\mathcal{O}_8|/102\), respectively; both require preservation of every original exclusion and recovery of at least eight samples. D-LSM decrease is the fraction of the \(102\) settings with exact \(s_{\mathrm{post}}<s_0=0.6\) within each seed.}
\label{tab:affine_activation_coverage}
\setlength{\tabcolsep}{5pt}
\begin{tabular}{lccccc}
\toprule
Activation & \(\nu\) & \shortstack{Recovery\\coverage} & \shortstack{Certified\\region} & \shortstack{Observed\\region} & \shortstack{D-LSM\\decrease} \\
\midrule
ReLU & \(0.2\) & \(100.00\pm0.00\) & \(7.35\pm16.62\) & \(7.35\pm16.62\) & \(17.65\pm0.00\) \\
ReLU & \(1\) & \(100.00\pm0.00\) & \(42.35\pm5.26\) & \(42.35\pm5.26\) & \(17.65\pm0.00\) \\
ReLU & \(10\) & \(100.00\pm0.00\) & \(42.94\pm2.77\) & \(42.94\pm2.77\) & \(17.65\pm0.00\) \\
\midrule
GELU & \(0.2\) & \(9.45\pm4.30\) & \(0.00\pm0.00\) & \(20.10\pm14.06\) & \(4.22\pm1.06\) \\
GELU & \(1\) & \(14.45\pm2.28\) & \(3.19\pm2.68\) & \(43.63\pm5.63\) & \(3.92\pm0.78\) \\
GELU & \(10\) & \(25.96\pm1.95\) & \(12.55\pm1.37\) & \(46.81\pm3.05\) & \(3.92\pm0.78\) \\
\midrule
SiLU & \(0.2\) & \(6.84\pm3.28\) & \(0.00\pm0.00\) & \(17.45\pm11.86\) & \(3.04\pm0.30\) \\
SiLU & \(1\) & \(9.60\pm1.58\) & \(1.52\pm1.81\) & \(39.51\pm5.71\) & \(2.94\pm0.00\) \\
SiLU & \(10\) & \(16.96\pm1.31\) & \(7.70\pm0.92\) & \(43.43\pm2.66\) & \(2.89\pm0.22\) \\
\bottomrule
\end{tabular}
\end{table}

\subsection{Affine-Tube Results}
\label{app:tube_experiment_results}

We first report the numerical margins of the deterministic construction shown in Figure~\ref{fig:affine_tube_audit}. Table~\ref{tab:affine_tube_reference_results} compares the sufficient certificate margins with the corresponding exact post-activation margins. All parameter slacks and certificate margins are strictly positive under the common choices \(\kappa=2\), \(\lambda=5\), and \(h=3\).

\begin{table}[htbp]
\centering
\small
\caption{Parameter slacks and normalized separation margins for the deterministic affine-tube construction.}
\label{tab:affine_tube_reference_results}
\setlength{\tabcolsep}{4.5pt}
\begin{tabular}{lcccccc}
\toprule
Act. & \(\kappa\) slack & \(\lambda\) slack & Rec. cert. & Rec. actual & Pres. cert. & Pres. actual \\
\midrule
ReLU & \(2.000\) & \(3.667\) & \(0.981\) & \(0.981\) & \(3.175\) & \(3.175\) \\
GELU & \(1.029\) & \(0.773\) & \(0.208\) & \(0.992\) & \(0.335\) & \(3.123\) \\
SiLU & \(0.409\) & \(0.629\) & \(0.083\) & \(0.713\) & \(0.272\) & \(2.981\) \\
\bottomrule
\end{tabular}
\end{table}

For ReLU, the exact activation envelope makes the certified and actual margins identical. For GELU and SiLU, the exact margins are larger than their certified counterparts, reflecting the conservativeness of the smooth-activation envelopes. Nevertheless, every certified margin remains positive, so the same direction \(\widetilde u\) simultaneously recovers \(R\) and preserves the exclusion of \(B_b(f^\star)\) for all three activations.

We next examine the randomized families described in Appendix~\ref{app:tube_experiment_setup}. Their pre-activation D-LSM values are \(0.9\), \(0.8\), \(0.7\), and \(0.6\) for \(|R|=4,8,12,16\), respectively. The full-column-rank affine embedding preserves these values exactly, while every post-activation collection attains D-LSM \(1\). Figure~\ref{fig:affine_tube_gain} compares the observed improvement with the lower bound \(|R|/|A|\).

\begin{figure}[htbp]
\centering
\includegraphics[width=0.55\linewidth]{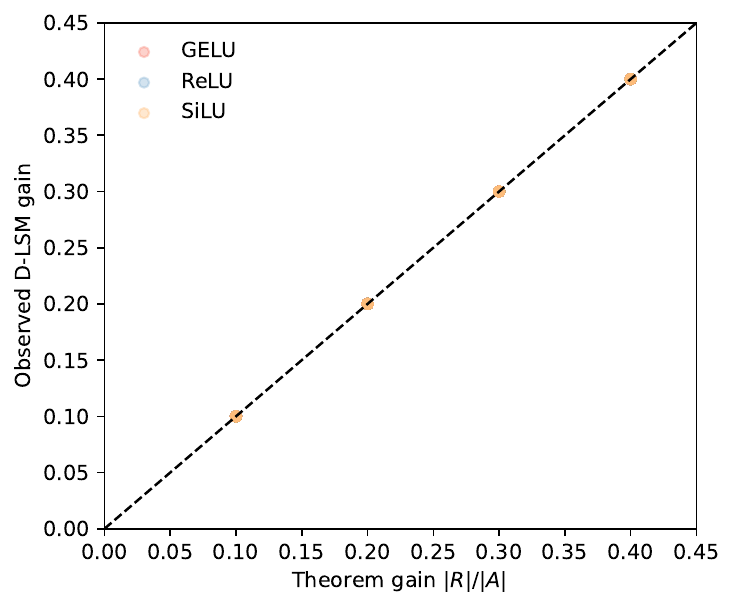}
\caption{Observed D-LSM gain versus the guaranteed gain \(|R|/|A|\) over all \(600\) randomized affine-tube evaluations. Results for all three activations lie exactly on the identity line.}
\label{fig:affine_tube_gain}
\end{figure}

Table~\ref{tab:affine_tube_recovery_sizes} provides the recovery-size-specific certificate margins. Each entry reports the median and interquartile range over \(50\) paired datasets. Success requires that every sample in \(R\) be recovered and every sample in \(B_b(f^\star)\) remain excluded after activation.

\begin{table}[htbp]
\centering
\small
\caption{Certificate margins by recovery-set size. Margins are reported as median \([q_{.25},q_{.75}]\).}
\label{tab:affine_tube_recovery_sizes}
\setlength{\tabcolsep}{4.5pt}
\begin{tabular}{lccccc}
\toprule
Act. & \(|R|\) & Recovery margin & Preservation margin & Gain & Success \\
\midrule
ReLU & \(4\) & \(1.011\,[0.997,1.030]\) & \(3.708\,[3.537,3.885]\) & \(0.100\) & \(50/50\) \\
ReLU & \(8\) & \(0.994\,[0.986,1.008]\) & \(3.728\,[3.559,3.922]\) & \(0.200\) & \(50/50\) \\
ReLU & \(12\) & \(0.989\,[0.985,0.998]\) & \(3.789\,[3.601,4.015]\) & \(0.300\) & \(50/50\) \\
ReLU & \(16\) & \(0.987\,[0.984,0.994]\) & \(3.813\,[3.601,4.036]\) & \(0.400\) & \(50/50\) \\
\midrule
GELU & \(4\) & \(0.237\,[0.223,0.256]\) & \(0.710\,[0.621,0.820]\) & \(0.100\) & \(50/50\) \\
GELU & \(8\) & \(0.220\,[0.212,0.234]\) & \(0.738\,[0.630,0.908]\) & \(0.200\) & \(50/50\) \\
GELU & \(12\) & \(0.216\,[0.211,0.224]\) & \(0.785\,[0.635,0.980]\) & \(0.300\) & \(50/50\) \\
GELU & \(16\) & \(0.214\,[0.211,0.220]\) & \(0.803\,[0.637,0.995]\) & \(0.400\) & \(50/50\) \\
\midrule
SiLU & \(4\) & \(0.112\,[0.098,0.131]\) & \(0.647\,[0.558,0.757]\) & \(0.100\) & \(50/50\) \\
SiLU & \(8\) & \(0.095\,[0.087,0.109]\) & \(0.675\,[0.567,0.845]\) & \(0.200\) & \(50/50\) \\
SiLU & \(12\) & \(0.090\,[0.086,0.099]\) & \(0.722\,[0.573,0.918]\) & \(0.300\) & \(50/50\) \\
SiLU & \(16\) & \(0.089\,[0.085,0.095]\) & \(0.741\,[0.574,0.932]\) & \(0.400\) & \(50/50\) \\
\bottomrule
\end{tabular}
\end{table}

All recovery and preservation margins remain positive across the four recovery-set sizes. Moreover, every affine map has rank \(2=d\), and the maximum cancellation residual over all \(600\) evaluations is zero. Thus, the affine stage preserves the original D-LSM and supporting-halfspace membership exactly, while the coordinatewise activation realizes the complete recovery predicted by Theorem~\ref{thm:affine_tube_recovery}. In every evaluation,
\[
s\!\left(\rho_\sigma(T(A)),\rho_\sigma(T(B))\right)-s(A,B)
=
\frac{|R|}{|A|},
\]
so the theorem lower bound is attained with equality.

\subsection{Expanded ViT D-LSM Results}
\label{app:vit_full_results}

This appendix expands the ViT results summarized in Table~\ref{tab:vit_dlsm_profile}. For each checkpoint, FFN, representation location, target class, and estimator seed, we retain the best feasible D-LSM estimate defined in Appendix~\ref{app:vit_experimental_details}. A class contributes an exact-saturation certificate when its estimate equals \(1\). All class counts are first computed separately for each estimator seed and then averaged over the five seeds.

\paragraph{Complete layer-wise certificate counts.}
Table~\ref{tab:vit_full_certificate_counts} reports the input, affine-output, and post-GELU certificate counts separately. At initialization, none of the input representations reaches exact saturation, while the affine expansion and GELU progressively produce more certificates with depth. At epochs \(20\) and \(200\), the post-GELU representation certifies the same \(177\) classes in every FFN. The affine output reaches this profile from FFN \(2\) onward, while the input reaches it from FFN \(4\) at epoch \(20\) and FFN \(3\) at epoch \(200\).

\begin{table}[htbp]
\centering
\footnotesize
\caption{Number of classes with \(\widehat{s}=1\), out of \(200\), averaged over five estimator seeds. \emph{In}, \emph{Lin}, and \emph{Post} denote \texttt{ln\_2}, \texttt{mlp.0}, and \texttt{mlp.1}, respectively.}
\label{tab:vit_full_certificate_counts}
\setlength{\tabcolsep}{3.1pt}
\begin{tabular}{c*{9}{r}}
\toprule
& \multicolumn{3}{c}{Epoch \(0\)} & \multicolumn{3}{c}{Epoch \(20\)} & \multicolumn{3}{c}{Epoch \(200\)} \\
\cmidrule(lr){2-4}\cmidrule(lr){5-7}\cmidrule(lr){8-10}
FFN & In & Lin & Post & In & Lin & Post & In & Lin & Post \\
\midrule
\(0\)  & \(0.0\) & \(0.0\)  & \(65.4\)  & \(0.0\)   & \(6.0\)   & \(177.0\) & \(0.0\)   & \(7.0\)   & \(177.0\) \\
\(1\)  & \(0.0\) & \(8.0\)  & \(99.4\)  & \(54.6\)  & \(174.0\) & \(177.0\) & \(72.0\)  & \(173.0\) & \(177.0\) \\
\(2\)  & \(0.0\) & \(16.0\) & \(101.4\) & \(158.8\) & \(177.0\) & \(177.0\) & \(171.0\) & \(177.0\) & \(177.0\) \\
\(3\)  & \(0.0\) & \(20.0\) & \(101.6\) & \(174.0\) & \(177.0\) & \(177.0\) & \(177.0\) & \(177.0\) & \(177.0\) \\
\(4\)  & \(0.0\) & \(31.4\) & \(101.2\) & \(177.0\) & \(177.0\) & \(177.0\) & \(177.0\) & \(177.0\) & \(177.0\) \\
\(5\)  & \(0.0\) & \(37.6\) & \(103.0\) & \(177.0\) & \(177.0\) & \(177.0\) & \(177.0\) & \(177.0\) & \(177.0\) \\
\(6\)  & \(0.0\) & \(42.4\) & \(110.0\) & \(177.0\) & \(177.0\) & \(177.0\) & \(177.0\) & \(177.0\) & \(177.0\) \\
\(7\)  & \(0.0\) & \(44.4\) & \(108.4\) & \(177.0\) & \(177.0\) & \(177.0\) & \(177.0\) & \(177.0\) & \(177.0\) \\
\(8\)  & \(0.0\) & \(46.4\) & \(111.4\) & \(177.0\) & \(177.0\) & \(177.0\) & \(177.0\) & \(177.0\) & \(177.0\) \\
\(9\)  & \(0.0\) & \(43.6\) & \(107.8\) & \(177.0\) & \(177.0\) & \(177.0\) & \(177.0\) & \(177.0\) & \(177.0\) \\
\(10\) & \(0.0\) & \(44.6\) & \(102.0\) & \(177.0\) & \(177.0\) & \(177.0\) & \(177.0\) & \(177.0\) & \(177.0\) \\
\(11\) & \(0.0\) & \(46.6\) & \(103.0\) & \(177.0\) & \(177.0\) & \(177.0\) & \(177.0\) & \(177.0\) & \(177.0\) \\
\bottomrule
\end{tabular}
\end{table}

\paragraph{Affine reference and layer-wise mean estimates.}
For each checkpoint and FFN, let \(\widehat{s}_{\mathrm{in},c}^{(r)}\) and \(\widehat{s}_{\mathrm{linear},c}^{(r)}\) denote the input and affine-output estimates for target class \(c\) and estimator seed \(r\). We define
\[
\widehat{s}_{\mathrm{aff},c}^{(r)}
:=
\max\left\{
\widehat{s}_{\mathrm{in},c}^{(r)},
\widehat{s}_{\mathrm{linear},c}^{(r)}
\right\}
\le
s_{\mathrm{in},c}
=
s_{\mathrm{linear},c}.
\]
The inequality follows from the feasible lower-bound property and Theorem~\ref{thm:affine_invariance}: both estimates bound the same exact quantity. The maximum is taken separately for each class and seed before computing certificate counts or averaging scores. Table~\ref{tab:vit_layerwise_mean_scores} reports the resulting affine-reference and post-GELU means. The post-GELU mean reaches \(0.99960\) in every block at both trained checkpoints, whereas the affine reference reaches this value from FFN \(2\). Negative shallow-layer means are permitted by Definition~\ref{def:directional_lsm}, since a feasible halfspace may retain more competing samples than the target-class cardinality.

\begin{table}[htbp]
\centering
\small
\caption{Mean D-LSM estimates over \(200\) classes and five estimator seeds. \emph{Aff.} is the class-wise affine reference, and \emph{Post} is the post-GELU estimate.}
\label{tab:vit_layerwise_mean_scores}
\setlength{\tabcolsep}{5pt}
\begin{tabular}{c*{6}{r}}
\toprule
& \multicolumn{2}{c}{Epoch \(0\)} & \multicolumn{2}{c}{Epoch \(20\)} & \multicolumn{2}{c}{Epoch \(200\)} \\
\cmidrule(lr){2-3}\cmidrule(lr){4-5}\cmidrule(lr){6-7}
FFN & Aff. & Post & Aff. & Post & Aff. & Post \\
\midrule
\(0\)  & \(-5.70834\) & \(0.93413\) & \(-0.34165\) & \(0.99960\) & \(-0.38736\) & \(0.99960\) \\
\(1\)  & \(0.02310\)  & \(0.97540\) & \(0.99957\)  & \(0.99960\) & \(0.99948\)  & \(0.99960\) \\
\(2\)  & \(0.50884\)  & \(0.97801\) & \(0.99960\)  & \(0.99960\) & \(0.99960\)  & \(0.99960\) \\
\(3\)  & \(0.62364\)  & \(0.97735\) & \(0.99960\)  & \(0.99960\) & \(0.99960\)  & \(0.99960\) \\
\(4\)  & \(0.74068\)  & \(0.97848\) & \(0.99960\)  & \(0.99960\) & \(0.99960\)  & \(0.99960\) \\
\(5\)  & \(0.79218\)  & \(0.98017\) & \(0.99960\)  & \(0.99960\) & \(0.99960\)  & \(0.99960\) \\
\(6\)  & \(0.81829\)  & \(0.98191\) & \(0.99960\)  & \(0.99960\) & \(0.99960\)  & \(0.99960\) \\
\(7\)  & \(0.83404\)  & \(0.98053\) & \(0.99960\)  & \(0.99960\) & \(0.99960\)  & \(0.99960\) \\
\(8\)  & \(0.85462\)  & \(0.98252\) & \(0.99960\)  & \(0.99960\) & \(0.99960\)  & \(0.99960\) \\
\(9\)  & \(0.85273\)  & \(0.98020\) & \(0.99960\)  & \(0.99960\) & \(0.99960\)  & \(0.99960\) \\
\(10\) & \(0.85546\)  & \(0.97729\) & \(0.99960\)  & \(0.99960\) & \(0.99960\)  & \(0.99960\) \\
\(11\) & \(0.86381\)  & \(0.97741\) & \(0.99960\)  & \(0.99960\) & \(0.99960\)  & \(0.99960\) \\
\bottomrule
\end{tabular}
\end{table}

\paragraph{Stability across estimator seeds.}
Table~\ref{tab:vit_seed_certificates} gives the total number of exact-saturation certificates across all \(12\times200=2{,}400\) FFN--class pairs for each seed. The initialization results vary only slightly across seeds. At epochs \(20\) and \(200\), both the affine-reference and post-GELU totals are identical across all seeds. The corresponding layer-wise counts are also identical across seeds at these two checkpoints.

\begin{table}[htbp]
\centering
\small
\caption{Total exact-saturation certificates across \(12\) FFNs and \(200\) classes. The maximum possible count per entry is \(2{,}400\).}
\label{tab:vit_seed_certificates}
\setlength{\tabcolsep}{6pt}
\begin{tabular}{ccrrrrr}
\toprule
Checkpoint & Representation & Seed \(5\) & Seed \(10\) & Seed \(1024\) & Seed \(2026\) & Seed \(2027\) \\
\midrule
Epoch \(0\)   & Affine reference & \(381\)  & \(380\)  & \(383\)  & \(379\)  & \(382\) \\
              & Post-GELU       & \(1218\) & \(1219\) & \(1209\) & \(1215\) & \(1212\) \\
\midrule
Epoch \(20\)  & Affine reference & \(1950\) & \(1950\) & \(1950\) & \(1950\) & \(1950\) \\
              & Post-GELU       & \(2124\) & \(2124\) & \(2124\) & \(2124\) & \(2124\) \\
\midrule
Epoch \(200\) & Affine reference & \(1950\) & \(1950\) & \(1950\) & \(1950\) & \(1950\) \\
              & Post-GELU       & \(2124\) & \(2124\) & \(2124\) & \(2124\) & \(2124\) \\
\bottomrule
\end{tabular}
\end{table}

\paragraph{Shared post-GELU profile.}
At epochs \(20\) and \(200\), the complete \(200\)-class post-GELU estimate vector is identical across all \(12\) FFNs and all five estimator seeds. It is also identical between the two checkpoints, giving the same vector in all \(2\times12\times5=120\) evaluated cases. Table~\ref{tab:vit_common_profile} gives its complete value distribution. Since \(|A_c|=500\), adjacent attainable D-LSM values differ by \(1/500=0.002\). Thus, \(177\) classes are certified as exactly separable, while the selected feasible boundaries for the remaining \(23\) classes retain between one and four competing samples.

\begin{table}[htbp]
\centering
\small
\caption{Class-wise distribution of the common post-GELU D-LSM estimate vector at epochs \(20\) and \(200\).}
\label{tab:vit_common_profile}
\begin{tabular}{ccc}
\toprule
D-LSM estimate & Competing samples retained & Number of classes \\
\midrule
\(1.000\) & \(0\) & \(177\) \\
\(0.998\) & \(1\) & \(14\) \\
\(0.996\) & \(2\) & \(4\) \\
\(0.994\) & \(3\) & \(2\) \\
\(0.992\) & \(4\) & \(3\) \\
\bottomrule
\end{tabular}
\end{table}

The \(177\) estimates equal to \(1\) certify exact D-LSM values of \(1\). For the remaining \(23\) classes, the reported values are valid lower bounds; they describe the best boundaries found by the finite estimator search and do not imply that the corresponding residual intrusions are unavoidable.

\paragraph{Affine-equivalent estimator behavior.}
The affine expansion is full-column-rank, so Theorem~\ref{thm:affine_invariance} gives \(s_{\mathrm{in},c}=s_{\mathrm{linear},c}\) for the exact D-LSM values. The two practical estimates are nevertheless obtained from independently optimized boundaries. Table~\ref{tab:vit_affine_estimator_comparison} quantifies this finite-search difference across the \(12{,}000\) input--linear comparisons at each checkpoint. The affine-output estimate is never below its input counterpart in the observed runs, and the fraction of unequal estimates decreases substantially after training. Taking their class-wise maximum therefore preserves lower-bound validity while using the stronger affine-equivalent estimate found by either search.

\begin{table}[htbp]
\centering
\footnotesize
\caption{Comparison of independently optimized input and affine-output D-LSM estimates. Each checkpoint contains \(5\times12\times200=12{,}000\) paired estimates. Certificate disagreement counts pairs for which exactly one estimate equals \(1\).}
\label{tab:vit_affine_estimator_comparison}
\setlength{\tabcolsep}{4.5pt}
\begin{tabular}{crrrrr}
\toprule
Checkpoint & Mean absolute difference & Linear \(>\) input & Equal & Input \(>\) linear & Certificate disagreement \\
\midrule
Epoch \(0\)   & \(1.909526\) & \(100.00\%\) & \(0.00\%\)  & \(0.00\%\) & \(1905\) \\
Epoch \(20\)  & \(0.173237\) & \(15.34\%\)  & \(84.66\%\) & \(0.00\%\) & \(733\) \\
Epoch \(200\) & \(0.144812\) & \(13.79\%\)  & \(86.21\%\) & \(0.00\%\) & \(570\) \\
\bottomrule
\end{tabular}
\end{table}

\paragraph{Post-activation comparison.}
Table~\ref{tab:vit_post_affine_comparison} compares every post-GELU estimate with its affine reference. Across all checkpoints and evaluated pairs, the post-GELU estimate is never lower than the affine reference. At initialization it is strictly higher for \(84.08\%\) of the comparisons and provides \(4{,}168\) additional exact-saturation certificates across the five seeds. At each trained checkpoint, the difference is concentrated in the shallow FFNs: GELU provides \(870\) additional certificates across the five seeds, or \(174\) per seed, while the deeper affine representations already attain the common profile.

\begin{table}[htbp]
\centering
\small
\caption{Pairwise comparison between post-GELU estimates and their affine references across \(12{,}000\) FFN--class--seed combinations per checkpoint.}
\label{tab:vit_post_affine_comparison}
\setlength{\tabcolsep}{5pt}
\begin{tabular}{ccccc}
\toprule
Checkpoint & Post \(>\) affine & Equal & Post \(<\) affine & Additional certificates \\
\midrule
Epoch \(0\)   & \(10090\;(84.08\%)\) & \(1910\;(15.92\%)\)  & \(0\) & \(4168\) \\
Epoch \(20\)  & \(985\;(8.21\%)\)    & \(11015\;(91.79\%)\) & \(0\) & \(870\) \\
Epoch \(200\) & \(985\;(8.21\%)\)    & \(11015\;(91.79\%)\) & \(0\) & \(870\) \\
\bottomrule
\end{tabular}
\end{table}

These results establish a stable depth-dependent certification pattern. At both trained checkpoints, the post-GELU representation yields the common high-certification profile from FFN \(0\), whereas the affine reference reaches the same profile from FFN \(2\). Once reached, the profile remains unchanged across subsequent blocks, estimator seeds, and the two trained checkpoints. This comparison identifies the earliest representation locations at which the estimator finds exact-saturation certificates.

\subsection{Affine Boundary Transport Results}
\label{app:vit_invariance_results}

This section reports the complete results of the affine boundary transport audit described in Appendix~\ref{app:vit_invariance_details}. The analysis covers \(360\) transport comparisons, \(72{,}000\) class-wise boundaries, and \(7.2\times10^9\) exhaustive sample-wise membership comparisons.

\paragraph{Aggregate preservation.}
All \(72{,}000\) target-feasibility decisions are preserved under boundary transport. Among the \(7.2\times10^9\) sample-wise membership comparisons, \(84\) differ between an original boundary and its transported realization, giving \(99.9999988\%\) agreement. These differences involve only competing samples. The intrusion counts agree in \(71{,}917\) of the \(72{,}000\) class-wise comparisons, or \(99.8847\%\). The maximum intrusion-count difference is two samples, corresponding to a maximum D-LSM discrepancy of \(2/500=0.004\). Across the \(360\) transport comparisons, the median of the per-transport maximum score errors is \(6.10\times10^{-5}\).

\paragraph{Checkpoint- and direction-wise results.}
Table~\ref{tab:vit_invariance_by_direction} separates the audit by model checkpoint and transport direction. Each row summarizes \(5\) seeds, \(12\) FFNs, \(200\) class-wise boundaries, and \(1.2\times10^9\) membership comparisons. We report the median and maximum of the per-transport maximum score errors, the total membership mismatch \(M_{\mathrm{mem}}\), the class-wise feasibility mismatch \(M_{\mathrm{feas}}\), the maximum absolute intrusion-count difference \(|\Delta m|\), and the maximum absolute D-LSM difference \(|\Delta s|\).

\begin{table}[htbp]
\centering
\footnotesize
\caption{Affine boundary transport results by checkpoint and direction. Each row contains \(12{,}000\) class-wise boundary comparisons and \(1.2\times10^9\) sample-wise membership comparisons.}
\label{tab:vit_invariance_by_direction}
\setlength{\tabcolsep}{3.3pt}
\begin{tabular}{cclrrrrr}
\toprule
Epoch & Direction & Median \(\epsilon_{\mathrm{score}}\) & Max. \(\epsilon_{\mathrm{score}}\) & \(M_{\mathrm{mem}}\) & \(M_{\mathrm{feas}}\) & Max. \(|\Delta m|\) & Max. \(|\Delta s|\) \\
\midrule
\(0\)   & In \(\rightarrow\) Lin & \(3.81\times10^{-5}\) & \(4.96\times10^{-5}\) & \(54\) & \(0\) & \(1\) & \(0.002\) \\
\(0\)   & Lin \(\rightarrow\) In & \(5.25\times10^{-5}\) & \(6.10\times10^{-5}\) & \(11\) & \(0\) & \(1\) & \(0.002\) \\
\(20\)  & In \(\rightarrow\) Lin & \(6.10\times10^{-5}\) & \(9.92\times10^{-5}\) & \(5\)  & \(0\) & \(1\) & \(0.002\) \\
\(20\)  & Lin \(\rightarrow\) In & \(8.39\times10^{-5}\) & \(1.18\times10^{-4}\) & \(1\)  & \(0\) & \(1\) & \(0.002\) \\
\(200\) & In \(\rightarrow\) Lin & \(6.10\times10^{-5}\) & \(3.98654\)              & \(11\) & \(0\) & \(2\) & \(0.004\) \\
\(200\) & Lin \(\rightarrow\) In & \(8.77\times10^{-5}\) & \(1.24\times10^{-4}\) & \(2\)  & \(0\) & \(1\) & \(0.002\) \\
\bottomrule
\end{tabular}
\end{table}

Target feasibility is preserved at every checkpoint and in both transport directions. Excluding the single isolated score deviation at epoch \(200\), all per-transport maximum score errors are at most \(1.24\times10^{-4}\), and the median errors remain between \(3.81\times10^{-5}\) and \(8.77\times10^{-5}\). The discrete geometric quantities are correspondingly stable: every direction has zero feasibility mismatch, and the maximum D-LSM difference is at most \(0.004\).

\paragraph{Consistency across estimator seeds.}
Table~\ref{tab:vit_invariance_by_seed} aggregates the results across checkpoints, FFNs, and directions for each estimator seed. Every seed contains \(72\) transport comparisons, \(14{,}400\) class-wise boundaries, and \(1.44\times10^9\) sample-wise membership comparisons. All seeds preserve every target-feasibility decision. The median maximum score error is identical at the reported precision, and four of the five seeds have maximum score errors between \(1.16\times10^{-4}\) and \(1.24\times10^{-4}\).

\begin{table}[htbp]
\centering
\footnotesize
\caption{Affine boundary transport results by estimator seed. \(M_{\mathrm{intr}}\) counts class-wise comparisons with different competing-sample intrusion counts.}
\label{tab:vit_invariance_by_seed}
\setlength{\tabcolsep}{4pt}
\begin{tabular}{crrrrrr}
\toprule
Seed & Median \(\epsilon_{\mathrm{score}}\) & Max. \(\epsilon_{\mathrm{score}}\) & \(M_{\mathrm{mem}}\) & \(M_{\mathrm{intr}}\) & Max. \(|\Delta m|\) & Max. \(|\Delta s|\) \\
\midrule
\(5\)    & \(6.10\times10^{-5}\) & \(3.98654\)              & \(26\) & \(25\) & \(2\) & \(0.004\) \\
\(10\)   & \(6.10\times10^{-5}\) & \(1.16\times10^{-4}\) & \(16\) & \(16\) & \(1\) & \(0.002\) \\
\(1024\) & \(6.10\times10^{-5}\) & \(1.24\times10^{-4}\) & \(13\) & \(13\) & \(1\) & \(0.002\) \\
\(2026\) & \(6.10\times10^{-5}\) & \(1.22\times10^{-4}\) & \(13\) & \(13\) & \(1\) & \(0.002\) \\
\(2027\) & \(6.10\times10^{-5}\) & \(1.18\times10^{-4}\) & \(16\) & \(16\) & \(1\) & \(0.002\) \\
\bottomrule
\end{tabular}
\end{table}

\paragraph{Layer-wise distribution.}
Table~\ref{tab:vit_invariance_by_layer} reports where the finite-precision membership differences occur. Each entry gives the total membership mismatch and maximum D-LSM difference, aggregated over both transport directions and all five seeds. At epoch \(20\), all differences are confined to FFN \(0\). At epoch \(200\), they occur only in FFNs \(0\) and \(2\). Every other trained FFN produces identical membership, intrusion, and D-LSM results across the transported boundaries.

\begin{table}[htbp]
\centering
\small
\caption{Layer-wise membership mismatch and maximum D-LSM difference. Each entry is \(M_{\mathrm{mem}}/\max|\Delta s|\), aggregated over both transport directions and five estimator seeds. Target-feasibility mismatch is zero in every entry.}
\label{tab:vit_invariance_by_layer}
\setlength{\tabcolsep}{10pt}
\begin{tabular}{cccc}
\toprule
FFN & Epoch \(0\) & Epoch \(20\) & Epoch \(200\) \\
\midrule
\(0\)  & \(21/0.002\) & \(6/0.002\) & \(8/0.002\) \\
\(1\)  & \(7/0.002\)  & \(0/0\)     & \(0/0\) \\
\(2\)  & \(4/0.002\)  & \(0/0\)     & \(5/0.004\) \\
\(3\)  & \(4/0.002\)  & \(0/0\)     & \(0/0\) \\
\(4\)  & \(3/0.002\)  & \(0/0\)     & \(0/0\) \\
\(5\)  & \(4/0.002\)  & \(0/0\)     & \(0/0\) \\
\(6\)  & \(4/0.002\)  & \(0/0\)     & \(0/0\) \\
\(7\)  & \(5/0.002\)  & \(0/0\)     & \(0/0\) \\
\(8\)  & \(5/0.002\)  & \(0/0\)     & \(0/0\) \\
\(9\)  & \(4/0.002\)  & \(0/0\)     & \(0/0\) \\
\(10\) & \(3/0.002\)  & \(0/0\)     & \(0/0\) \\
\(11\) & \(1/0.002\)  & \(0/0\)     & \(0/0\) \\
\bottomrule
\end{tabular}
\end{table}

\paragraph{Finite-precision behavior.}
The largest score deviation occurs for the input-to-output transport at epoch \(200\), seed \(5\), and FFN \(2\). This direction obtains a minimum-norm solution of \(M_j^\top v=u\) through an FP32 least-squares computation. The combination of an underdetermined linear system, high-dimensional boundary parameters, and accumulated floating-point operations can amplify absolute score error for an individual transported boundary. The resulting record contains five membership differences across four class-wise boundaries, preserves target feasibility for all \(200\) classes, and has a maximum intrusion difference of two samples and maximum D-LSM difference of \(0.004\). All reported aggregates retain this record.

The remaining membership differences follow the same finite-precision mechanism at substantially smaller score-error scales: thresholding affine scores at zero can change the membership of samples located numerically close to a boundary. These changes never alter target feasibility and modify each affected D-LSM value by at most one intrusion step, except for the single two-step case above.

\paragraph{Summary.}
The audit confirms the boundary correspondence underlying Theorem~\ref{thm:affine_invariance} throughout the learned ViT representations. Transport preserves every target-feasibility decision, agrees on \(99.9999988\%\) of all sample-wise memberships, and limits the maximum D-LSM discrepancy to \(0.004\). Consequently, differences between independently optimized input and affine-output estimates reflect the finite boundary search, whereas the exact target-constrained geometry is preserved by the full-column-rank affine expansion.

\end{document}

%% file: math_commands.tex
\usepackage{amsmath,amsfonts,bm}

\def\eqref#1{equation~\ref{#1}}

\def\1{\bm{1}}

\DeclareMathAlphabet{\mathsfit}{\encodingdefault}{\sfdefault}{m}{sl}
\SetMathAlphabet{\mathsfit}{bold}{\encodingdefault}{\sfdefault}{bx}{n}

